\documentclass[twoside,11pt]{article}

\usepackage{jmlr2e}
\usepackage{physics,mathtools,enumerate,bm,mathrsfs}
\usepackage{autobreak}
\allowdisplaybreaks[4]
\usepackage{wrapfig,subfigure,multirow}
\usepackage{tikz}
\usepackage{mleftright}
\usetikzlibrary{shapes.geometric}
\usepackage{hyperref,prettyref}

\newtheorem{condition}[theorem]{Condition}

\newrefformat{eq}{\textup{(\ref{#1})}}
\newrefformat{lem}{Lemma~\textup{\ref{#1}}}
\newrefformat{thm}{Theorem~\textup{\ref{#1}}}
\newrefformat{cor}{Corollary~\textup{\ref{#1}}}
\newrefformat{prop}{Proposition~\textup{\ref{#1}}}
\newrefformat{cond}{Condition~\textup{\ref{#1}}}
\newrefformat{assu}{Assumption~\textup{\ref{#1}}}
\newrefformat{rem}{Remark~\textup{\ref{#1}}}
\newrefformat{def}{Definition~\textup{\ref{#1}}}
\newrefformat{sec}{Section~\textup{\ref{#1}}}
\newrefformat{subsec}{Section~\textup{\ref{#1}}}

\newcommand{\ep}{\varepsilon}
\newcommand{\R}{\mathbb{R}}
\newcommand{\E}{\mathbb{E}}

\newcommand{\caH}{\mathcal{H}}

\newcommand{\caL}{\mathcal{L}}

\newcommand{\caN}{\mathcal{N}}

\newcommand{\caX}{\mathcal{X}}

\newcommand{\bbS}{\mathbb{S}}
\newcommand{\bbZ}{\mathbb{Z}}
\newcommand{\bbN}{\mathbb{N}}

\DeclareMathOperator*{\argmin}{arg\,min}
\DeclareMathOperator*{\supp}{Supp}

\DeclareMathOperator{\ran}{Ran}

\newcommand{\x}{x}
\newcommand{\y}{y}
\newcommand{\z}{z}
\newcommand{\X}{\bm{X}}
\newcommand{\mr}{\mathrm}
\newcommand{\mb}{\mathbb}
\newcommand{\mc}{\mathcal}

\newcommand{\bs}{\boldsymbol}

\newcommand{\poly}{\mr{poly}}

\newcommand{\NTK}{K^{\mathrm{NT}}}
\newcommand{\fNN}{\hat{f}^{\mathrm{NN}}_t}
\newcommand{\fNTK}{\hat{f}^{\mathrm{NTK}}_t}

\newcommand{\ang}[1]{\left\langle{#1}\right\rangle}
\newcommand{\xk}[1]{\left(#1\right)}
\newcommand{\zk}[1]{\left[#1\right]}
\newcommand{\dk}[1]{\left\{#1\right\}}
\newcommand{\xkm}[1]{\mleft(#1\mright)}
\newcommand{\zkm}[1]{\mleft[#1\mright]}
\newcommand{\dkm}[1]{\mleft\{#1\mright\}}

\newcommand{\widesim}{\mathrel{\scalebox{2.2}[1]{\hbox{$\sim$}}}}
\newcommand{\iid}{\mathrel{\stackrel{\mr{i.i.d.}}{\widesim}}} \providecommand{\cref}{\prettyref}

\usepackage{lastpage}
\jmlrheading{25}{2024}{1-\pageref{LastPage}}{7/23; Revised
1/24}{3/24}{23-0866}{Yicheng Li, Zixiong Yu, Guhan Chen and Qian Lin}
\ShortHeadings{Eigenvalues of NTK on General Domains}{Li, Yu, Chen and Lin}

\firstpageno{1}

\begin{document}

  \title{On the Eigenvalue Decay Rates of a Class of Neural-Network Related Kernel Functions Defined on General Domains}

  \author{\name Yicheng Li \email liyc22@mails.tsinghua.edu.cn \\
  \addr Center for Statistical Science, Department of Industrial Engineering\\
  Tsinghua University\\
  Beijing, 100084, China
  \AND
  \name Zixiong Yu \email yuzx19@mails.tsinghua.edu.cn \\
  \addr Yau Mathematical Sciences Center, Department of Mathematical Sciences\\
  Tsinghua University\\
  Beijing, 100084, China
  \AND Guhan Chen \email ghchen12@qq.com \\
  \addr Center for Statistical Science, Department of Industrial Engineering\\
  Tsinghua University\\
  Beijing, 100084, China
  \AND Qian Lin \email qianlin@tsinghua.edu.cn \\
  \addr Center for Statistical Science, Department of Industrial Engineering\\
  Tsinghua University\\
  Beijing, 100084, China
  }

  \editor{Ohad Shamir}

  \maketitle

  \begin{abstract}
In this paper,
we provide a strategy to determine the eigenvalue decay rate (EDR) of a large class of kernel functions defined on a general domain rather than $\mathbb S^{d}$.
This class of kernel functions include but are not limited to the neural tangent kernel associated with neural networks with different depths and various activation functions.
After proving that the dynamics of training the wide neural networks uniformly approximated that of the neural tangent kernel regression on general domains,
we  can further illustrate  the minimax optimality of the wide neural network provided that the underground truth function
$f\in [\mathcal H_{\mr{NTK}}]^{s}$, an interpolation space associated with the RKHS $\mathcal{H}_{\mr{NTK}}$ of NTK\@.
We also showed that the overfitted neural network can not generalize well.
We believe our approach for determining the EDR of kernels might be also of independent interests.

   \end{abstract}


  \section{Introduction}\label{sec:intro}
  Deep neural networks have achieved incredible success in a variety of areas,
from image classification~\citep{he2016deep,krizhevsky2017_ImagenetClassification} to natural language processing~\citep{devlin2019_BERTPretraining},
generative models~\citep{karras2019_StylebasedGenerator}, and beyond.
The number of parameters appearing in modern deep neural networks is often ten or hundreds of times larger than the sample size of the data.
It is widely observed that large neural networks possess smaller generalization errors than traditional methods.
This ``benign overfitting phenomenon'' brings challenges to the usual bias-variance trade-off doctrine in statistical learning theory.
Understanding the mysterious generalization power of deep neural networks might be one of the most interesting statistical problems.

Although the training dynamics of neural networks is highly non-linear and non-convex,
the celebrated neural tangent kernel (NTK) theory~\citep{jacot2018_NeuralTangent} provides us a way to study the generalization ability of over-parametrized neural networks.
It is shown that
when the width of neural networks is sufficiently large (i.e., in the over-parameterized or lazy trained regime),
the training dynamics of the neural network can be well approximated by a simpler kernel regression method with respect to the corresponding NTK\@.
Consequently, it offers us a way to investigate the generalization ability of the over-parametrized neural network
by means of the well established theory of generalization in kernel regression~\citep{caponnetto2007_OptimalRates,andreaschristmann2008_SupportVector,lin2018_OptimalRates}.

However, to obtain the generalization results in kernel regression, the eigenvalue decay rate (EDR) of the kernel (see \cref{eq:MercerDecomp} and below) is an essential quantity that must be determined a priori.
Considering the NTKs associated with two-layer and multilayer fully-connected ReLU neural networks,
\citet{bietti2019_InductiveBias} and the subsequent work \citet{bietti2020_DeepEquals} showed that the EDR of the NTKs is $i^{-(d+1)/d}$ when the inputs are uniformly distributed on $\bbS^d$.
Consequently, \citet{hu2021_RegularizationMatters} and \citet{suh2022_NonparametricRegression} claimed that the neural network can achieve the minimax rate $n^{-(d+1)/(2d+1)}$ of the excess risk.
However, their assumption on the input distribution is too restrictive, and can hardly be satisfied in practice,
so it is of interest to determine the EDR of the NTKs for general input domains and distributions.
As far as we know, few works have studied the EDR of the NTKs beyond the case of uniform distribution on $\bbS^d$.
More recently, focusing on one dimensional data over an interval,
\citet{lai2023_GeneralizationAbility} showed that the EDR of the NTK associated with two-layer neural networks is $i^{-2}$
and thus the neural network can achieve the minimax rate $n^{-2/3}$ of the excess risk.
However, their approach of determining the EDR, which relies heavily on the closed form expression of the NTK,
can not be generalized to $d$-dimensional inputs or the NTK associated with multilayer neural networks.

In this work, we study the EDR of the NTKs associated with multilayer fully-connected ReLU neural networks
on a general domain in $\R^d$ with respect to a general input distribution $\mu$ satisfying mild assumptions.
For this purpose, we develop a novel approach for determining the EDR of kernels by transformation and restriction.
As a key contribution, we prove that the EDR of a dot-product kernel on the sphere remains the same if one restricts it to a subset of the sphere,
which is a non-trivial generalization of the result in \citet{widom1963_AsymptoticBehavior}.
Consequently, we can show that the EDR of the NTKs is $i^{-(d+1)/d}$ for general input domains and distributions.
Moreover, after proving the uniform approximation of the over-parameterized neural network by the NTK regression,
we show the statistical optimality of the over-parameterized neural network trained via gradient descent with proper early stopping.
In comparison, we also show that the overfitted neural network can not generalize well.

\subsection{Related works}

\paragraph{The EDR of NTKs}
The spectral properties of NTK have been of particular interests to the community of theorists since \citet{jacot2018_NeuralTangent} introduced the neural tangent kernel.
For example, noticing that the NTKs associated with fully-connected ReLU networks are inner product kernels on the sphere,
several works utilized the theory of the spherical harmonics~\citep{dai2013_ApproximationTheory,azevedo2014_SharpEstimates} to study the eigen-decomposition of the NTK
~\citep{bietti2019_InductiveBias,ronen2019_ConvergenceRate,geifman2020_SimilarityLaplace,chen2020_DeepNeural,bietti2020_DeepEquals}.
In particular, \citet{bietti2019_InductiveBias} and \citet{bietti2020_DeepEquals} showed that the EDR of
the NTKs associated with the two-layer and multilayer neural network is $i^{-(d+1)/d}$ if the inputs are uniformly distributed on $\bbS^d$.
However, their analysis depends on the spherical harmonics theory on the sphere to derive the explicit expression of the eigenvalues,
which cannot be extended to general input domains and distributions.
Recently, considering two-layer ReLU neural networks on an interval,
\citet{lai2023_GeneralizationAbility} showed that the EDR of the corresponding NTK is $i^{-2}$.
However, their technique relies heavily on the explicit expression of the NTK on $\R$ and can hardly be extended to NTKs defined on $\R^d$
or NTKs associated with multilayer wide networks.

\paragraph{The generalization performance of over-parameterized neural networks}

Though now it is a common strategy to study the generalization ability of over-parameterized neural networks through that of the NTK regression,
few works state it explicitly or rigorously.
For example, \citet{du2018_GradientDescent,li2018_LearningOverparameterized,arora2019_FinegrainedAnalysis}
showed that the training trajectory of two-layer neural networks converges pointwisely to that of the NTK regressor;
\citet{du2019_GradientDescent,allen-zhu2019_ConvergenceRate,lee2019_WideNeural} further extended the results to the multilayer networks and ResNet.
However, if one wants to approximate the generalization error of over-parameterized neural network by that of the NTK regressor,
the approximation of the neural network by the kernel regressor has to be uniform.
Unfortunately, the existing two works~\citep{hu2021_RegularizationMatters,suh2022_NonparametricRegression}
studying the generalization error of over-parameterized neural networks
overlooked the aforementioned subtle difference between the pointwise convergence and uniform convergence,
so there might be some gaps in their claims.
To the best of our knowledge, \citet{lai2023_GeneralizationAbility} may be among the first works to show that two-layer wide ReLU neural networks converge uniformly to the corresponding NTK regressor.

\paragraph{The high-dimensional setting}
It should be also noted that several other works tried to consider the generalization error of NTK regression in the high-dimensional setting,
where the dimension of the input diverges as the number of samples tends to infinity.
These works include the eigenvalues of NTK, the ``benign overfitting phenomenon'', the ``double descent phenomenon'',
and the generalization error.
For example, \citet{frei2022_BenignOverfitting}, \citet{nakkiran2019_DeepDouble} and \citet{liang2020_JustInterpolate}
have shown the benign overfitting and double descent phenomena, while \citet{fan2020_SpectraConjugate} and \citet{nguyen2021_TightBounds}
have investigated the eigenvalue properties of NTK in the high-dimensional setting.
Furthermore, recent works by \citet{montanari2022_InterpolationPhase} have examined the generalization performance of neural networks in the high-dimensional setting.
However, it has been suggested by \citet{rakhlin2018_ConsistencyInterpolation,beaglehole2022_KernelRidgeless}
that there may be differences between the traditional fixed-dimensional setting and the high-dimensional setting.
In this work, we focus solely on the fixed-dimensional setting.

\subsection{Our contributions}

The main contribution of this paper is that we determine the EDR of the NTKs associated with multilayer fully-connected ReLU neural networks
on $\R^d$ with respect to a general input distribution $\mu$ satisfying mild assumptions.
We develop a novel approach for determining the EDR of kernels by means of algebraic transformation and restriction to subsets:
if the kernel can be transformed to a dot-product kernel on the sphere,
its EDR on a general domain coincides with the EDR of the resulting dot-product kernel with respect to the uniform distribution over the entire sphere,
while the latter can be determined more easily by the theory of spherical harmonics.
In particular, we show that the EDR of the considered NTKs is $i^{-(d+1)/d}$, which coincides with that of the NTKs on the sphere.
Besides, we also prove that the NTKs are strictly positive definite.
As a key technical contribution,
we prove that the EDR of a dot-product kernel on the sphere remains the same if one restricts it to a subset of the sphere, provided that
the EDR of the kernel satisfies a very mild assumption.
This result is a non-trivial generalization of the result on shift-invariant kernels in \citet{widom1963_AsymptoticBehavior}
and its proof involves fine-grained harmonic analysis on the sphere.
We believe that our approach is also of independent interest in the research of kernel methods.

Another contribution of this paper is that we rigorously prove that
the over-parameterized multilayer neural network trained by gradient descent can be approximated uniformly by the corresponding NTK regressor.
Combined with the aforementioned EDR result,
this uniform approximation allows us to characterize the generalization performance of the neural network through the well-established kernel regression theory.
The theoretical results show that proper early stopping is essential for the generalization performance of the neural networks,
which urges us to scrutinize the widely reported ``benign overfitting phenomenon'' in deep neural network literature.

\subsection{Notations}

For two sequences $a_n, b_n,~n\geq 1$ of non-negative numbers,
we write $a_n = O(b_n)$ (or $a_n = \Omega(b_n)$) if there exists absolute constant $C > 0$ such that
$a_n \leq C b_n$ (or $a_n \geq C b_n$).
We also denote $a_n \asymp b_n$ (or $a_n =\Theta(b_n)$) if $a_n = O(b_n)$ and $a_n = \Omega(b_n)$.
For a function $f : \caX \to \R$, we denote by $\norm{f}_{\infty} = \sup_{x \in \caX}\abs{f(x)}$ the sup-norm of $f$.
We denote by $L^p(\caX,\dd \mu)$ the Lebesgue $L^p$-space over $\caX$ with respect to $\mu$.

  \section{Analysis of Eigenvalue Decay Rate}\label{sec:NTK}
  The neural tangent kernel (NTK) theory~\citep{jacot2018_NeuralTangent} has been widely used to explain the generalization ability of neural networks,
which establishes a connection between neural networks and kernel methods~\citep{caponnetto2007_OptimalRates,bauer2007_RegularizationAlgorithms}.
In the framework of kernel methods, the spectral properties, in particular the eigenvalue decay rate,
of the kernel function are crucial in the analysis of the generalization ability.
Although there are several previous works~\citep{bietti2019_InductiveBias,chen2020_DeepNeural,geifman2020_SimilarityLaplace,bietti2020_DeepEquals}
investigating the spectral properties of NTKs on the sphere,
their results are limited to the case where the input distribution is uniform on the sphere.
Therefore, we would like to determine the spectral properties of NTKs on a general domain with a general input distribution.
In this section, we provide some general results on the asymptotic behavior of the eigenvalues of certain type of kernels.
As a consequence, we are able to determine the eigenvalue decay rate of NTKs on a general domain.

\subsection{The integral operator and the eigenvalues}
Let $\caX$ be a Hausdorff space and $\mu$ be a Borel measure on $\caX$.
In the following, we always consider a continuous positive definite kernel $k(x,x') : \caX \times \caX \to \R$ such that
\begin{align}
  \label{eq:IntegrableKernel}
  \int_{\caX} k(x,x) \dd \mu(x) < \infty.
\end{align}
We denote by $L^2 = L^2(\caX,\dd \mu)$ and $\caH$ the reproducing kernel Hilbert space (RKHS) associated with $k$.
Introduce the integral operator $T = T_{k;\caX,\mu}: L^2 \to L^2$ by
\begin{align}
  \label{eq:T_Def}
  (Tf)(x) = \int_{\caX} k(x,x') f(x') \dd \mu(x').
\end{align}
It is well-known~\citep{andreaschristmann2008_SupportVector,steinwart2012_MercerTheorem} that $T$ is self-adjoint, positive and trace-class (hence compact).
Consequently, we can derive the spectral decomposition of $T$ and also the Mercer's decomposition of $k$ as
\begin{align}
  \label{eq:MercerDecomp}
  T = \sum_{i\in N} \lambda_i \ang{\cdot,e_i}_{L^2} e_i, \quad\quad
  k(x,x') = \sum_{i\in N} \lambda_i e_i(x) e_i(x'),
\end{align}
where $N \subseteq \mathbb{N}$ is an index set ($N = \bbN$ if the space is infinite dimensional),
$\xk{\lambda_i}_{i \in N}$ is the set of positive eigenvalues (counting multiplicities) of $T$ in descending order
and $\xk{e_i}_{i \in N}$ are the corresponding eigenfunction, which are an orthonormal set in $L^2(\caX,\dd \mu)$.
To emphasize the dependence of the eigenvalues on the kernel and the measure,
we also denote by $\lambda_i(k;\caX,\dd \mu) = \lambda_i$.
We refer to the asymptotic rate of $\lambda_i$ as $i$ tends to infinity as the eigenvalue decay rate (EDR) of $k$ with respect to $\caX$ and $\mu$.

In the kernel regression literature,
the EDR of the kernel is closely related to the capacity condition of the corresponding reproducing kernel Hilbert space (RKHS)
and affects the rate of convergence of the kernel regression estimator~(see, e.g., \citet{caponnetto2007_OptimalRates,lin2018_OptimalRates}).
Particularly, a power-law decay that $\lambda_{i} \asymp i^{-\beta}$ is often assumed in the literature
and the corresponding minimax optimal rate depends on the exponent $\beta$.
Therefore, it would be helpful to determine such decay rate for a kernel of interest.

\subsection{Preliminary results on the eigenvalues}
\label{subsec:preliminary-results-on-the-eigenvalues}
In this subsection, we present some preliminary results on the eigenvalues of $T$,
which allow us to manipulate the kernel with algebraic transformations to simplify the analysis.
Let us first define the scaled kernel $(\rho \odot k)(x,x') = \rho(x) k(x,x') \rho(x')$ for some function $\rho: \caX \to \R$.
It is easy to see the following:
\begin{proposition}
  \label{prop:ScaleMeasureEquiv}
  Let $\rho: \caX \to \R$ be a measurable function such that $\rho\odot k$ satisfies \cref{eq:IntegrableKernel}.
  Then,
  \begin{align}
    \lambda_i(\rho \odot k;\caX,\dd \mu) = \lambda_i(k;\caX,\rho^2 \dd \mu).
  \end{align}
\end{proposition}

Furthermore, if $\rho$ is bounded, we can further estimate the eigenvalues using the minimax principle on the eigenvalues of self-adjoint compact positive operators.

\begin{lemma}
  \label{lem:ScaledKernel}
  Let a measurable function $\rho: \caX \to \R$ satisfy $0\leq c \leq \rho^2(x) \leq C$.
  Then,
  \begin{align*}
c \lambda_i(k;\caX,\dd \mu) \leq \lambda_i(\rho \odot k;\caX,\dd \mu) \leq C \lambda_i(k;\caX,\dd \mu),
    \quad \forall i = 1,2,\dots.
  \end{align*}
  Consequently, if $\nu$ is another measure on $\caX$ such that $0 \leq c \leq \frac{\dd \nu}{\dd \mu} \leq C$,
  then
  \begin{align}
    \label{eq:MeasureInv}
    c \lambda_i(k;\caX,\dd \mu) \leq \lambda_i(k;\caX,\dd \nu) \leq C \lambda_i(k;\caX,\dd \mu), \quad \forall i = 1,2,\dots.
  \end{align}
\end{lemma}

Now, we consider the transformation of the kernel.
Let $\caX_1,\caX_2$ be two sets, $\varphi : \caX_1 \to \caX_2$ be a bijection and $k_2$ be a kernel over $\caX_2$.
We define the pull-back kernel $\varphi^* k_2$ over $\caX_1$ by
\begin{align*}
(\varphi^* k_2) (x_1,x_1')
  = k_2(\varphi(x_1),\varphi(x_1')).
\end{align*}
Moreover, suppose $\caX_1$ is a measurable space with measure $\mu_1$, we define the push-forward measure $\mu_2 = \varphi_* \mu_1$ on $\caX_2$ by
$\mu_2(A) = \mu_1(\varphi^{-1}(A))$.
Then, it is easy to see that:

\begin{proposition}
  \label{prop:MapKernel}
  Let $\caX_1,\caX_2$ be two measurable spaces, $\varphi : \caX_1 \to \caX_2$ be a measurable injection,
  $\mu_1$ be a measure on $\caX_1$ and $\mu_2 = \varphi_* \mu_1$.
  Suppose $k_2$ is a kernel over $\caX_2$ and $k_1 = \varphi^* k_2$ satisfies \cref{eq:IntegrableKernel}.
  Then,
  \begin{align}
    \label{eq:MapKernel}
    \lambda_i(k_1;\caX_1,\dd \mu_1) = \lambda_i(k_2;\caX_2,\dd \mu_2).
  \end{align}
\end{proposition}

Finally, this lemma deals with the case of the sum of two kernels of different EDRs,
which is a direct consequence of \cref{lem:A_WeylOp}.
\begin{lemma}
  \label{lem:B_SumEDR}
  Let $k_1,k_2$ be two positive definite kernels on $\caX$.
  Suppose $\lambda_i(k_1;\caX,\dd \mu) \asymp \lambda_{2i}(k_1;\caX,\dd \mu)$ and
  $\lambda_i(k_2;\caX,\dd \mu) = O\left( \lambda_i(k_1;\caX,\dd \mu) \right)$ as $i \to \infty$.
  Then,
  \begin{align*}
    \lambda_i(k_1+k_2;\caX,\dd \mu) \asymp \lambda_i(k_1;\caX,\dd \mu).
  \end{align*}
\end{lemma}

\subsection{Eigenvalues of kernels restricted on a subdomain}

Suppose we are interested in $\lambda_i(k_1;\caX_1,\dd \mu_1)$.
If $k_1 = \varphi^* k_2$ for some transformation $\varphi$ and the EDR of $k_2$ with respect to some measure $\sigma$ on $\caX_2$
is known or can be easily obtained,
Then, it is tempting to combine \cref{prop:MapKernel} and \cref{lem:ScaledKernel} to obtain the EDR of $k_1$ with respect to $\mu_1$.
However, in many cases $\varphi(\caX_1)$ is a proper subset of $\caX_2$ and $\mu_2 = \varphi_* \mu_1$ is only supported on $\varphi(\caX_1)$,
so the Radon derivative $\frac{\dd \mu_2}{\dd \sigma}$ is not bounded from below (that is, $c=0$) and the lower bound in \cref{eq:MeasureInv} vanishes,
which is exactly the case of the NTK that we are interested in.
Fortunately, we can still provide such a lower bound if the kernel satisfies an appropriate invariance property.
Considering translation invariant kernels (that is, $k(x,x') = g(x-x')$), the following result based on \citet{widom1963_AsymptoticBehavior} is very inspiring.

\begin{proposition}[\citet{widom1963_AsymptoticBehavior}]
  \label{prop:WidomResult}
  Let $\mathbb{T}^d = [-\pi,\pi)^d$ be the $d$-dimensional torus and
  \begin{align*}
    k(x,x') = \sum_{\bm{n} \in \bbZ^d} c_{\bm{n}} e^{i \bm{n} \cdot x} e^{-i \bm{n} \cdot x'}
  \end{align*}
  be a translation invariant kernel on $\mathbb{T}^d$.
  Suppose further that $c_{\bm{n}}$ satisfies (i) $c_{\bm{n}} \geq 0$;
  (ii) with all $n_i$ fixed but $n_{i_0}$, $c_{\bm{n}}$, as a function of $n_{i_0}$, is nondecreasing between $-\infty$
  and some $\bar{n} = \bar{n}(i_0)$ and nonincreasing between $\bar{n}$ and $\infty$;
  (iii) if $\abs{\bm{n}}, \abs{\bm{m}} \to \infty$ and $\abs{\bm{n}} = O(\abs{\bm{m}})$, then $c_{\bm{m}} = O(c_{\bm{n}})$;
  (iv) if $\abs{\bm{n}}, \abs{\bm{m}} \to \infty$ and $\abs{\bm{n}} = o(\abs{\bm{m}})$, then $c_{\bm{m}} = o(c_{\bm{n}})$.
  Then, for a bounded non-zero Riemann-integrable function $\rho$,
  we have
  \begin{align*}
    \lambda_i(k;\mathbb{T}^d, \rho^2 \dd x) \asymp \lambda_i(k;\mathbb{T}^d,\dd x).
  \end{align*}
\end{proposition}

However, the above result is not applicable to our case since the NTKs we are interested in is not translation invariant on the torus,
but rotation invariant on the sphere.
Nevertheless, inspired by this result, we establish a similar result for dot-product kernels on the sphere as one of our main contribution.
Let $\bbS^d \subset \R^{d+1}$ be the $d$-dimensional unit sphere and $\sigma$ be the Lebesgue measure on $\bbS^d$.
We recall that a dot-product kernel $k(x,x')$ is a kernel that depends only on the dot product $u = \ang{x,x'}$ of the inputs.
Thanks to the theory of spherical harmonics~\citep{dai2013_ApproximationTheory},
the eigenfunctions of the integral operator $T$ and also the Mercer's decomposition of $k$ can be explicitly given by
\begin{align}
  \label{eq:4_k-SHDecomp}
  k(x,x') = \sum_{n=0}^{\infty} \mu_n \sum_{l=1}^{a_n} Y_{n,l}(x) Y_{n,l}(x'),
\end{align}
where $\left\{ Y_{n,l}, n \geq 0,~ l=1,\dots,a_n\right\}$ is an orthonormal basis formed by spherical harmonics,
$a_n = \binom{n+d}{n} - \binom{n-2+d}{n-2}$ is the dimension of the space of order-$n$ spherical harmonics,
and $\mu_n$ an eigenvalue of $T$ with multiplicity $a_n$.
To state our result, let us first introduce the following condition on the asymptotic decay rate of the eigenvalues.

\begin{condition}
  \label{cond:EDR}
  Let $(\mu_n)_{n \geq 0}$ be a decreasing sequence of positive numbers.
  \begin{enumerate}[(a)]
    \item Define $N(\ep) = \max \{n : \mu_n > \ep\}$.
    For any fixed constant $c > 0$, $N(c \ep) = \Theta(N(\ep))$ as $\ep \to 0$;
    suppose $\ep, \delta \to 0$ with $\ep = o(\delta)$, then $N(\delta) = o(N(\ep))$.
    \item $\triangle^{d+1} \mu_n \geq 0$ for all $n$, where $\triangle$ is the forward difference operator in \cref{def:DifferenceOperator}.
    \item There is some constant $q \in \bbN_+$ and $D > 0$ such that for any $n \geq 0$,
    \begin{align}
      \label{eq:EDRDerivativeBound}
      \sum_{l=0}^{d} \binom{\tilde{n}+l}{l} \triangle^l \mu_{\tilde{n}} \leq D \mu_n, \qq{where} \tilde{n} = qn.
    \end{align}
  \end{enumerate}
\end{condition}

\begin{remark}
  \label{rem:EDR_Condition}
  \cref{cond:EDR} is a mild condition on the decay rate and \cref{thm:EDRS} only requires that \cref{cond:EDR} holds in the asymptotic sense,
  so this requirement is quite general.
  For instance, this requirement is satisfied if
  \begin{itemize}
    \item $\mu_n \asymp n^{-\beta}$ for some $\beta > d$.
    \item $\mu_n \asymp \exp(-c_1 n^{\beta})$ for some $c_1, \beta > 0$.
    \item $\mu_n \asymp n^{-\beta} (\ln n)^p$ for $c_0 > 0$, $\beta > d$ and $p \in \R$, or $\beta = d$ and $p > 1$.
  \end{itemize}
  Furthermore, our decay rate condition aligns with existing theory,
  as similar conditions (ii)-(iv) are also needed in \citet{widom1963_AsymptoticBehavior}.
\end{remark}

\begin{theorem}
  \label{thm:EDRS}
  Let $k(x,x')$ be a dot-product kernel on $\bbS^d$
  whose corresponding eigenvalues in the decomposition \cref{eq:4_k-SHDecomp} are $(\mu_n)_{n \geq 0}$.
  Assume that there is a sequence $(\tilde{\mu}_n)_{n \geq 0}$ satisfying \cref{cond:EDR} such that $\mu_n \asymp \tilde{\mu}_n$.
  Then, for a bounded non-zero Riemann-integrable function $\rho$ on $\bbS^d$,
  we have
\begin{align}
    \lambda_i(k;\bbS^d, \rho^2 \dd \sigma)
    \asymp \lambda_i(k;\bbS^d, \dd \sigma).
  \end{align}
\end{theorem}

As our main technical contribution,
this theorem is a non-trivial generalization of the result in \citet{widom1963_AsymptoticBehavior},
adapting it from the torus to the sphere.
Following the basic idea of \citet{widom1963_AsymptoticBehavior}, we establish the theorem by proving first the main lemma (\cref{lem:EDRS_MainLemma}),
but now the approach of \citet{widom1963_AsymptoticBehavior} is not applicable since the eigen-system differs greatly.
To prove the main lemma, we utilize refined harmonic analysis on the sphere,
incorporating the technique of Cesaro summation and the left extrapolation of eigenvalues,
which necessitates the subtle requirement of \cref{cond:EDR}.
Detailed proof can be found in \cref{sec:eigen-decay-on-the-sphere}.

\cref{thm:EDRS} shows that the EDR of a dot-product kernel with respect to a general measure is the same as that of the kernel with respect to the uniform measure.
Combined with the results in \cref{subsec:preliminary-results-on-the-eigenvalues},
it provides a new approach to determine the EDR of a kernel on a general domain.
One could first transform the kernel to a dot-product kernel on the sphere with respect to some measure;
then use \cref{thm:EDRS} to show that the decay rate of the resulting dot-product kernel remains the same if we consider the uniform measure on the sphere instead;
and finally determine the decay rate of the dot-product kernel on the entire sphere by some analytic tools.
This approach enables us to determine the EDR of the NTKs corresponding to multilayer neural networks on a general domain.

\subsection{EDR of NTK on a general domain}

A bunch of previous literature
~\citep{bietti2019_InductiveBias,chen2020_DeepNeural,geifman2020_SimilarityLaplace,bietti2020_DeepEquals}
have analyzed the RKHS as well as the spectral properties of the NTKs on the sphere by means of the theory of spherical harmonics.
However, these results require the inputs to be uniformly distributed on the sphere and hence do not apply to general domains with general input distribution.
Therefore, it is of our interest to investigate the eigenvalue properties of the NTKs on a general domain with a general input distribution since it is more realistic.
To the best of our knowledge,
only \citet{lai2023_GeneralizationAbility} considered a non-spherical case of an interval on $\R$
and the EDR of the NTK corresponding to a two-layer neural network,
but their techniques are very restrictive and can not be extended to higher dimensions or multilayer neural networks.
Thanks to the results established in previous subsections,
we can determine the EDR of the NTKs on a general domain using the established results on their spectral properties on the whole sphere.

Let us focus on the following explicit formula of the NTK, which corresponds to a multilayer neural network defined later in \cref{subsec:NN_Setting}.
Introduce the arc-cosine kernels~\citep{cho2009_KernelMethods} by
\begin{align}
  \label{eq:Arccos_Formula}
  \kappa_0(u) =  \frac{1}{\pi}\left( \pi - \arccos u \right),\quad
  \kappa_1(u) = \frac{1}{\pi}\left[ \sqrt {1-u^2} + u (\pi - \arccos u)  \right].
\end{align}
Then, we define the kernel $\NTK$ on $\R^d$ by
\begin{align}
  \label{eq:NTK_Formula}
  \NTK(\x,\x') =
  \norm{\tilde{\x}} \lVert \tilde{\x}'\rVert \sum_{r=0}^L \kappa^{(r)}_1(\bar{u}) \prod_{s=r}^{L-1} \kappa_0(\kappa^{(s)}_1(\bar{u})) + 1,
\end{align}
where $L \geq 2$ is the number of hidden layers, $\tilde{\x} = (\x,1)/\norm{(\x,1)}$, $\bar{u} = \ang{\tilde{\x},\tilde{\x}'}$ and $\kappa^{(r)}_1$ represents $r$-times composition of $\kappa_1$,
see, e.g., \citet{jacot2018_NeuralTangent,bietti2020_DeepEquals}.
First, we show that $\NTK$ is strictly positive definite, the proof of which is deferred to \cref{subsec:D_PDNTK}.

\begin{proposition}
  \label{prop:NTK_PD}
  $\NTK$ is strictly positive definite on $\R^d$,
  that is, for distinct points $\x_1,\dots,\x_n \in \R^d$, the kernel matrix's smallest eigenvalue  $\lambda_{\min}\big(\NTK(\x_i,\x_j) \big)_{n \times n} > 0$.
\end{proposition}

\begin{theorem}
  \label{thm:NTK_EDR}
  Let $\mu$ be a probability measure on $\R^d$ with Riemann-integrable density $p(x)$ such that $p(x) \leq C (1+\norm{x}^2)^{-(d+3)/2}$
  for some constant $C$.
  Then, the EDR of $\NTK$ on $\R^d$ with respect to $\mu$ is
  \begin{align}
    \lambda_i(\NTK;\R^d,\dd \mu) \asymp i^{-\frac{d+1}{d}}.
  \end{align}
\end{theorem}

\begin{remark}
  The condition on the density $p(x)$ is satisfied by many common distributions, such as sub-Gaussian distributions or distributions with bounded support.
  Moreover, the result on the EDR can also be established for the NTKs corresponding to other activations
  (including homogeneous activations such as $\mr{ReLU}^\alpha(x) = \max(x,0)^\alpha$ and leaky ReLU)
  and other network architectures (such as residual neural networks),
  as long as the corresponding kernel can be transformed to a dot-product kernel on the sphere.
\end{remark}

\begin{proof}[of \cref{thm:NTK_EDR}]
  Let us denote $\bbS^d_+ = \left\{ y= (y_1,\dots,y_{d+1}) \in \bbS^d : y_{d+1} >0 \right\}$
  and introduce the homeomorphism $\Phi: \R^d \to \bbS^d_+$ by $x \mapsto \tilde{x} / \norm{\tilde{x}}$,
  where $\tilde{x} = (x,1) \in \R^{d+1}$.
  It is easy to show that the Jacobian and the Gram matrix are given by
  \begin{align*}
    J \Phi = \frac{1}{\norm{\tilde{x}}}
    \begin{pmatrix}
      I_d \\
      0
    \end{pmatrix}
    - \frac{\tilde{x}x^T}{\norm{\tilde{x}}^3},\quad
    G = (J\Phi)^T J \Phi = \frac{1}{\norm{\tilde{x}}^2} I_d - \frac{x x^T}{\norm{\tilde{x}}^4},\quad
    \det G = \norm{\tilde{x}}^{-2(d+1)}.
  \end{align*}
  Defining the homogeneous NTK $\NTK_0$ on $\bbS^d$ by
  \begin{align}
    \label{eq:NTK0_Def}
    \NTK_0(y,y') \coloneqq \sum_{r=0}^L \kappa^{(r)}_1(u) \prod_{s=r}^{L-1} \kappa_0(\kappa^{(s)}_1(u)) ,\quad u = \ang{y,y'},
  \end{align}
  it is easy to verify that
  \begin{align*}
    K_1(x,x') \coloneqq \Phi^* \NTK_0 = \sum_{r=0}^l \kappa^{(r)}_1(\bar{u}) \prod_{s=r}^{l-1} \kappa_0(\kappa^{(s)}_1(\bar{u})),
    \quad \NTK = \norm{\tilde{x}} \odot K_1 + 1.
  \end{align*}
  Therefore, \cref{prop:ScaleMeasureEquiv} and then \cref{prop:MapKernel} yields
  \begin{align*}
    \lambda_i(\norm{\tilde{x}}\odot K_1;\caX,\dd \mu) =
    \lambda_i(K_1;\caX, \norm{\tilde{x}}^2 \dd \mu) = \lambda_i\left(\NTK_0;\bbS^d, \Phi_*( \norm{\tilde{x}}^2 \dd \mu)\right).
  \end{align*}
  Moreover, denoting $\tilde{\sigma} = \Phi_*( \norm{\tilde{x}}^2 \dd \mu)$ and $p(x) = \dv{\mu}{x}$,
  we have $\dd \tilde{\sigma} = p(x) \norm{\tilde{x}}^2 \Phi_*(\dd x)$.
  On the other hand, the canonical uniform measure $\sigma$ on $\bbS^d_+$ is given by $\dd \sigma = \abs{\det G}^{\frac{1}{2}} \Phi_*(\dd x)$,
  so we have
  \begin{align*}
    q(y) \coloneqq \dv{\tilde{\sigma}}{\sigma} =  \abs{\det G}^{-\frac{1}{2}} \norm{\tilde{x}}^2 p(x) = \norm{\tilde{x}}^{d+3} p(x),
    \quad y \in \bbS^d_+.
  \end{align*}
  Therefore, the condition on $p(x)$ implies that $q(y)$ is Riemann-integrable and upper bounded.
  Now, the EDR of the dot-product kernel $\NTK_0$ on $\bbS^d$ with respect to $\dd \sigma$ is already established in \citet{bietti2020_DeepEquals}
  that $\lambda_i(\NTK_0;\bbS^d,\dd \sigma) \asymp i^{-\frac{d+1}{d}}$, so \cref{thm:EDRS} shows that
  $\lambda_i\left(\NTK_0;\bbS^d, \dd{\tilde{\sigma}} \right) \asymp i^{-\frac{d+1}{d}}$.
  Finally, the proof is completed by applying \cref{lem:B_SumEDR} to show that the extra constant does not affect the EDR\@.
\end{proof}

  \section{Application: Optimal Rates of Over-parameterized Neural Networks}\label{sec:optimality}

In this section, using the spectral properties of the NTK obtained in the previous section,
we derive the optimal rates of over-parameterized neural networks by combining the NTK theory and the kernel regression theory.
Let $d$ be fixed, $\caX \subseteq \R^d$ and $\mu$ be a sub-Gaussian\footnote{
That is, $\mu(\left\{ x \in \R^d : \norm{x} \geq t \right\}) \leq 2 \exp(-t^2 / C^2),~ \forall t \geq 0$ for some constant $C > 0$.
}
probability distribution supported on $\caX$ with upper bounded Riemann-integrable density.
Suppose we are given i.i.d.\ samples $(\x_1,y_1),(\x_2,y_2),\dots,(\x_n,y_n) \in \caX \times \R$ generated from the model
$y = f^*(\x) + \ep$,
where $\x \sim \mu$, $f^* : \caX \to \R$ is an unknown regression function and the independent noise $\ep$ is sub-Gaussian.

In terms of notations, we denote $\X = (\x_1,\dots,\x_n)$ and $\bm{y} = (y_1,\dots,y_n)^T$.
For a kernel function $k : \caX\times\caX \to \R$, we write $k(\x,\X) = (k(\x,\x_1),\dots,k(\x,\x_n))$
and $k(\X,\X) = \big(k(\x_i,\x_j)\big)_{n\times n}$.

\subsection{Setting of the neural network}
\label{subsec:NN_Setting}

We are interested in the following fully connected ReLU neural network $f(\x;\bm{\theta})$ with $L$-hidden layers
of widths $m_1,m_2,\dots, m_L$, where $L \geq 2$ is fixed.
The network includes bias terms on the first and the last layers.
To ensure that the final predictor corresponds to the kernel regressor, we consider a special mirrored architecture.
In detail, the network model is given by the following:
\begin{align*}
\begin{split}
\bm{\alpha}^{(1,p)}(\x) & = \sqrt{\tfrac{2}{m_{1}}} \sigma\xkm{\bm{A}^{(p)}\x+\bm{b}^{(0,p)}}\in\mb{R}^{m_1},~ p\in\dk{1,2},\\
    \bm{\alpha}^{(l,p)}(\x) & = \sqrt{\tfrac{2}{m_{l}}}\sigma\xkm{\bm{W}^{(l-1,p)}\bm{\alpha}^{(l-1,p)}(\x)}\in\mb{R}^{m_{l}},~
    l \in \dk{2,3,\dots,L},~p\in\dk{1,2},\\
    g^{(p)}(\x;\bm{\theta}) & = \bm{W}^{(L,p)}\bm{\alpha}^{(L,p)}(\x)+b^{(L,p)}\in\mb{R},~p\in\dk{1,2},\\
    f(\x;\bm{\theta}) &= \frac{\sqrt{2}}{2}\zk{ g^{(1)}(\x;\bm{\theta}) - g^{(2)}(\x;\bm{\theta}) }\in\mb{R}.
  \end{split}
\end{align*}
Here, $\bm{\alpha}^{(l,p)}$ represents the hidden layers; $l\in\dk{1,2,\dots,L}$, $p\in\dk{1,2}$ stand for the index of layers and parity respectively;
$\sigma(x) \coloneqq \max(x,0)$ is the ReLU activation (applied elementwise);
parameters $\bm{A}^{(p)}\in\mb{R}^{m_1\times d}$, $\bm{W}^{(l,p)} \in \R^{m_{l+1} \times m_l}$, $\bm{b}^{(0,p)} \in \R^{m_{1}}$, $b^{(L,p)} \in \R$,
where we set $m_{L+1}=1$;
and we use $\bm{\theta}$ to represent the collection of all parameters flattened into a column vector.
Letting $m = \min(m_1,m_2,\dots,m_{L})$, we assume that $\max(m_1,m_2,\dots,m_{L}) \leq C_{\mr{width}} m$ for some constant $C_{\mr{width}}$.

\paragraph{Initialization}
Considering the mirrored architecture, we initialize the parameters in one parity to be i.i.d.\ normal and set the parameters in the other parity to be the same as the corresponding ones.
More precisely,
\begin{align*}
\begin{aligned}
    &\bm{A}^{(1)}_{i,j},\bm{W}^{(l,1)}_{i,j},\bm{b}^{(0,1)}_i,b^{(L,1)} \iid N(0,1),\quad \text{for}~l=0,1,\dots,L, \\
    & \bm{W}^{(l,2)} = \bm{W}^{(l,1)},\quad\bm{A}^{(2)}=\bm{A}^{(1)},\quad\bm{b}^{(0,2)} = \bm{b}^{(0,1)}, \quad b^{(L,2)} = b^{(L,1)}.
  \end{aligned}
\end{align*}
This kind of ``mirror initialization'' ensures that the model output is always zero at initialization,
which is also considered in \citet{lai2023_GeneralizationAbility}.

\paragraph{Training}
Neural networks are often trained by the gradient descent (or its variants) with respect to the empirical loss
$\mathcal{L}(\bm{\theta}) = \frac{1}{2n}\sum_{i=1}^n (f(x_i;\bm{\theta}) - y_i)^2$.
For simplicity, we consider the continuous version of gradient descent, namely the gradient flow for the training process.
Denote by $\bm{\theta}_t$ the parameter at the time $t \geq 0$, the gradient flow is given by
\begin{align}
  \label{eq:2_GD}
  \dot{\bm{\theta}}_t = - \nabla_{\bm{\theta}} \mathcal{L}(\bm{\theta}_t)= - \frac{1}{n} \nabla_{\bm{\theta}} f(\X;\bm{\theta}_t) (f(\X;\bm{\theta}_t) - \bm{y})
\end{align}
where $f(\X;\bm{\theta}_t) = (f(x_1;\bm{\theta}_t),\dots,f(x_n;\bm{\theta}_t))^T$ and $\nabla_{\bm{\theta}} f(\X;\bm{\theta}_t)$ is an $M \times n$
matrix where $M$ is the number of parameters.
Finally, let us denote by $\fNN(\x) \coloneqq f(\x;\bm{\theta}_t)$
the resulting neural network predictor.

\subsection{Uniform convergence to kernel regression}\label{subsec:UniformConv}

Although the gradient flow \cref{eq:2_GD} is a highly non-linear and hard to analyze,
the celebrated neural tangent kernel (NTK) theory~\citep{jacot2018_NeuralTangent} provides a way to approximate the gradient flow by a kernel regressor
when the width of the network tends to infinity,
which is also referred to as the \textit{lazy training regime}.
Introducing a random kernel function $K_t(\x,\x') = \ang{\nabla_{\bm{\theta}} f(\x;\bm{\theta}_t),\nabla_{\bm{\theta}} f(\x';\bm{\theta}_t)}$,
it is shown that $K_t(\x,\x')$ concentrates in probability to a deterministic kernel $\NTK$ called the neural tangent kernel (NTK).
Consequently, the predictor $\fNN(\x)$ is well approximated by the kernel regressor $\fNTK(\x)$ given by the following gradient flow:
\begin{align}
  \label{eq:2_NTK_GF}
  \dv{t} \fNTK(\x) = - \frac{1}{n} \NTK(\x,\X) (\fNTK(\X) - \bm{y}),
\end{align}
where $\fNTK(\X) = (\fNTK(x_1),\dots,\fNTK(x_n))^T$.
Thanks to the mirrored architecture, we have $\hat{f}^{\mathrm{NN}}_0(\x) \equiv 0$ at initialization and thus $\hat{f}^{\mathrm{NTK}}_0(\x) \equiv 0$.
The recursive formula of the NTK also enables us to give explicitly~\citep{jacot2018_NeuralTangent,bietti2020_DeepEquals}
the formula of $\NTK$ in \cref{eq:NTK_Formula}.

Although previous works
~\citep{lee2019_WideNeural,arora2019_ExactComputation,allen-zhu2019_ConvergenceTheory}
showed that the neural network regressor $\fNN(\x)$ can be approximated by $\fNTK(\x)$,
most of these results are established pointwisely, namely, for fixed $\x$,
$\sup_{t\geq 0}\abs{\fNTK(\x)-\fNN(\x)}$ is small with high probability.
However,
to analyze the generalization performance of $\fNN(\x)$, the convergence is further needed to be uniform over $x \in \caX$.
Consider the simple case of two-layer neural network, \citet{lai2023_GeneralizationAbility} rigorously showed such uniform convergence.
With more complicated analysis, we prove the uniform convergence of $\fNN(\x)$ to $\fNTK(\x)$ for multilayer neural networks.
To state our result, let us denote by $\lambda_0 = \lambda_{\min}\left( \NTK(\X,\X) \right)$ the minimal eigenvalue of the kernel matrix,
which, by \cref{prop:NTK_PD}, can be assumed to be positive in the following.

\begin{lemma}
  \label{lem:UnifConverge}
  Denote $M_{\bm{X}} = \sum_{i=1}^n \norm{x_i}_2$ and $B_r = \dk{x \in \R^d : \norm{x} \leq r}$ for $r \geq 1$.
  There exists a polynomial $\mathrm{poly}(\cdot)$ such that
  for any $\delta \in (0,1)$ and $k > 0$,
  when $m \geq \mathrm{poly}(n,M_{\bm{X}},\lambda_0^{-1},$ $\norm{\bm{y}}, \ln(1/\delta), k)$ and $m \geq r^k$,
  with probability at least $1-\delta$ with respect to random initialization, we have
  \begin{align*}
    \sup_{t\geq 0} \sup_{\x \in B_r} \abs{\fNTK(\x) - \fNN(\x)} \leq O(r^2 m^{-\frac{1}{12}} \sqrt {\ln m}).
  \end{align*}
\end{lemma}

\cref{lem:UnifConverge} shows that as $m$ tends to infinity, $\fNN(\x)$ can be approximated uniformly by $\fNTK(\x)$ on a bounded set, which is also allowed to grow with $m$.
Consequently, we can study the generalization performance of the neural network in the lazy training regime by that of the corresponding kernel regressor.

To establish \cref{lem:UnifConverge},
it is essential to demonstrate the uniform convergence of the kernel $K_t(\x,\x')$ towards $\NTK(\x,\x')$.
This is achieved by first establishing the Hölder continuity of $K_t(\x,\x')$ and $\NTK(\x,\x')$,
and then applying an $\epsilon$-net argument in conjunction with the pointwise convergence.
Since the detailed proof is laborious, it is deferred to \cref{sec:A_NN}.

\subsection{The optimal rates of the over-parameterized neural network}

With the uniform convergence of the neural network to the kernel regressor established and the eigenvalue decay rate of the NTK determined,
we can now derive the optimal rates of the over-parameterized neural network.
Let us denote by $\caH = \caH_{\mathrm{NTK}}$ the RKHS associated with the NTK \cref{eq:NTK_Formula} on $\caX$.
We introduce the integral operator $T$ in \cref{eq:T_Def} and recall its spectral decomposition in \cref{eq:MercerDecomp}.
The kernel regression literature often introduces the interpolation spaces of the RKHS to characterize the regularity of the regression function
~\citep{steinwart2012_MercerTheorem,fischer2020_SobolevNorm}.
For $s \geq 0$, we define the interpolation space $[\caH]^s$ by
\begin{align}
  \zk{\caH}^s = \left\{ \sum_{i =1}^\infty a_i \lambda_i^{s/2} e_i ~\Big|~ \sum_{i =1}^\infty a_i^2 < \infty \right\}
  \subseteq L^2,
\end{align}
which is equipped with the norm $\norm{\sum_{i =1}^\infty a_i \lambda_i^{s/2} e_i}_{[\caH]^s} \coloneqq \left(  \sum_{i =1}^\infty a_i^2  \right)^{1/2}$.
It can be seen that $[\caH]^s$ is a separable Hilbert space with $\xk{\lambda_i^{s/2} e_i}_{i \geq 1}$ as its orthonormal basis.
We also have $[\caH]^0 = L^2$ and $[\caH]^1 = \caH$.
Moreover, when $s \in (0,1)$,
the space $[\caH]^s$ also coincides with the space $(L^2,\caH)_{s,2}$ defined by real interpolation~\citep{steinwart2012_MercerTheorem}.
We also denote by $B_R([\caH]^s) = \left\{ f \in [\caH]^s \mid \norm{f}_{[\caH]^s}^2 \leq R \right\}$.
Then, we derive the following optimal rates of the neural network from the optimality result in the kernel regression~\citep{lin2018_OptimalRates}.

\begin{proposition}
  \label{prop:NN_Gen}
  Suppose $f^* \in B_R([\caH]^s) \cap L^\infty$ for constants $s > \frac{1}{d+1}$ and $R > 0$.
  Let us choose $t_{\mathrm{op}} = t_{\mathrm{op}}(n) \asymp n^{(d+1)/ [s(d+1)+d]}$.
  Then, there exists a polynomial $\operatorname{poly}(\cdot)$ such that for any $\delta \in (0,1)$,
  when $n$ is sufficiently large and the width $m \geq \mathrm{poly}(n,\ln(1/\delta),\lambda_{0}^{-1})$,
  with probability at least $1-\delta$ with respect to random samples and random initialization,
  \begin{align}
    \label{eq:NN_Rate}
    \norm{\hat{f}^{\mathrm{NN}}_{t_{\mathrm{op}}} - f^*}_{L^2}^2 \leq C \left( \ln\frac{12}{\delta} \right)^2 n^{-\frac{s(d+1)}{s(d+1)+d}},
  \end{align}
  where the constant $C>0$ is independent of $\delta,n$.
  Moreover, the convergence rate in \cref{eq:NN_Rate} achieves the optimal rate in $B_R([\caH]^s)$.
\end{proposition}

The results in the kernel regression literature also allow us to provide the following sup-norm learning rate.

\begin{proposition}
  \label{prop:NN_Sup_Gen}
  Under the settings of \cref{prop:NN_Gen}, suppose further that $s \geq 1$ and $\caX$ is bounded.
  Then, when $n$ is sufficiently large,  with probability at least $1-\delta$,
  \begin{align*}
    \norm{\hat{f}^{\mathrm{NN}}_{t_{\mathrm{op}}} - f^*}_{\infty}^2 \leq C \left( \ln\frac{12}{\delta} \right)^2 n^{-\frac{(s-1)(d+1)}{s(d+1)+d}},
  \end{align*}
  where the constant $C>0$ is independent of $\delta,n$.
\end{proposition}

\begin{remark}
  \cref{prop:NN_Gen} shows the minimax optimality of wide neural networks,
  where optimal rate is also adaptive to the relative smoothness of the regression function to the NTK\@.
Our result extends the result in \citet{lai2023_GeneralizationAbility} to the scenario of $d > 1$ and $L > 1$,
  and also differs from \citet{hu2021_RegularizationMatters,suh2022_NonparametricRegression} in the following aspects:
  (1) The critical uniform convergence (\cref{lem:UnifConverge}) is not well-supported in these two works, as pointed out in \citet{lai2023_GeneralizationAbility};
  (2) They have to assume the data distribution is uniform on the sphere, while we allow $\caX$ to be a general domain;
  (3) They introduce an explicit $\ell_2$ regularization in the gradient descent and approximate the training dynamics by kernel ridge regression (KRR),
  while we consider directly the kernel gradient flow and early stopping serves as an implicit regularization, which is more natural.
  Moreover, our gradient method can adapt to higher-order smoothness of the regression function and does not
  saturate as KRR~\citep{li2023_SaturationEffect} or consequently their $\ell_2$-regularized neural networks.
\end{remark}

Moreover, using the idea in \citet{caponnetto2010_CrossvalidationBased},
we can also show that cross-validation can be used to choose the optimal stopping time.
Let us further assume that $\supp \mu$ is bounded and $y \in [-M,M]$ almost surely for some $M$ and introduce the truncation $L_{M}(a)=\min\{\abs{a},M\}\operatorname{sgn}(a)$.
Suppose now we have $\tilde{n}$ extra independent samples $(\tilde{\x}_1,\tilde{y}_1),\dots,(\tilde{\x}_{\tilde{n}},\tilde{y}_{\tilde{n}})$,
where $\tilde{n} \geq c_{\mathrm{v}} n$ for some constant $c_{\mathrm{v}} > 0$.
Let $T_n$ be a set of stopping time candidates, and we choose the empirical stopping time by cross-validation
\begin{align}
  \label{eq:5_StoppingTimeCV}
  \hat{t}_{\mathrm{cv}} = \argmin_{t \in T_n}
  \sum_{i=1}^{\tilde{n}} \left[ L_{M}\xk{\fNN(\tilde{\x}_i)} - \tilde{y}_i \right]^2.
\end{align}

\begin{proposition}
\label{prop:NN_CV}
  Under the settings of \cref{prop:NN_Gen} and the further assumptions given above,
  let $T_n = \left\{ 1, Q, \dots,Q^{\lfloor \ln_Q n \rfloor} \right\}$ for arbitrary fixed $Q > 1$ and $\hat{t}_{\mathrm{cv}}$ be chosen from \cref{eq:5_StoppingTimeCV}.
  Define $\hat{f}^{\mathrm{NN}}_{\mr{cv}}(x) = L_{M}\xk{\hat{f}^{\mathrm{NN}}_{\hat{t}_{\mathrm{cv}}}(x)}$.
  Then, there exists a polynomial $\operatorname{poly}(\cdot)$
  such that when $n$ is sufficiently large and $m \geq \mathrm{poly}(n,\ln(1/\delta),\lambda_{0}^{-1})$,
  one has
  \begin{align*}
    \norm{\hat{f}^{\mathrm{NN}}_{\mr{cv}} - f^*}_{L^2}^2 \leq C \left( \ln\frac{12}{\delta} \right)^2 n^{-\frac{s(d+1)}{s(d+1)+d}}
  \end{align*}
  with probability at least $1-\delta$ with respect to random samples and initialization,
  where the constant $C>0$ is independent of $\delta,n$.
\end{proposition}

Early stopping, as an implicit regularization, is necessary for the generalization of neural networks.
The following proposition, which is a consequence of the result in \citet{li2023_KernelInterpolation},
shows overfitted multilayer neural networks generalize poorly.
\begin{proposition}
  \label{prop:InterpolationNoGen}
  Suppose further that the samples are distributed uniformly on $\bbS^d$ and the noise is non-zero.
  Then, for any $\ep > 0$ and $\delta \in (0,1)$, there is some $c > 0$ such that when $n$ and $m$ is sufficiently large,
  one has that
  \begin{align*}
    \E \zkm{\liminf_{t \to \infty} \norm{\hat{f}^{\mathrm{NN}}_t - f^*}_{L^2}^2 ~\Big|~ \X} \geq c n^{-\ep}
  \end{align*}
  holds with probability at least $1-\delta$.
\end{proposition}

\begin{remark}
  \cref{prop:InterpolationNoGen} seems to contradict with the ``benign overfitting'' phenomenon~(e.g., \citet{bartlett2020_BenignOverfittinga,frei2022_BenignOverfitting}).
  However, we point out that in these works the dimension $d$ of the input diverges with the sample size $n$,
  while in our case $d$ is fixed, so the setting is different.
  In fact, in the fixed-$d$ scenario, several works have argued that overfitting is harmful~\citep{rakhlin2018_ConsistencyInterpolation,beaglehole2022_KernelRidgeless,li2023_KernelInterpolation}
  and our result is consistent with theirs.
\end{remark}

\begin{remark}
  The requirement of uniformly distributed samples on the sphere is due to the technical condition of the embedding index in \citet{li2023_KernelInterpolation},
  which is critical for more refined analysis in the kernel regression~\citep{fischer2020_SobolevNorm}.
  With this condition, the requirement of $s$ in \cref{prop:NN_Gen} can further be relaxed to $s > 0$.
  We hypothesize that this embedding index condition is also satisfied for the NTK on a general domain,
  but we would like to leave it to future work since more techniques on function theory are needed.
\end{remark}

  \section{Proof of the Result on the Eigenvalues}\label{sec:eigen-decay-on-the-sphere}
  In this section we provide the proof of our key result, \cref{thm:EDRS}.
The proof idea follows the same line as \citet{widom1963_AsymptoticBehavior}:
we first establish the key lemma (\cref{lem:EDRS_MainLemma}) and then use the decomposition of domains on the sphere to show \cref{thm:EDRS}.
The key technical contribution here lies in the proof of \cref{lem:EDRS_MainLemma} where
we apply a refined analysis on bounding the spherical harmonics using the Cesaro summation.

For a compact self-adjoint operator $T$, we denote by $N^{\pm}(\ep,T)$ the count of eigenvalues of $T$
that is strictly greater (smaller) than $\ep$ ($-\ep$).
We denote by $P_{\Omega}$ the operator of the multiplication of the characteristic function $\bm{1}_{\Omega}$.
For convenience, we will use $C$ to represent some positive constant that may vary in each appearance in the proof.

\subsection{Spherical harmonics}
Let us first introduce spherical harmonics and some properties that will be used.
We refer to \citet{dai2013_ApproximationTheory} for more details.
Let $\sigma$ be the Lebesgue measure on $\bbS^d$ and $L^2(\bbS^d)$ be the (real) Hilbert space equipped with the inner product
\begin{align*}
  \ang{f,g}_{L^2(\bbS^d)} = \frac{1}{\omega_{d}}\int_{\bbS^d} f g~ \dd \sigma.
\end{align*}
By the theory of spherical harmonics, the eigen-system of the Laplace-Beltrami operator $\Delta_{\bbS^d}$,
the spherical Laplacian, gives an orthogonal direct sum decomposition
$L^2(\bbS^d) = \bigoplus_{n=0}^{\infty} \caH^d_n(\bbS^d)$,
where $\caH^d_n(\bbS^d)$
is the restriction of $n$-degree homogeneous harmonic polynomials with $d+1$ variables on $\bbS^d$
and each element in $\caH^d_n(\bbS^d)$ is an eigen-function of $\Delta_{\bbS^d}$ with eigenvalue $-n(n+d-1)$.
This gives an orthonormal basis
\begin{align*}
  \left\{ Y_{n,l},\ l = 1,\dots,a_n,\ n = 1,2,\dots \right\}
\end{align*}
of $L^2(\bbS^d)$, where
$a_n = \binom{n+d}{n} - \binom{n-2+d}{n-2} \asymp n^{d-1}$
is the dimension of $\caH^d_n(\bbS^d)$ and $Y_{n,l} \in \caH^d_n(\bbS^d)$.
We also notice that
\begin{align}
  \label{eq:SH_DimensionCount}
  \sum_{n\leq N} a_n = C^N_{N+d} + C^{N-1}_{N-1+d} \asymp N^d.
\end{align}
Moreover, the summation
\begin{align}
  Z_n(\x,\y) = \sum_{l=1}^{a_n} Y_{n,l}(\x)Y_{n,l}(\y)
\end{align}
is invariant of selection of orthonormal basis $Y_{n,l}$
and $Z_n$'s are called zonal polynomials.
When $d \geq 2$, we have
\begin{align}
  \label{eq:Zonal_Gegenbauer}
  Z_n(\x,\y) = \frac{n+\lambda}{\lambda} C_n^{\lambda}(u),\quad u = \ang{\x,\y},\; \lambda = \frac{d-1}{2},
\end{align}
where $C_n^\lambda$ is the Gegenbauer polynomial.

The key property of spherical harmonics is the following Funk-Hecke formula
\citep[Theorem 1.2.9]{dai2013_ApproximationTheory}.
\begin{proposition}[Funk-Hecke formula]
  \label{prop:seca_FunkFormula}
  Let $d \geq 3$ and $f$ be an integrable function such that $\int_{-1}^1 \abs{f(t)} (1-t^2)^{d/2-1} \dd t$ is finite.
  Then for every $Y_n \in \caH^d_n(\bbS^d)$,
  \begin{align}
    \label{eq:C_FunkHecke}
    \frac{1}{\omega_d}\int_{\bbS^d} f(\ang{\x,\y}) Y_n(\y) \dd \sigma(\y) = \mu_n(f) Y_n(\x),\quad \forall \x \in \bbS^{d},
  \end{align}
  where $\mu_n(f)$ is a constant defined by
  $\mu_n(f) = \omega_d \int_{-1}^1 f(t) \frac{C_n^\lambda(t)}{C_n^\lambda(1)} (1-t^2)^{\frac{d-2}{2}} \dd t.$
\end{proposition}

We also need the following theorem relating to the Cesaro sum of zonal polynomials.
Readers may refer to \cref{subsec:AUX_Cesaro} for a definition of the Cesaro sum.

\begin{proposition}
  \label{prop:seca_CesaroJacobiPoly}
  Let
  \begin{align}
    \label{eq:C_CesaroKn}
    K_n = \frac{1}{A_n^d}\sum_{k=0}^n A_{n-k}^{d} \frac{k+\lambda}{\lambda} C^{\lambda}_k(u),
  \end{align}
  be the $d$-Cesaro sum of $\frac{k+\lambda}{\lambda}C_k^\lambda$.
  Then,
  \begin{align}
    \label{eq:C_Kn_PositiveAndBound}
    0 \leq K_n(u) \leq C n^{-1} (1-u+n^{-2})^{-(\lambda+1)},\quad \forall n \geq 1
  \end{align}
  for some positive constant $C$.
\end{proposition}
\begin{proof}
  Please refer to \citet[Theorem 2.4.3 and Lemma 2.4.6]{dai2013_ApproximationTheory}.
\end{proof}

\subsection{Dot-product kernel on the sphere}
Comparing \cref{eq:C_FunkHecke} with \cref{eq:T_Def},
the Funk-Hecke formula shows that $Y_n$ is an eigenfunction of any dot-product kernel $k(\x,\y) = f(\ang{\x,\y})$ on the sphere.
Therefore, a dot-product kernel $k(\x,\y)$ always admits the following Mercer and spectral decompositions
\begin{align}
  \label{eq:C_Mercer}
  k(\x,\y) = \sum_{n=0}^{\infty} \mu_n \sum_{l=1}^{a_n} Y_{n,l}(\x)Y_{n,l}(\y),
  \quad
  T = \sum_{n=0}^{\infty} \mu_n \sum_{l=1}^{a_n} Y_{n,l} \otimes Y_{n,l}.
\end{align}
Here we notice that $\mu_n$ is an eigenvalue having multiplicity $a_n$ and it should not be confused with $\lambda_i$
where multiplicity are counted.
In the view of \cref{eq:C_Mercer}, we may connect a dot-product kernel as well as the corresponding integral operator with the sequence
$(\mu_n)_{n \geq 0}$.

Moreover, since each $\mu_n$ is of multiplicity $a_n$, \cref{eq:SH_DimensionCount} gives
\begin{align}
  \label{eq:ConnectionMultiplicity}
  N^+(\ep,T) = \sum_{n \leq N(\ep)} a_n \asymp N(\ep)^d,
\end{align}
where $N(\ep) = \max \{n : \mu_n > \ep\}$ as defined in \cref{cond:EDR} (a).
This gives a simple relation between the asymptotic rates $\lambda_i$ and $\mu_n$.

\subsection{The main lemma}
The following main lemma is essential in the proof of our final result,
which is a spherical version of the main lemma in \citet{widom1963_AsymptoticBehavior}.
Since the eigen-system is now given by the spherical harmonics,
the approach in \citet{widom1963_AsymptoticBehavior} can not be applied.
The proof is now based on refined harmonic analysis on the sphere with the technique of Cesaro summation and the left extrapolation of eigenvalues.

\begin{lemma}
  \label{lem:EDRS_MainLemma}
  Let $T$ be given by \cref{eq:C_Mercer} with the descending eigenvalues $\bm{\mu} = (\mu_n)_{n \geq 0}$.
  Suppose further that $(\mu_n)_{n \geq 0}$ satisfies \cref{cond:EDR}.
  Let $\Omega_1,\Omega_2$ be two disjoint domains with piecewise smooth boundary.
  Then, we have
  \begin{align}
    N^{\pm}(\ep,P_{\Omega_1}T P_{\Omega_2}+P_{\Omega_2}T P_{\Omega_1}) = o(N^+(\ep,T)),\qq{as} \ep \to 0.
  \end{align}

\end{lemma}

\begin{proof}

  Let $\delta > \ep > 0$ and $\delta$ will be determined later.
  Take $M_\delta =\min\{n : \mu_n \leq \delta\} \leq M_{\ep} = \min\{n : \mu_n \leq \ep\}$.
  Using \cref{lem:LeftExtrapolation} for $p = d+1$ with \cref{cond:EDR},
  we can first construct a sequence $\bm{\mu}^{(1)}$ as the left extrapolation of $\bm{\mu}$ at $qM_{\ep}$,
  then construct a sequence $\bm{\mu}^{(2)}$ as the left extrapolation of the residual sequence $\bm{\mu} - \bm{\mu}^{(1)}$ at $qM_\delta$,
  and denote $\bm{\mu}^{(3)} = \bm{\mu} - \bm{\mu}^{(1)} - \bm{\mu}^{(2)}$, where $q$ is the integer specified in \cref{cond:EDR}
  Then, the three sequences satisfy
  \begin{align}
    \label{eq:C_MuDecomp}
    \begin{aligned}
      & \mu_n = \mu_n^{(1)} + \mu_n^{(2)} + \mu_n^{(3)},   \quad \triangle^{d+1}\mu_n^{(i)} \geq 0,~ i=1,2,3; \\
      & \mu_0^{(1)} = \caL_{qM_{\ep}}^{d+1} \bm{\mu} \leq D \mu_{M_{\ep}} \leq D \ep; \\
      &\mu_n^{(2)} = 0,~\forall n \geq qM_{\ep},\quad
      \mu_0^{(2)} = \caL_{qM_\delta}^{d+1} (\bm{\mu} - \bm{\mu}^{(1)}) \leq \caL_{qM_\delta}^{d+1} \bm{\mu} \leq D \delta; \\
      &\mu_n^{(3)} = 0,~\forall n \geq qM_\delta,
    \end{aligned}
  \end{align}
  where the control $\caL_{qM_{\delta}}^{d+1} \bm{\mu} \leq C \mu_M$ comes from \cref{eq:EDRDerivativeBound}.
  Now, we define $T_i$ to be the integral operator associated with $\bm{\mu}^{(i)}$, that is,
  $T_i = \sum_{n=0}^{\infty} \mu_{n}^{(i)} \sum_{l=1}^{a_n} Y_{n,l} \otimes Y_{n,l}.$
  Let $N_i^{+}(\ep)$ be the count of eigenvalues of $P_{\Omega_1}T_i P_{\Omega_2}+P_{\Omega_2}T_i P_{\Omega_1}$ greater than $\ep$.
  By \cref{lem:seca_SumEigenCount} we have
  \begin{align}
    \label{eq:Proof_MainLemma_N_decomp}
    N^+( (2D+1)\ep, P_{\Omega_1}T P_{\Omega_2}+P_{\Omega_2}T P_{\Omega_1} )
    \leq N_1^+(2D\ep) + N_2^+(\ep) + N_3^+(0).
  \end{align}

  For $N_1^+(2D\ep)$, we notice that $\norm{T_1} \leq D \ep$ and hence
  \begin{align*}
    \norm{P_{\Omega_1}T_1 P_{\Omega_2}+P_{\Omega_2}T_1 P_{\Omega_1}} \leq 2 D \ep,
  \end{align*}
  which implies that
  \begin{align}
    \label{eq:Proof_MainLemma_N1_Bound}
    N_1^+(2D\ep) = 0.
  \end{align}

  For $N_3^+(0)$,  since $\mu_{n}^{(3)} \neq 0$ only when $n < qM_\delta$, so
  \begin{align}
    \label{eq:Proof_MainLemma_N3_Bound}
    N_3^+(0) \leq 2 \mathrm{Rank}(T_3) \leq 2 \sum_{n < qM_{\delta}} a_n \leq 2 C q^d \sum_{n < M_{\delta}}a_n = 2 C N^+(\delta,T),
  \end{align}
  where we use \cref{eq:SH_DimensionCount} in the third inequality
  and $\sum_{n < M_{\delta}}a_n = N^+(\delta,T)$ (notice that $\mu_n$ is an eigenvalue of multiplicity $a_n$).

  It remains to bound $N_2^+(\ep)$.
  First, by definition of the HS-norm and \citet[Theorem 3.8.5]{simon2015_OperatorTheory}, we have
  \begin{align}
    \label{eq:Proof_MainLemma_N2_Int}
    \ep^2 N_2^+(\ep) & \leq \norm{P_{\Omega_1}T_2 P_{\Omega_2}+P_{\Omega_2}T_2 P_{\Omega_1}}_{\mathrm{HS}}^2
    = 2 \int_{\Omega_1} \int_{\Omega_2} \abs{\sum_{n} \mu_n^{(2)} Z_n(\x,\y)}^2 \dd \y \dd \x \eqqcolon 2 I.
  \end{align}
  Fixing an interior point ${e}$ of $\Omega_2$,
  we introduce an isometric transform $R_{{e},\x}$ such that $R_{{e},\x} {e} = \x$.
  It can be taken to be the rotation over the plane spanned by ${e},\x$ if they are not parallel,
  to be the identity map if $\x = {e}$ and to be reflection if $\x = -{e}$.
  Then, since $R_{{e},\x}$ is isometric and $Z_n(\x,\y)$ depends only on $\ang{\x,\y}$, we have
  \begin{align*}
    I &= \int_{\Omega_1} \int_{\Omega_2} \abs{\sum_{n} \mu_n^{(2)} Z_n(\x,\y)}^2 \dd \y \dd \x
    = \int_{\Omega_1}\int_{\Omega_2} \abs{\sum_{n} \mu_n^{(2)} Z_n(R_{{e},\x}^{-1} \x,R_{{e},\x}^{-1} \y)}^2 \dd \y \dd \x \\
    &= \int_{\Omega_1}\int_{\Omega_2} \abs{\sum_{n} \mu_n^{(2)} Z_n({e},R_{{e},\x}^{-1} \y)}^2 \dd \y \dd \x
    = \int_{\Omega_1}\int_{R_{{e},\x}^{-1} \Omega_2}\abs{\sum_{n} \mu_n^{(2)} Z_n({e},\y)}^2 \dd \y \dd \x \\
    &= \iint_{\bbS^d \times \bbS^d} \bm{1}\left\{ \x \in \Omega_1,~ \y \in R_{{e},\x}^{-1} \Omega_2 \right\} \abs{\sum_n \mu_n^{(2)} Z_n({e},\y)}^2 \dd \x \dd \y \\
    & = \int_{\bbS^d}\left( \int_{\bbS^d} \bm{1}\left\{ \x \in \Omega_1,~ R_{{e},\x} \y \in \Omega_2 \right\} \dd \x \right) \abs{\sum_n \mu_n^{(2)} Z_n({e},\y)}^2 \dd \y  \\
    &= \int_{\bbS^d} \abs{\left\{ \x \in \Omega_1 : R_{{e},\x} \y \in \Omega_2 \right\}} \abs{\sum_n \mu_n^{(2)} Z_n({e},\y)}^2 \dd \y \\
    & \leq \int_{\bbS^d} \abs{\left\{ \x \in \Omega_1 : R_{{e},\x} \y \in \Omega_2 \right\}} \abs{\sum_{n} \mu_n^{(2)} Z_n({e},\y)}^2 \dd \y \\
    & \leq C \int_{\bbS^d} \arccos \ang{\y,{e}}  \abs{\sum_{n} \mu_n^{(2)} Z_n({e},\y)}^2 \dd \y,
  \end{align*}
  where the last inequality comes from \cref{prop:AreaControl}.
  Let $\eta > 0$ (which will be determined later), we decompose the last integral into two parts:
  \begin{align*}
    I = I_1 + I_2 = \int_{\ang{\y,{e}} > 1-\eta}  + \int_{\ang{\y,{e}} < 1-\eta} \arccos \ang{\y,{e}} \abs{\sum_{n} \mu_n^{(2)} Z_n({e},\y)}^2 \dd \y.
\end{align*}

  For $I_1$, using the estimation $\arccos u \leq C \sqrt {1-u}$,
  we obtain
  \begin{align*}
    I_1 \leq \int_{\ang{\y,{e}} > 1-\eta} C\eta^{\frac{1}{2}} \abs{\sum_{n} \mu_n^{(2)} Z_n({e},\y)}^2 \dd \y
    \leq C\eta^{\frac{1}{2}} \int_{\bbS^d}  \abs{\sum_{n} \mu_n Z_n({e},\y)}^2 \dd \y
    = C\eta^{\frac{1}{2}} \sum_n a_n (\mu_{n}^{(2)})^2.
\end{align*}
  Using \cref{eq:C_MuDecomp}, we get
  \begin{align}
    \notag
    I_1 & \leq C\eta^{\frac{1}{2}} \sum_n a_n (\mu_{n}^{(2)})^2
    \leq C\eta^{\frac{1}{2}} \sum_{n < qM_\ep}a_n (\mu_{0}^{(2)})^2
    \leq C\eta^{\frac{1}{2}} q^d \sum_{n < M_\ep} a_n (\mu_{0}^{(2)})^2 \\
    &\leq C\eta^{\frac{1}{2}} N^+(\ep,T) \delta^2,
    \label{eq:Proof_MainLemma_I1_Bound}
  \end{align}
  where we use \cref{eq:SH_DimensionCount} again in the third inequality.

  For $I_2$, recalling \cref{eq:Zonal_Gegenbauer} and denoting $u = \ang{\y,{e}}$, we have
  \begin{align*}
    I_2    = \int_{-1}^{1-\eta} \abs{\sum_{n} \mu_n^{(2)} \frac{n+\lambda}{\lambda}  C_n^{\lambda}(u)}^2 \left( 1-u^2 \right)^{\frac{d-2}{2}} \arccos u \dd u,
  \end{align*}
  where $\lambda = \frac{d-1}{2}$.
  Using summation by parts (\cref{prop:Aux_SumByParts}), we obtain
  \begin{align*}
    \sum_{n} \mu_n^{(2)} \frac{k+\lambda}{\lambda}  C_n^{\lambda}(u) = \sum_n \triangle^{d+1} \mu_n^{(2)}  A_n^d K_n(u) ,
  \end{align*}
  where $K_n$  is the $d$-Cesaro sum of $C^{\lambda}_k(u)$ as in \cref{eq:C_CesaroKn}.
  Moreover, \cref{eq:C_Kn_PositiveAndBound} in \cref{prop:seca_CesaroJacobiPoly} yields
  \begin{align*}
    \abs{\sum_{n} \mu_n \frac{k+\lambda}{\lambda}  C_n^{\lambda}(u)}
    &= \sum_n \triangle^{d+1} \mu_n^{(2)} A_n^d K_n(u)
    \leq C (1-u)^{-(\lambda+1)} \sum_n A_n^d \triangle^{d+1}\mu_n^{(2)} \\
    & = C (1-u)^{-(\lambda+1)} \mu_{0,2} \leq C (1-u)^{-(\lambda+1)} \delta ,
  \end{align*}
  where the last but second equality comes from  \cref{prop:TailSumDifference}.
  Plugging this estimation back into $I_2$, we obtain
  \begin{align}
    \label{eq:Proof_MainLemma_I2_Bound}
    I_2 & \leq C  \delta^2  \int_{-1}^{1-\eta} (1-u)^{-2(\lambda+1)} \left( 1-u^2 \right)^{\frac{d-2}{2}} (1-u)^{1/2} \dd u
    \leq  C \delta^2  \eta^{-\frac{d-1}{2}}.
  \end{align}

  Now we obtain the estimations \cref{eq:Proof_MainLemma_I1_Bound} and \cref{eq:Proof_MainLemma_I2_Bound}.
  Taking $\eta = N^+(\delta,T)^{-\frac{2}{d}}$, we have
  \begin{align*}
    I \leq I_1 + I_2 \leq C\delta^2 N^+(\ep,T)^{\frac{d-1}{d}},
  \end{align*}
  so \cref{eq:Proof_MainLemma_N2_Int} yields
  \begin{align}
    \label{eq:Proof_MainLemma_N2_Bound}
    N_2^+(\ep) \leq C \left( \frac{\delta}{\ep} \right)^2 N^+(\ep,T)^{\frac{d-1}{d}}.
  \end{align}

  Finally, plugging \cref{eq:Proof_MainLemma_N1_Bound}, \cref{eq:Proof_MainLemma_N2_Bound} and \cref{eq:Proof_MainLemma_N3_Bound} into
  \cref{eq:Proof_MainLemma_N_decomp}, we have
  \begin{align*}
    N^+( (2D+1)\ep, P_{\Omega_1}T P_{\Omega_2}+P_{\Omega_2}T P_{\Omega_1} )
    \leq 2N^+(\delta,T) + C \left( \frac{\delta}{\ep} \right)^2 N^+(\ep,T)^{\frac{d-1}{d}}.
  \end{align*}
  Now, \cref{eq:ConnectionMultiplicity} allows us to derive a similar condition on $N^+(\ep,T)$ as (a) in \cref{cond:EDR}.
  Therefore, taking $\delta = \ep N^+(\ep,T)^{\frac{1}{4d}}$ so that $\ep = o(\delta)$, we obtain $N^+(\delta,T) = o\left( N^+(\ep,T) \right)$,
  so
  \begin{align*}
    N^+((2D+1)\ep, P_{\Omega_1}T P_{\Omega_2}+P_{\Omega_2}T P_{\Omega_1} )
    \leq o\left( N^+(\ep,T) \right) + C N^+(\ep,T)^{\frac{2d-1}{2d}} = o\left( N^+(\ep,T) \right).
  \end{align*}
  Since $D$ is a fixed constant and $N^+(\ep/(2D+1),T)  = \Theta(N^+(\ep,T)) $,
  replacing $\ep$ by $\ep/(2D+1)$ yields the desired result.

  The proof of the case $N^{-}(\ep,P_{\Omega_1}T P_{\Omega_2}+P_{\Omega_2}T P_{\Omega_1})$ is similar.
\end{proof}

\subsection{The main result}
The following is a direct corollary of \cref{lem:ScaledKernel}.
We present it here since it will be frequently used later.

\begin{corollary}
  \label{cor:EigenCountUpperSubdomain}
  Suppose $\Omega_1 \subseteq \Omega_2$.
  Then, $N^+(\ep, P_{\Omega_1}TP_{\Omega_1}) \leq N^+(\ep,P_{\Omega_2}TP_{\Omega_2}).$
\end{corollary}

We first prove the following lemma about dividing a domain into isometric subdomains,
which will be used recursively in the proof later.

\begin{lemma}
  \label{lem:C_EigenCountAsympSubdomain}
  Let $T$ be the same in \cref{lem:EDRS_MainLemma}.
  Let $S \subseteq \bbS^d$ and suppose $N^+(\ep, P_{S} T P_{S}) \asymp N^+(\ep,T)$ as $\ep \to 0$.
  Suppose further that $\Omega \subseteq \bbS^d$ is a subdomain with piecewise smooth boundary and
  there exists isometric copies $\Omega_1,\dots,\Omega_m$ of $\Omega$ such that their disjoint union
  (with a difference of a null-set) is $S$.
  Then, there is some constant $c > 0$ such that for small $\ep$,
  \begin{align}
    \label{eq:EigenCountAsympSubdomain}
    c N^+(\ep,T) \leq  N^+(\ep, P_{\Omega} T P_{\Omega}) \leq N^+(\ep,T).
  \end{align}
\end{lemma}
\begin{proof}
  The upper bound follows from \cref{cor:EigenCountUpperSubdomain}.
  Now we consider the lower bound.
  Since $\Omega_1,\dots,\Omega_m$ form a disjoint cover of $S$, we have
  \begin{align*}
    P_{S} T P_{S} = (\sum_{i=1}^m P_{\Omega_i}) T (\sum_{j=1}^m P_{\Omega_j}) = \sum_{i} P_{\Omega_i} T P_{\Omega_i}
    + \sum_{i < j} \left( P_{\Omega_i} T P_{\Omega_j} + P_{\Omega_j} T P_{\Omega_i} \right).
  \end{align*}
  Using \cref{lem:seca_SumEigenCount}, we get
  \begin{align*}
    N^+(2\ep, T)
    \leq N^+(\ep, \sum_{i} P_{\Omega_i} T P_{\Omega_i})
    + \sum_{i < j} N^+\left( \frac{1}{C_m^2}\ep,  P_{\Omega_i} T P_{\Omega_j} + P_{\Omega_j} T P_{\Omega_i} \right),
  \end{align*}
  and thus
  \begin{align*}
    N^+(\ep, \sum_{i} P_{\Omega_i} T P_{\Omega_i})
    \geq N^+(2\ep, P_{S} T P_{S})
    -\sum_{i < j} N^+\left( \frac{1}{C_m^2}\ep,  P_{\Omega_i} T P_{\Omega_j} + P_{\Omega_j} T P_{\Omega_i} \right).
  \end{align*}

  Noticing the fact that $\Omega_i$ are disjoint and isometric with $\Omega$, for the left hand side we obtain
  \begin{align*}
    N^+(\ep, \sum_{i} P_{\Omega_i} T P_{\Omega_i}) = \sum_{i} N^+(\ep, P_{\Omega_i} T P_{\Omega_i})
    = m N^+(\ep, P_{\Omega} T P_{\Omega}).
  \end{align*}
  On the other hand, by \cref{lem:EDRS_MainLemma},
  \begin{align*}
    N^+\left( \frac{1}{C_m^2}\ep,  P_{\Omega_i} T P_{\Omega_j} + P_{\Omega_j} T P_{\Omega_i} \right)
    = o\left( N^+\left(\frac{1}{C_m^2}\ep, T\right) \right) = o\left( N^+(\ep,T) \right),
  \end{align*}
  where we notice that $N^+(c \ep, T) \asymp N^+(\ep,T)$ for fixed $c > 0$ by (a) in \cref{cond:EDR}.
  Plugging in the two estimation and using $N^+(\ep,T) \asymp  N^+(\ep,  P_{S} T P_{S})$, we obtain
  \begin{align*}
    N^+(\ep, P_{\Omega} T P_{\Omega})
    & \geq \frac{1}{m}N^+(2\ep,  P_{S} T P_{S}) - o\left( N^+(\ep,T) \right)
    \geq c N^+(2 \ep,T) - o\left( N^+(\ep,T) \right) \\
    & \geq c N^+(\ep,T) - o\left( N^+(\ep,T) \right),
  \end{align*}
  which proves the desired lower bound.
\end{proof}

After all these preparation, we can prove \cref{thm:EDRS}:

\paragraph{Proof of \cref{thm:EDRS}}
Let $T$ be the integral operator associated with $k$.
We start with the case of $\rho = \bm{1}_{S}$ is the indicator of an open set $S$ and $\mu_n = \tilde{\mu}_n$.
Then from \cref{prop:ScaleMeasureEquiv} it suffices to consider $P_S T P_S$.
Since the asymptotic behavior of $N^+(\ep,A)$ determines uniquely $\lambda_i(A)$,
it suffices to prove that $N^+(\ep,P_{S} T P_{S}) \asymp N^+(\ep,T)$.

We take the sequence $U_0,V_0,U_1,V_1,\dots \subseteq \bbS^d$ of subdomains given in \cref{prop:C_SphereDecomp}
and prove that $N^+(\ep,P_{U_{ii}} T P_{U_{i}}) \asymp N^+(\ep,T)$ by induction.
The initial case follows from $U_0 = \bbS^d$.
Suppose $N^+(\ep,U_i) \asymp N^+(\ep,T)$, by \cref{lem:C_EigenCountAsympSubdomain}
and the fact that there are isometric copies of $V_i$ whose disjoint union is $U_i$, we obtain
$N^+(\ep,P_{V_i} T P_{V_i}) \asymp N^+(\ep,T).$
Moreover, since $V_i \subseteq U_{i+1}$, by \cref{cor:EigenCountUpperSubdomain} again,
we have
\begin{align*}
  N^+(\ep,P_{V_i} T P_{V_i}) \leq N^+(\ep,P_{U_{i+1}} T P_{U_{i+1}}) \leq N^+(\ep,T)
\end{align*}
and thus $N^+(\ep,P_{U_{i+1}} T P_{U_{i+1}}) \asymp N^+(\ep,T)$.

Now we have shown that $N^+(\ep,P_{U_i} T P_{U_{i}}) \asymp N^+(\ep,T)$.
Since $S$ is an open set and $\operatorname{diam}~U_i \to 0$,
we can find some $U_i \subseteq S$, and hence
$N^+(\ep,P_{S} T P_{S}) \asymp N^+(\ep,T) $ by \cref{cor:EigenCountUpperSubdomain}.

For the general case of $\mu_n$,
let $T_-$ and $T_+$ be the integral operators defined similarly to \cref{eq:C_Mercer} by the sequences
$c_1 \tilde{\mu}_n$ and $c_2 \tilde{\mu}_n$ respectively.
Then, $T_- \preceq T \preceq T_+$ and thus $P_{\Omega} T_- P_{\Omega}\preceq P_{\Omega} T P_{\Omega} \preceq P_{\Omega} T_+ P_{\Omega}$,
implying that
\begin{align*}
  \lambda_i\left(P_{S} T_- P_{S}\right)
  \leq \lambda_i\left(P_{S} T P_{S}\right)
  \leq \lambda_i\left(P_{S} T_+ P_{S}\right)
\end{align*}
and the results are obtained immediately from the previous case.

Finally, suppose $\rho$ is a bounded Riemann-integrable function that is non-zero.
The upper bound is proven by \cref{lem:ScaledKernel} with boundedness of $\rho$.
For the lower bound, we assert that there is an open set $\Omega$ such that
$\rho(\x)^2 \geq c > 0$ on $\Omega$.
Then, by \cref{lem:ScaledKernel} again, we conclude that
\begin{align*}
  \lambda_i(k;\bbS^d, \rho^2 \dd \sigma) \geq
  \lambda_i(k; \Omega,\rho^2 \dd \sigma) \geq c \lambda_i(k;\Omega,\dd \sigma)
  \asymp \lambda_i(k;\bbS^d, \dd \sigma)
\end{align*}
Now we prove the assertion.
Since $\rho$ is Riemann-integrable, the set of discontinuity is a null-set.
If $\rho(\x) = 0$ for all the continuity point, then $\rho(\x) = 0, $ a.e., which contradicts to the assumption that $\rho$ is non-zero.
So there is a continuity point $\x_0$ such that $\rho(\x_0) > 0$, and $\Omega$ can be taken as a small neighbour of $\x_0$.

\subsection{Discussion on \cref{cond:EDR}}
\label{subsec:EDR_Cond_Discussion}

In this subsection, we discuss some sufficient conditions that \cref{cond:EDR} holds.
The first proposition shows the basic relation between the difference and the derivative if $\mu_n$ is given by a function $f$,
which is a direct consequences of \cref{eq:A_Difference_Derivative}

\begin{proposition}
  Suppose $\mu_n = f(n)$ for some function $f(x)$ defined on $\R_{\geq 0}$.
  Then,
  \begin{enumerate}[(1)]
    \item If $(-1)^p f^{(p)}(x) \geq 0,~\forall x \geq N_0$,  then $\triangle^p \mu_n \geq 0,~\forall n \geq N_0$.
    \item If $(-1)^{p+1}f^{(p+1)}(x) \geq 0,~\forall x \geq N_0$, then
    $\triangle^p \mu_n \leq f^{(p)}(n),~\forall n \geq N_0$.
  \end{enumerate}
\end{proposition}

The next lemma shows that bounding the highest order term in \cref{eq:EDRDerivativeBound} is sufficient.

\begin{lemma}
  If $\binom{n+d}{d} \triangle^d \mu_n \leq B_n$ holds for a decreasing sequence $B_n$ for all $n \geq 0$.
  Then, $\binom{n+l}{l} \triangle^l \mu_n \leq \frac{d}{l} B_n$  holds for all $1 \leq l \leq d$.
  Consequently, if $B_{qn} \leq D' \mu_n$ for some $q \in \bbN_+$ and $D' >0$, then \cref{eq:EDRDerivativeBound} holds.
\end{lemma}
\begin{proof}
  We prove the result by induction.
  Suppose the statement holds for $l+1$, then
  \begin{align*}
    \binom{n+l}{l} \triangle^l \mu_n
    &=\binom{n+l}{l} \sum_{k \geq n} \triangle^{l+1} \mu_k
    \leq \binom{n+l}{l}\sum_{k \geq n} \frac{d}{l+1} B_k \binom{k+l+1}{l+1}^{-1} \\
    &\leq \binom{n+l}{l} \frac{d}{l+1} B_n \sum_{k \geq n} \binom{k+l+1}{l+1}^{-1} \\
    &=  \binom{n+l}{l} \frac{d}{l+1} B_n \frac{n! (l+1)!}{l(n+l)!}
    = \frac{d}{l} B_n.
  \end{align*}
\end{proof}

Combining the previous two results yields the following corollary.

\begin{corollary}
  \label{cor:EDR_Cond_Derivative}
  Suppose $\mu_n = f(n)$ for some function $f(x)$ defined on $\R_{\geq 0}$.
  Then a sufficient condition that (b,c) in \cref{cond:EDR} holds for $n \geq N_0$ is that
  \begin{align}
    \label{eq:EDR_Cond_Derivative}
    (-1)^{d+1}f^{(d+1)}(x) \geq 0 \qq{and} (-1)^d x^d f^{(d)}(qx) \leq D' f(x),\quad\forall x \geq N_0
  \end{align}
  for some $q \in \bbN_+$ and $D' > 0$.
\end{corollary}

\begin{proposition}
  For each of the following formulations of $\mu_n$,
  there is a sequence $(\mu_n)_{n\geq 0}$ that \cref{cond:EDR} is satisfied and the formulation holds when $n$ is sufficiently large.
  \begin{itemize}
    \item $\mu_n = c_0 n^{-\beta}$ for $c_0 > 0$ and $\beta > d$;
    \item $\mu_n = c_0 \exp(-c_1 n^{\beta})$ for $c_0,c_1,\beta > 0$;
    \item $\mu_n = c_0 n^{-\beta} (\ln n)^p$ for $c_0 > 0$, $\beta > d$ and $p \in \R$, or $\beta = d$ and $p > 1$.
  \end{itemize}
\end{proposition}
\begin{proof}
  The condition (a) is obviously satisfied by these asymptotic rates.
  We verify (b,c) by \cref{cor:EDR_Cond_Derivative} when $n \geq N_0$ for some large $N_0$ and take a left extrapolation of $\mu_n$
  as in \cref{lem:LeftExtrapolation} so the conditions hold for all $n \geq 0$.
  \begin{itemize}
    \item For $\mu_n = f(n)= c_0 n^{-\beta}$, we have $(-1)^p f^{(p)}(x) = c_0 (\beta)_p x^{-(\beta+p)}$,
    where $(\beta)_p = \beta (\beta+1)\cdots (\beta+p-1)$, so \cref{eq:EDR_Cond_Derivative} holds for $q=1$ and $D' = (\beta)_p$.
    \item For $\mu_n = f(n) = c_0 \exp(-c_1 n^{\beta})$, it is easy to show that
    \begin{align*}
    (-1)
      ^p f^{(p)}(x) \asymp c_0 (c_1 \beta)^p x^{p(\beta-1)} \exp(-c_1 x^{\beta}) \qq{as} x \to \infty,
    \end{align*}
    so \cref{eq:EDR_Cond_Derivative} holds if we take $q = 2$ since the exponential term is dominating.
    \item For $\mu_n = f(n) = c_0 n^{-\beta} (\ln n)^p$, we have
    \begin{align*}
    (-1)
      ^p f^{(p)}(x) \asymp c_0 (\beta)_p x^{-(\beta+p)} (\ln x)^{p} \qq{as} x \to \infty,
    \end{align*}
    so \cref{eq:EDR_Cond_Derivative} still holds for $q = 1$.
  \end{itemize}
\end{proof}

  \section{Conclusion}
  \label{sec:conclusion}

In this paper, we develop a novel approach for determining the eigenvalue decay rate (EDR) of certain kernels using transformation and restriction.
Using this approach, we determine the EDR of the NTKs associated with multilayer fully-connected ReLU neural networks on a general domain.
Combining this result with the uniform approximation of the neural network by the NTK regression,
we determine the generalization performance of the over-parameterized neural network through the kernel regression theory.
The theoretical results show that proper early stopping is essential for the generalization performance of the neural networks,
which urges us to scrutinize the widely reported ``benign overfitting phenomenon'' in deep neural network literature.

For future directions, it is natural to extend our results to the NTKs associated with other neural network architectures,
such as convolutional neural networks and residual neural networks.
Also, it would be of great interest to see if these results can be extended to the large dimensional data where $d \propto n^{s}$ for some $s>0$ instead of the fixed $d$ here.

  \acks{
    The authors thank the anonymous reviewers for their valuable suggestions.
    Lin’s research is supported in part by the National Natural Science Foundation of China (Grant 92370122, Grant 11971257).
  }

  \newpage

  \appendix

  \section{Uniform Convergence of the Neural Network to Kernel Regression}\label{sec:A_NN}
  In this section we will prove \cref{lem:UnifConverge}.
Applying Proposition 3.2 and Proposition 3.3 in \citet{lai2023_GeneralizationAbility},
it suffices to show \cref{prop:2_KernelUniform}, that is, the kernel $K_t$ converges uniformly to $\NTK$.
The rest of this section is organized as follows:
We first introduce some more preliminaries;
in \cref{subsec:Init}, we discuss some properties of the network at initialization;
in \cref{subsec:Training}, we analyze the effect of small perturbations during the training process of the network;
we then prove the lazy regime approximation of the neural network in \cref{subsec:Lazy};
we also show the Hölder continuity of $\NTK$ in \cref{subsec:Holder_NTK};
finally, we prove the kernel uniform convergence in \cref{subsec:KernelUniformConvergence}.

\paragraph{Further notations}
Let us denote $\tilde{B}_R = \left\{ x \in \R^d : \tilde{x} \leq R \right\}$ for $R\geq 1$.
For a vector $\bm{v}=(v_1, v_2, \cdots, v_m)\in\mb{R}^m$, we use $\norm{\bm{v}}_2$ (or simply $\norm{\bm{v}}$) to represent the Euclidean norm.
Additionally, if we have a univariate function $f:\mb{R}\to\mb{R}$, we define $f(\bm{v}) = (f(v_1), f(v_2),\cdots, f(v_m))\in\mathbb{R}^m$.
We denote by $\norm{\bm{M}}_2$ and $\norm{\bm{M}}_{\mr{F}}$ the spectral and Frobenius norm of a matrix $\bm{M}$ respectively.
Also, we use $\norm{\,\cdot\,}_0$ to represent the number of non-zero elements of a vector or matrix.
For matrices $\bm{A}\in\mb{R}^{n_1\times n_2}$ and $\bm{B}\in\mb{R}^{n_2\times n_1}$,
we define $\ang{\bm{A},\bm{B}} = \Tr(\bm{A}\bm{B}^T)$.
We remind that $\ang{\bm{M}, \bm{M}} = \norm{\bm{M}}_{\mr{F}}^2$ in this way.

\paragraph{Network Architecture} Let us recall the neural network in the main text.
Since it can be shown easily that the bias term $\bm{b}^{(0,p)}$ in the first layer can be absorbed into $\bm{A}^{(p)}$
if we append an $1$ at the last coordinate of $\x$,
we denote $\bm{W}^{(0,p)}=(\bm{A}^{(p)}~\bm{b}^{(0,p)})$, $\tilde{\x}=(\x^T,1)^T\in\mb{R}^d \times \{1\}\subset\mb{R}^{d+1}$
and consider the following equivalent neural network:
\begin{align}
  \label{eq:A_NN_Arch}
  \begin{split}
    \bm{\alpha}^{(0,p)}(\x)&=\widetilde{\bm{\alpha}}^{(0,p)}(\x)=\tilde{\x}\in\mb{R}^{d+1},\\
    \widetilde{\bm{\alpha}}^{(l,p)}(\x) & = \sqrt{\tfrac{2}{m_{l}}} \bm{W}^{(l-1,p)}{\bm{\alpha}}^{(l-1,p)}(\x)\in\mb{R}^{m_l},\\
    {\bm{\alpha}}^{(l,p)}(\x)&=\sigma\xkm{\widetilde{\bm{\alpha}}^{(l,p)}(\x)}\in\mb{R}^{m_l},\\
    g^{(p)}(\x;\bm{\theta}) & = \bm{W}^{(L,p)}{\bm{\alpha}}^{(L,p)}(\x)+b^{(L,p)}\in\mb{R},\\
    f(\x;\bm{\theta}) &= \frac{\sqrt{2}}{2}\zk{ g^{(1)}(\x;\bm{\theta}) - g^{(2)}(\x;\bm{\theta}) }\in\mb{R}.
  \end{split}
\end{align}
for $p=1,2$ and $l=1,\cdots,L$.
Recall that the integers $m_1,m_2,\cdots, m_L$ are the width of $L$-hidden layers and $m_{L+1} = 1$ is the width of output layer.
Additionally, we have set $m = \min(m_1,m_2,\cdots,m_L)$ and made the assumption that $\max(m_1,m_2,\dots,m_{L}) \leq C_{\mathrm{m}} m$ for some absolute constant $C_{\mathrm{m}}$.
By setting $m_0=d+1$ for convenience, we have $\bm{W}^{(l,p)} \in \R^{m_{l+1} \times m_l}$ for $l\in\{0,1,2,\cdots,L\}$.

For $p=1,2$, we define $g^{(p)}_t(\x) = g^{(p)}(\x;\bm{\theta}_t)$ and $f_t(\x) = f(\x;\bm{\theta}_t)$.
Similarly, we also add a subscript $t$ for all the related quantities (including those defined afterwards) to indicate their values at time $t$ during the training process.

\paragraph{The neural tangent kernel}
Let us consider the following neural network kernel
\begin{align*}
K_t(\x,\x') = \ang{\nabla_{\bm{\theta}} f(\x;\bm{\theta}_t),\nabla_{\bm{\theta}} f(\x';\bm{\theta}_t)}.
\end{align*}
Then, the gradient flow can be written as~\citep{jacot2018_NeuralTangent}
\begin{align*}
\dot{f}(\x;\bm{\theta}_t) = - \frac{1}{n} K_t(\x,\X)\xk{f(\X;\bm{\theta}_t) - \bm{y}}.
\end{align*}
As shown in \citet{jacot2018_NeuralTangent}, as the width $m$ goes to infinity, the kernel $K_t$ converges to the deterministic neural tangent kernel $\NTK$.
With the mirrored architecture given by \cref{eq:A_NN_Arch}, we can express the kernel $K_t(\x,\x')$ as follows:
\begin{align*}
  K_t(\x,\x')
  &=  \ang{\nabla_{\bm{\theta}} f(\x;\bm{\theta}_t), \nabla_{\bm{\theta}} f(\x;\bm{\theta}_t)} =\sum_{p=1}^2 \ang{\nabla_{\bm{\theta}^{(p)}} f(\x;\bm{\theta}_t), \nabla_{\bm{\theta}^{(p)}}f(\x;\bm{\theta}_t)} \\
&= \frac{1}{2} \sum_{p=1}^2 \ang{\nabla_{\bm{\theta}^{(p)}} g^{(p)}_t(\x),\nabla_{\bm{\theta}^{(p)}} g^{(p)}_t(\x')}=  \frac{1}{2} \xk{K_t^{(1)}(\x,\x') + K_t^{(2)}(\x,\x') },
\end{align*}
where $K_t^{(p)}(\x,\x')=\ang{\nabla_{\bm{\theta}^{(p)}} g^{(p)}_t(\x),\nabla_{\bm{\theta}^{(p)}} g^{(p)}_t(\x')}$ is the neural network kernel of $g_t^{(p)}$ for $p=1,2$, which is a vanilla neural network.
Consequently, due to the mirror initialization, we have
$K_0^{(1)}(\x,\x') = K_0^{(2)}(\x,\x')= K_0(\x,\x')$.

\paragraph{An expanded matrix form}
Sometimes it is convenient to write the neural network \cref{eq:A_NN_Arch} in an expanded matrix form as introduced in \citet{allen-zhu2019_ConvergenceTheory}.
Let us define the activation matrix
\begin{align*}
  \bm{D}^{(l,p)}_{\x}=
  \begin{cases}
    \bm{I}_{d+1}, &l=0;\\
    \mr{diag}\xkm{\dot\sigma\xkm{\widetilde{\bm{\alpha}}^{(l,p)}\xkm{\x}}}, &l \geq 1,
  \end{cases}\quad\in\mb{R}^{m_l\times m_l}
\end{align*}
for $p=1,2$, where $\bm{I}_{d+1}$ is the identity matrix.
Then, we have
\begin{align*}
  \widetilde{\bm{\alpha}}^{(l,p)}(\x) &= \sqrt{\frac{2}{m_l}} \bm{W}^{(l-1,p)}\bm{D}^{(l-1,p)}_{\x}\widetilde{\bm{\alpha}}^{(l-1,p)}(\x),\quad
  {\bm{\alpha}}^{(l,p)}(\x)
= \sqrt{\frac2{m_l}}\bm{D}^{(l,p)}_{\x}\bm{W}^{(l-1,p)}{\bm{\alpha}}^{(l-1,p)}(\x).
\end{align*}

To further write it as product of matrices, we first introduce the following notation to avoid confusion since matrix product is not commutative.
For matrices $\bm{A}_0,\bm{A}_1,\cdots,\bm{A}_L$, we define the left multiplication product
\begin{align*}
  \prod_{i=a}^b \bm{A}_i =
  \begin{cases}
    1,&0\leq b <a \leq L;\\
    \bm{A}_b \bm{A}_{b-1} \cdots \bm{A}_{a+1} \bm{A}_a,& 0 \leq a \leq b \leq L.
  \end{cases}
\end{align*}
Since real number multiplication is commutative, the notation introduced above is compatible with the traditional usage when $\bm{A}_0,\bm{A}_1, \cdots, \bm{A}_L$ degenerate into real numbers.
In this way, we have
\begin{align}
  \label{eq:NN_Layer_MatrixProductForm}
  \widetilde{\bm{\alpha}}^{(l,p)}(\x) = \xk{\prod_{r=1}^l \sqrt{\frac{2}{m_r}} \bm{W}^{(r-1,p)}\bm{D}^{(r-1,p)}_{\x}} \tilde{\x}, \quad
  {\bm{\alpha}}^{(l,p)}(\x) = \xk{\prod_{r=1}^{l} \sqrt{\frac{2}{m_r}} \bm{D}^{(r,p)}_{\x} \bm{W}^{(r-1,p)}}\tilde{\x}.
\end{align}
Using the above expressions, we can obtain
\begin{align*}
  g^{(p)}(\x) &= \bm{W}^{(L,p)}{\bm{\alpha}}^{(L,p)}(\x) + b^{(L,p)}=\bm{W}^{(L,p)}\xk{\prod_{r=l+1}^{L} \sqrt{\frac{2}{m_r}} \bm{D}^{(r,p)}_{\x} \bm{W}^{(r-1,p)}}{\bm{\alpha}}^{(l,p)}(\x) + b^{(L,p)}\\
  &=\xk{ \prod_{r=l+1}^{L} \sqrt{\frac{2}{m_{r}}} \bm{W}^{(r,p)}\bm{D}^{(r,p)}_{\x} }\bm{W}^{(l,p)} {\bm{\alpha}}^{(l,p)}(\x) + b^{(L,p)}.
\end{align*}
Finally, we use the above results to calculate the gradient $\nabla_{\bm{W}^{(l,p)}} g^{(p)}(\x)$.
To simplify notation, we define:
\begin{align}
  \label{eq:def_gamma}
  \begin{split}
    \begin{gathered}
      \widetilde{\bm{\alpha}}^{(l,p)}_{\x}=\widetilde{\bm{\alpha}}^{(l,p)}\xkm{\x},
      \quad{\bm{\alpha}}^{(l,p)}_{\x}={\bm{\alpha}}^{(l,p)}\xkm{\x},
      \quad\bm{\gamma}^{(l,p)}_{\x}=\xk{\prod_{r=l+1}^{L} \sqrt{\frac2{m_r}}  \bm{W}^{(r,p)} \bm{D}^{(r,p)}_{\x}}^T\in\mb{R}^{m_{l+1}}.
    \end{gathered}
  \end{split}
\end{align}
Then we can obtain
\begin{align*}
  g^{(p)}(\x)=\bm{\gamma}^{(l,p),T}_{\x}\bm{W}^{(l,p)}{\bm{\alpha}}^{(l,p)}_{\x}+b^{(L,p)},
\end{align*}
which can lead to
\begin{align}
  \label{eq:GradientExpr}
  \nabla_{\bm{W}^{(l,p)}} g^{(p)}(\x)=\bm{\gamma}^{(l,p)}_{\x}{\bm{\alpha}}^{(l,p),T}_{\x},\quad l=0,1,\cdots,L,~p=1,2.
\end{align}
Also, it is worth noting that for two vectors $\bm{a}$ and $\bm{b}$, we have
\begin{align*}
  \norm{\bm{a} \bm{b}^T}_{\mathrm{F}}^2 = \Tr(\bm{a} \bm{b}^T \bm{b} \bm{a}^T) = \Tr(\bm{a}^T \bm{a} \bm{b}^T \bm{b}) = \norm{\bm{a}}_2^2 \norm{\bm{b}}_2^2.
\end{align*}
Consequently, we can get
\begin{align}
  \label{eq:GradientExpr_Fnorm}
  \norm{\nabla_{\bm{W}^{(l,p)}} g^{(p)}(\x)}_{\mr{F}}=\norm{\bm{\gamma}^{(l,p)}_{\x}}_2\norm{{\bm{\alpha}}^{(l,p)}_{\x}}_2.
\end{align}

\subsection{Initialization}\label{subsec:Init}
Since our neural network is mirrored, we can focus only on one part $g^{(p)}(\x)$ of the network at initialization.
For notational simplicity, we omit the superscript $p$ for $\bm{W}^{(l,p)}_t$ and other notations in the following if there is no ambiguity.
And unless otherwise stated, it is understood that the conclusions hold for both $p = 1$ and $p = 2$.

Since $K^{(p)}_0$ corresponds to the tangent kernel of a vanilla fully connected neural network,
\citet[Theorem 3.1]{arora2019_ExactComputation} shows the following convergence result.

\begin{lemma}
[Convergence to the NTK at initialization]
  There exist some positive absolute constants $C_1>0$ and $C_2\geq 1$ such that if $\varepsilon\in\xk{0,1}$, $\delta\in\xk{0,1}$ and $m\geq C_1\varepsilon^{-4}\ln\xkm{C_2/\delta}$, then for any fixed $\bm{z}, \bm{z}'$ such that $\norm{\bm{z}} \leq 1$ and $\norm{\bm{z}'}\leq 1$, with probability at least $1-\delta$ with respect to the initialization, we have
  \[\abs{K^{(p)}_0\xkm{\bm{z},\bm{z}'}-\NTK\xkm{\bm{z},\bm{z}'}}\leq\varepsilon.\]

\end{lemma}

Letting $\varepsilon=m^{-1/5}$ in the previous lemma, we can get the following corollary:
\begin{corollary}
  \label{cor:Init_Kt_NTK}

  There exist some positive absolute constants $C_1>0$ and $C_2\geq 1$ such that if $\delta\in\xk{0,1}$ and $m\geq C_1\xk{\ln\xkm{C_2/\delta}}^5$, 
  then for any fixed $\bm{z},\bm{z}'\in \tilde{B}_R$, with probability at least $1-\delta$ with respect to the initialization, we have
  \[\abs{K^{(p)}_0\xkm{\bm{z},\bm{z}'}-\NTK\xkm{\bm{z},\bm{z}'}}=O\xkm{ R^2 m^{-1/5}}.\]

\end{corollary}

Now we further provide some bounds about the magnitudes of weight matrices and layer outputs.
The following is a standard estimation of Gaussian random matrix, which is a direct consequence of \citet[Corollary 5.35]{vershynin2010_IntroductionNonasymptotic}.
\begin{lemma}
  \label{lem:Random matrix}At initialization, there exists a positive absolute constant $C$, 
  such that when $m\geq C$, with probability at least $1-\exp\xkm{-\Omega(m)}$ with respect to the initialization,
we have
  \begin{align*}
    \norm{\bm{W}^{(l)}_0}_2 = O\xkm{\sqrt{m}},\quad l\in \dk{0,1,\dots,L}.
  \end{align*}
\end{lemma}

Noting that $\norm{\bm{D}^{(l)}_{\x}}_2 \leq 1$ and combining \cref{lem:Random matrix} with \cref{eq:NN_Layer_MatrixProductForm}, \cref{eq:def_gamma} and \cref{eq:GradientExpr_Fnorm}, we have:

\begin{lemma}
  \label{lem:Init_WeightsProductBound}There exists a positive absolute constant $C$, 
  such that when $m\geq C$, with probability at least $1-\exp(-\Omega(m))$ with respect to the initialization, for any
$l\in\dk{0,1,\cdots,L}$ and $\x \in \tilde{B}_R$, we have
  \begin{align*}
    \begin{gathered}
      \norm{\widetilde{\bm{\alpha}}_{\x,0}^{(l)}}_2 = O(R),\quad~\norm{{\bm{\alpha}}_{\x,0}^{(l)}}_2 = O(R),\quad~\norm{\bm{\gamma}^{(l)}_{\x,0}}_2=O(1)
      \quad~\text{and}\quad~
      \norm{\nabla_{\bm{W}^{(l)}} g_0(\x) }_{\mr{F}} = O(R).
    \end{gathered}
\end{align*}
\end{lemma}

\cref{lem:Random matrix} and \cref{lem:Init_WeightsProductBound} provide some upper bounds, and the subsequent lemma provides a lower bound.
It is important to note that the previous lemma holds uniformly for $\x\in\tilde{B}_R$, while the following lemma only holds pointwisely.

\begin{lemma}[Lemma 7.1 in \citet{allen-zhu2019_ConvergenceTheory}]
  \label{lem:Init_LayerOutputBounds}
  There exists a positive absolute constant $C$ such that when $m\geq C$, for any fixed $\bm{z}\in\tilde{B}_R$, 
  with probability at least $1-\exp\xkm{-\Omega(m)}$ with respect to the initialization, we have $\norm{{\bm{\alpha}}^{(l)}_{\bm{z},0}}_{2} =\Theta(R)$ for $l\in\dk{0,1,\cdots,L}$.

\end{lemma}

\subsection{The training process}\label{subsec:Training}

In this subsection we will show that as long as the parameters and input do not change much, some related quantities can also be bounded.
We still focus on one parity in this subsection and suppress the superscript $p$ for convenience.

The most crucial result we will obtain in this subsection is the following proposition:
\begin{proposition}
  \label{prop:Training_KernelBound}

  Fix $ \bm{z},\bm{z}' \in \tilde{B}_R$ and $T\subseteq[0,\infty)$.
  Suppose that $\norm{\bm{W}^{(l)}_t - \bm{W}^{(l)}_0}_{\mr{F}}= O(m^{1/4})$ holds for all $t\in T$ and $l\in\dk{0,1,\cdots,L}$.
  Then there exists a positive absolute constant $C$ such that when $m\geq C$, with probability at least $1 - \exp\xkm{-\Omega(m^{5/6})}$, for any $\x,\x' \in \mc{\X}$ such that $\norm{\x-\z}_2,\norm{\x'-\z'}_2 \leq O( 1/m)$, we have
  \[\sup_{t\in T}\abs{K^{(p)}_t(\x,\x') - K^{(p)}_0(\z,\z')}= O \xkm{R^2 m^{-{1}/{12}}\sqrt{\ln m}}.\]
\end{proposition}

The proof of this proposition will be presented at the end of this subsection.
Combining this proposition with \cref{cor:Init_Kt_NTK}, we can derive the following corollary:
\begin{corollary}
  \label{cor:Training_Kernel_uniform}
  Fix $ \bm{z},\bm{z}' \in \tilde{B}_R$ and let $\delta\in(0,1)$, $T\subseteq[0,\infty)$.
  Suppose that $\norm{\bm{W}^{(l)}_t - \bm{W}^{(l)}_0}_{\mr{F}}= O(m^{1/4})$ holds for all $t\in T$ and $l\in\dk{0,1,\cdots,L}$.
  Then there exist some positive absolute constants $C_1>0$ and $C_2\geq 1$ such that with probability at least $1 - \delta$, for any $\x,\x' \in \mc{\X}$ such that $\norm{\x-\z}_2,\norm{\x'-\z'}_2 \leq O( 1/m)$, we have
  \[\sup_{t\in T}\abs{K^{(p)}_t(\x,\x') - \NTK(\z,\z')}= O \xkm{R^2 m^{-{1}/{12}}\sqrt{\ln m}},~\text{when}~m\geq C_1\xk{\ln(C_2/\delta)}^5.\]
\end{corollary}

To prove \cref{prop:Training_KernelBound}, we need to introduce some necessary lemmas.
In \cref{lem:Init_WeightsProductBound}, we have provided upper bounds for the norms of ${\widetilde{\bm{\alpha}}_{\x,0}^{(l)}}$, ${{\bm{\alpha}}_{\x,0}^{(l)}}$, $\bm{\gamma}^{(l)}_{\x,0}$ and $\nabla_{\bm{W}^{(l)}} g_0(\x)$ at initialization.
Next, we aim to prove that under perturbations in the parameters and the input, the corresponding changes in these quantities will also be small.

In fact, similar lemmas can be found in \citet{allen-zhu2019_ConvergenceTheory}, although they have different conditions compared to the lemmas in this paper.
For example, in \citet{allen-zhu2019_ConvergenceTheory}, the input points are constrained to lie on a sphere, the input and output layers are not involved in training, and each hidden layer has the same width.

However, the most crucial point is that \citet{allen-zhu2019_ConvergenceTheory} did not consider the impact of small perturbations in the input, which is vital for proving uniform convergence.
In fact, the slight perturbation between $\x$ and $\z$ can be regarded as taking a slight perturbation on $\bm{W}^{(0)}$, with other $\bm{W}^{(l)} $ fixed.
Additionally, since this paper fixes the number of layers $L$, there is no need to consider the impact of $L$ on the bounds.
This simplifies the proof of the corresponding conclusions.

Inspired by \citet[Lemma 8.2]{allen-zhu2019_ConvergenceTheory}, we can prove the following lemma:

\begin{lemma}
  \label{lem:NN_forward_perturbation}
  Let $\Delta=O(1/\sqrt{m})$, $\tau\in\zk{\Delta\sqrt{m},O\xkm{{\sqrt{m}}/{(\ln m)^3}}}$, $T\subseteq [0,\infty)$ and fix ${\bm{z}} \in\tilde{B}_R$.
  Suppose that $\norm{\bm{W}^{(l)}_t - \bm{W}^{(l)}_0 }_{\mr{F}}\leq\tau$ holds for all $t\in T$ and $l\in\dk{0,1,\cdots,L}$.
  Then there exists a positive absolute constant $C$ such that with probability at least $1-\exp\xkm{-\Omega(m^{2/3}\tau^{2/3})}$, for all $t\in T$, $l\in\dk{0,1,\cdots,L}$ and $\x \in\tilde{B}_R$ such that $\norm{\x-\bm{z}}_2 \leq \Delta$, we have

  $(a)$ $\norm{\widetilde{\bm{\alpha}}^{(l)}_{\x,t} - \widetilde{\bm{\alpha}}_{\z,0}^{(l)}}_2 = O(R \tau/\sqrt{m})$ and thus $\norm{\widetilde{\bm{\alpha}}^{(l)}_{\x,t}}_2=O(R)$;

  $(b)$ $\norm{\bm{D}_{\x,t}^{(l)} - \bm{D}_{{\bm{z}},0}^{(l)}}_0 = O(m^{2/3} \tau^{2/3})$ and $\norm{\xk{\bm{D}_{\x,t}^{(l)} - \bm{D}_{{\bm{z}},0}^{(l)}}\widetilde{\bm{\alpha}}_{\x,t}^{(l)}}_2= O(R{\tau}/{\sqrt{m}})$;

  $(c)$ $\norm{{\bm{\alpha}}^{(l)}_{\x,t} - {\bm{\alpha}}_{\z,0}^{(l)}}_2 = O(R \tau/\sqrt{m})$ and thus $\norm{{\bm{\alpha}}^{(l)}_{\x,t}}_2=O(R)$,

\noindent when $m$ is greater than the positive constant $C$.
\end{lemma}
\begin{proof}
  We use mathematical induction to prove this lemma.
  When $l=0$, it can be easily verified that $\bs{D}_{\bs{x},t}^{(0)} - \bs{D}_{\bs{z},0}^{(0)} = \bm{O}$ and $\norm{\widetilde{\bm{\alpha}}^{(l)}_{\x,t} - \widetilde{\bm{\alpha}}_{\z,0}^{(l)}}_2=\norm{{{\bm{\alpha}}}_{\x,t}^{(0)} -{{\bm{\alpha}}}^{(0)}_{\bm{z},0}}_2=\norm{\tilde{\x}-\tilde{\z}}_2 = \norm{\x-{\bm{z}}}_2 \leq \Delta\leq \tau/\sqrt{m}$, where $\bm{O}$ represents the zero matrix, $\tilde{\x}$ and $\tilde{\bm{z}}$ are extended vectors with an additional coordinate of $1$.
  Thus, all the statements hold for $l=0$.
  Now we assume that this lemma holds for $l=k\in\dk{0,1,\cdots,L-1}$.

  $(a)$  First of all, we can decompose $\widetilde{\bm{\alpha}}^{(k+1)}_{\x,t} -\widetilde{\bm{\alpha}}^{(k+1)}_{\z,0}$ as following:
  \begin{equation*}
    \begin{aligned}
      &\widetilde{\bm{\alpha}}^{(k+1)}_{\x,t} -\widetilde{\bm{\alpha}}^{(k+1)}_{\z,0}=
      \sqrt{\frac{2}{m_{k+1}}} \bm{W}^{(k)}_t {\bm{\alpha}}^{(k)}_{\x,t} - \sqrt{\frac{2}{m_{k+1}}} \bm{W}_0^{(k)}{\bm{\alpha}}^{(k)}_{\z,0}\\
      &\qquad=\sqrt{\frac{2}{m_{k+1}}} \xk{\bm{W}^{(k)}_t - \bm{W}^{(k)}_0} {\bm{\alpha}}^{(k)}_{\x,t} + \sqrt{\frac{2}{m_{k+1}}}\bm{W}^{(k)}_0 \xk{{\bm{\alpha}}^{(k)}_{\x,t} - {\bm{\alpha}}^{(k)}_{\z,0}}.
    \end{aligned}
  \end{equation*}
  From the above equation, we can deduce that $(a)$ holds for $l=k+1$ by the induction hypothesis and \cref{lem:Init_WeightsProductBound}.

  $(b)$
Let us consider the following choices in
  \cref{lem:Allen_Claim 8.3}:
  \begin{align*}
    \bm{g}&=\frac{\widetilde{\bm{\alpha}}^{(k+1)}_{\z,0}}{\sqrt{{2}/{m_{k+1}}}\norm{{\bm{\alpha}}_{\z,0}^{(k)}}_2}=\frac{\bm{W}_0^{(k)}{\bm{\alpha}}_{\z,0}^{(k)}}{\norm{{\bm{\alpha}}_{\z,0}^{(k)}}_2},\qquad~
    \bm{g}'=\frac{\widetilde{\bm{\alpha}}^{(k+1)}_{\x,t}-\widetilde{\bm{\alpha}}^{(k+1)}_{\z,0}}{\sqrt{{2}/{m_{k+1}}}\norm{{\bm{\alpha}}_{\z,0}^{(k)}}_2}.
  \end{align*}
  It follows that $\bm{g}\sim N(0,\bm{I})$ if we fix ${\bm{\alpha}}_{\z,0}^{(k)}$ and only consider the randomness of $\bm{W}^{(k)}_0$.
  Also, we have $\norm{\bm{g}'}_2\leq O(\tau/\sqrt{m})\cdot O(\sqrt{m})\leq O(\tau)$ holds for all $\x \in\tilde{B}_R$ such that $\norm{\x-\bm{z}}_2 \leq \Delta$ since we have previously shown that $\norm{{\bm{\alpha}}_{\z,0}^{(k)}}_2=\Theta(R)$ in \cref{lem:Init_LayerOutputBounds}.
  Therefore, we can choose $\delta=\Theta(\tau)$ such that $\norm{\bm{g}'}_2\leq\delta$.
  Then, we can obtain
  \begin{align*}
    \bm{g}+\bm{g}'=\frac{\widetilde{\bm{\alpha}}^{(k+1)}_{\x,t}}{\sqrt{\frac{2}{m_{k+1}}}\norm{{\bm{\alpha}}_{\z,0}^{(k)}}_2},~~\bm{D}'=\bm{D}_{\x,t}^{(k+1)} - \bm{D}_{{\bm{z}},0}^{(k+1)}~~\text{and}~~\bm{u}=\frac{\xk{\bm{D}_{\x,t}^{(k+1)} - \bm{D}_{{\bm{z}},0}^{(k+1)}}\widetilde{\bm{\alpha}}^{(k+1)}_{\x,t}}{\sqrt{\frac{2}{m_{k+1}}}\norm{{\bm{\alpha}}_{\z,0}^{(k)}}_2}&,
  \end{align*}
  where $\bm{D}'$ and $\bm{u}$ are defined as in \cref{lem:Allen_Claim 8.3}.
  By applying \cref{lem:Allen_Claim 8.3}, we can get $\norm{\bm{D}'}_0\leq O(m^{2/3}\tau^{2/3})$ and $\norm{\bm{u}}_2\leq O(\delta)$, which establish the conclusion of part $(b)$.

  $(c)$ Further, the third statement can be directly obtained from the following inequality:
  \begin{align*}
    \norm{{\bm{\alpha}}^{(k+1)}_{\x,t} -{\bm{\alpha}}^{(k+1)}_{\bm{z},0}}_2 &\leq \norm{\bm{D}^{(k+1)}_{{\bm{z}},0}\xk{\widetilde{\bm{\alpha}}^{(k+1)}_{\x,t} - \widetilde{\bm{\alpha}}^{(k+1)}_{\bm{z},0}}}_2+\norm{\xk{\bm{D}^{(k+1)}_{\x,t} - \bm{D}^{(k+1)}_{{\bm{z}},0} }\widetilde{\bm{\alpha}}^{(k+1)}_{\x,t} }_2.
  \end{align*}
  Thus, the proof of this lemma is complete.
\end{proof}

In the proof of \cref{lem:NN_forward_perturbation}, we use the following result:
\begin{lemma}[\citet{allen-zhu2019_ConvergenceTheory} Claim 8.3]
  \label{lem:Allen_Claim 8.3}

  Suppose each entry of $\bm{g} \in \mathbb{R}^m$ follows $g_i \iid {N}(0,1)$.
  For any $\delta > 0$, with probability at least $1-\exp\xkm{-\Omega(m^{2/3}\delta^{2/3})} $, the following proposition holds:

  Select $\bm{g}' \in \mathbb{R}^m$ such that $ \norm{\bm{g}'}_2\leq\delta.$
  Let $\bm{D}' = \mr{diag}(D'_{kk})$ be a diagonal matrix, where the $k$-th diagonal element
  $D_{kk}' $ follows
  $$D_{k k}^{\prime}=\bs{1}\left\{\left(g+g^{\prime}\right)_{k} \geq 0\right\}-\bs{1}\left\{g_{k} \geq 0\right\},\quad k\in[m].$$
  If we define $\bm{u} = \bm{D}'(\bm{g}+\bm{g}')$, then it satisfies the following inequalities:
  \[\norm{\bm{u}}_0 \leq \norm{\bm{D}'}_0 \leq O\xkm{m^{2/3}\delta^{2/3}} \quad \text{ and }\quad \norm{\bm{u}}_2 \leq O(\delta).\]
\end{lemma}

Inspired by \citet[Lemma 8.7]{allen-zhu2019_ConvergenceTheory}, we then have the following lemma:
\begin{lemma}
  \label{lem:backward perturbation}
  Let $\Delta=O(1/\sqrt{m})$, $\tau\in\zk{\Delta\sqrt{m},O\xkm{{\sqrt{m}}/{(\ln m)^3}}}$, $T\subseteq [0,\infty)$ and fix ${\bm{z}} \in\tilde{B}_R$.
  Suppose that $\norm{\bm{W}^{(l)}_t - \bm{W}^{(l)}_0 }_{\mr{F}}\leq\tau$ holds for all $t\in T$ and $l\in\dk{0,1,\cdots,L}$.
  Then there exists a positive absolute constant $C$ such that with probability at least $1-\exp\xkm{-\Omega(m^{2/3}\tau^{2/3})}$, for all $t\in T$, $l\in\dk{0,1,\cdots,L}$ and $\x \in\tilde{B}_R$ such that $\norm{\x-\bm{z}}_2 \leq \Delta$, we have
  \begin{align*}
    \norm{{\bm{\gamma}}^{(l)}_{\x,t} - {\bm{\gamma}}^{(l)}_{\z,0}}_2
    = O\xkm{m^{-1/6} \tau^{1/3}\sqrt{\ln m}}\qquad~\text{and thus}\qquad~ \norm{{\bm{\gamma}}^{(l)}_{\x,t}}_2=O(1)
  \end{align*}
\end{lemma}
when $m$ is greater than the positive constant $C$.

By using \cref{lem:NN_forward_perturbation} and \cref{lem:backward perturbation}, we can prove the following lemma:
\begin{lemma}
  \label{lem:Training_GradientBound}
  Let $\Delta=O(1/\sqrt{m})$, $\tau\in\zk{\Delta\sqrt{m},O\xkm{{\sqrt{m}}/{(\ln m)^3}}}$, $T\subseteq [0,\infty)$ and fix ${\bm{z}} \in\tilde{B}_R$.
  Suppose that $\norm{\bm{W}^{(l)}_t - \bm{W}^{(l)}_0 }_{F}\leq\tau$ holds for all $t\in T$ and $l\in\dk{0,1,\cdots,L}$.
  
  Then there exists a positive absolute constant $C$ such that with probability at least $1-\exp\xkm{-\Omega(m^{2/3}\tau^{2/3})}$, for all $t\in T$, $l\in\dk{0,1,\cdots,L}$ and $\x\in\tilde{B}_R$ such that $\norm{\x-\bm{z}}_2\leq \Delta$, we have
\begin{align*}
    \norm{\nabla_{\bm{W}^{(l)}} g_t(\x) - \nabla_{\bm{W}^{(l)}} g_0(\z)}_{\mr{F}} = O\xkm{R m^{-1/6} \tau^{1/3}\sqrt{\ln m}}~\text{and thus}~\norm{ \nabla_{\bm{W}^{(l)}} g_t(\x)}_{\mr{F}} = O(R).
  \end{align*}
  when $m$ is greater than the positive constant $C$.
\end{lemma}
\begin{proof}
Recalling \cref{eq:GradientExpr}, we have
  $\nabla_{\bm{W}^{(l)}} g_t(\x)=\bm{\gamma}^{(l)}_{\x,t}\
    {\bm{\alpha}}^{(l),T}_{\x,t}$.
  Then, we have
  \begin{align*}
    &\norm{\nabla_{\bm{W}^{(l)}} g_t(\x) - \nabla_{\bm{W}^{(l)}} g_0(\z)}_{\mr{F}}=\norm{\bm{\gamma}^{(l)}_{\x,t}{\bm{\alpha}}^{(l),T}_{\x,t} -\bm{\gamma}^{(l)}_{\z,0}{\bm{\alpha}}^{(l),T}_{\z,0}}_{\mr{F}}\notag \\
&\qquad \leq \norm{\bm{\gamma}^{(l)}_{\x,t} - \bm{\gamma}^{(l)}_{\z,0}}_2\norm{{\bm{\alpha}}^{(l)}_{\x,t}}_2
    + \norm{\bm{\gamma}^{(l)}_{\z,0}}_2\norm{{\bm{\alpha}}^{(l)}_{\x,t} - {\bm{\alpha}}^{(l)}_{\z,0}}_2 \leq O\xkm{R m^{-1/6} \tau^{1/3}\sqrt{\ln m}}
  \end{align*}
  by \cref{lem:backward perturbation}, \cref{lem:NN_forward_perturbation} and \cref{lem:Init_WeightsProductBound}.

\end{proof}

After preparing these tools, now we are ready to give the proof of \cref{prop:Training_KernelBound}.

\paragraph{Proof of \cref{prop:Training_KernelBound}}
By applying \cref{lem:Training_GradientBound} with $\Delta=O(1/m)\leq O(1/\sqrt{m})$ and $\tau \asymp m^{1/4}\geq \Delta\sqrt{m}$, with probability at least $1 - \exp(-\Omega(m^{5/6}))$, for all $\x\in\tilde{B}_R$ such that $\norm{\x-\z}_2\leq O(1/m)$, we can obtain the following result
\begin{align*}
  \norm{ \nabla_{\bm{W}^{(l)}} g_t(\x)- \nabla_{\bm{W}^{(l)}} g_0(\z) }_{\mr{F}}
  = O\xkm{ R m^{-1/12}\sqrt{ \ln m}}~\text{and}~\norm{ \nabla_{\bm{W}^{(l)}} g_t(\x)}_{\mr{F}} = O(R).
\end{align*}
The same conclusion holds if we replace $\x$ and $\z$ with $\x'$ and $\z'$.
Thus, we have
\begin{align*}
  &\abs{K^{(p)}_t(\x,\x') - K^{(p)}_0(\z,\z')}\leq \sum_{l=0}^L\abs{\ang{\nabla_{\bm{W}^{(l)}} g_t(\x), \nabla_{\bm{W}^{(l)}} g_t(\x') } - \ang{\nabla_{\bm{W}^{(l)}} g_0(\z), \nabla_{\bm{W}^{(l)}} g_0(\z') }} \\
  \begin{split}
    &\qquad\qquad\qquad \leq \sum_{l=0}^L \Big[\norm{ \nabla_{\bm{W}^{(l)}} g_t(\x)}_{\mr{F}} \norm{ \nabla_{\bm{W}^{(l)}} g_t(\x')- \nabla_{\bm{W}^{(l)}} g_0(\z') }_{\mr{F}}\\
    &\qquad\qquad\qquad\qquad\,\, + \norm{ \nabla_{\bm{W}^{(l)}} g_t(\x)- \nabla_{\bm{W}^{(l)}} g_0(\z) }_{\mr{F}}\norm{ \nabla_{\bm{W}^{(l)}} g_0(\z')}_{\mr{F}}\Big]
  \end{split}
  = O\xkm{R m^{-1/12}\sqrt{ \ln m}}
\end{align*}
holds for all $t\in T$ with probability at least $1 - \exp\xkm{-\Omega(m^{5/6})}$ when $m$ is large enough.

\subsection{Lazy regime}\label{subsec:Lazy}

In this subsection, we will prove that during the process of gradient descent training, the parameters do not change much.
Therefore, the conditions for the lemmas stated in the previous subsection are satisfied.
Since training relies on the structure of the neural network, in this subsection, we will no longer omit the superscript $p$ (although the corresponding conclusions also hold for non-mirror neural networks).

Let $\lambda_0 = \lambda_{\min} (\NTK(\X,\X))$ and $\bm{u}(t)=f_t(\X)-\bm{y}$.
Denote $\tilde{M}_{\bm{X}} = \sum_{i=1}^n\norm{\tilde{\x_i}}_2 $.
Since we will show the positive definiteness of the NTK in \cref{prop:NTK_PD}, we can assume that $\lambda_0 > 0$ hereafter.
Similar to Lemma F.8 and Lemma F.7 in \citet{arora2019_ExactComputation}, we have the following lemmas:

\begin{lemma}
  \label{lem:GradientFlow_ExpDecay}
  Let $\delta\in(0,1)$ and $t\in[0,\infty)$.
  Suppose that $\norm{\bm{W}^{(l,p)}_s - \bm{W}^{(l,p)}_0}_{\mr{F}} = O(m^{1/4})$ holds for all $s \in [0,t] $, $l\in \dk{0,1,\cdots,L}$ and $p\in\dk{1,2}$.
  Then there exists a polynomial $\poly(\cdot)$ such that when $m\geq\poly\xkm{n,\lambda_0^{-1},\ln(1/\delta)}$,
with probability at least $1-\delta$, we have
  \begin{align}
    \label{eq:GradientFlow_ExpDecay}\norm{\bm{u}(s)}^{2}\leq \exp\xkm{- \frac{\lambda_{0}}{n}s}\norm{\bm{u}(0)}^{2} = \exp\xkm{-\frac{\lambda_{0}}{n}s} \|\bm{y}\|^{2},\quad~\text{for all}~s\in [0,t].
  \end{align}
\end{lemma}

\begin{proof}
  Denote $\tilde{\lambda}_0(s)=\lambda_{\min}\big(K_s(\X,\X)\big)$.
  By Weyl's inequality, we can get
  \begin{align*}
    \abs{\tilde{\lambda}_0(s)-\lambda_0} \leq \norm{ K_s(\X,\X) - \NTK(\X,\X) }_2\leq \norm{ K_s(\X,\X) - \NTK(\X,\X) }_{\mr{F}}&\\
    \leq\frac12\sum_{p=1}^2\sum^n_{i,j=1} \abs{K^{(p)}_s(\x_i,\x_j) - \NTK(\x_i,\x_j)}&.
  \end{align*}
  Applying \cref{cor:Training_Kernel_uniform} with $\delta'= {\delta}/{(2 n^2)}$ to each difference,
  with probability at least $1-2n^2\delta'=1-\delta$, we can obtain the following bound for all $s\in [0,t]$:
  \begin{align*}
    \abs{\tilde{\lambda}_0(s)-\lambda_0}\leq n^2 O\xkm{m^{-1/12} \sqrt{\ln m}}\leq n^2 O\xkm{m^{-1/15}}\leq\frac{\lambda_0}2,
  \end{align*}
  when $m\geq C_1\zk{\xk{n^{2}\lambda_0^{-1}}^{15}+\xk{\ln\xkm{C_2n^2/\delta}}^5}$ for some positive absolute constants $C_1>0$ and $C_2\geq 1$.
  This implies that $\tilde{\lambda}_0(s)\geq\lambda_0/2$ holds for all $s\in [0,t]$.
  Therefore, we have
  \begin{equation*}
    \frac{\dd}{\dd s}\|\bm{u}(s)\|^{2}=- \frac{2}{n}\bm{u}(s)^T K_{s}(\X,\X) \bm{u}(s)\le -\frac{\lambda_{0}}{n} \lVert \bm{u}(s)\rVert^{2},
  \end{equation*}
  which implies \cref{eq:GradientFlow_ExpDecay} by standard ODE theory.
  Finally, by choosing
  \begin{align*}
    \poly\xkm{n,\lambda_0^{-1},\ln(1/\delta)}=C_1\zkm{\xk{n^{2}\lambda_0^{-1}}^{15}+\xk{2n+\ln(1/\delta)+\ln C_2}^5}
  \end{align*}
  we can complete the proof of this lemma.
\end{proof}

\begin{lemma}
  \label{lem:A_lazy_W}
  Fix $l \in \{0,1,\cdots,L\}$, $p \in \{1,2\}$ and let $\delta\in(0,1)$, $t\in[0,\infty)$.
  Suppose that for $s \in [0,t]$, we have
  \begin{align*}
    \begin{gathered}
      \norm{f_s(\X)-\bm{y}}_2 \leq \exp(\frac{-\lambda_0 }{4n}s) \norm{\bm{y}}_2\quad~\text{and}\\
      \norm{ \bm{W}^{(l',p')}_s - \bm{W}^{(l',p')}_0 }_{\mr{F}} \leq \frac{\sqrt{m}}{(\ln m)^3},\qquad~\text{for all}~(l',p') \neq (l,p).
    \end{gathered}
  \end{align*}
  Then there exists a polynomial $\poly(\cdot)$ such that when $m\geq\poly\xkm{n, \tilde{M}_{\bm{X}}, \norm{\bm{y}}_2,\lambda_0^{-1},\ln(1/\delta)}$,
with probability at least $1-\delta$, we have
  $\sup_{s\in[0,t]}\norm{\bm{W}^{(l,p)}_s-\bm{W}^{(l,p)}_0 }_{\mr{F}} = O\xkm{{n \tilde{M}_{\bm{X}}\norm{\bm{y}}_2}/{\lambda_0}}$.
\end{lemma}

\begin{proof}
  First of all, recalling \cref{eq:2_GD} we have
  \begin{align}
    \notag
    \norm{\bm{W}^{(l,p)}_{t_0} - \bm{W}^{(l,p)}_0}_{\mr{F}}&= \norm{\int^{t_0}_0 {\dd \bm{W}^{(l,p)}_s}}_{\mr{F}} \leq  \int^{t_0}_0 \norm{\frac{1}{n\sqrt{2}}\sum^n_{i=1} ( f_s(\x_i) - y_i) \nabla_{\bm{W}^{(l,p)}} g^{(p)}_s(\x_i)}_{\mr{F}}\dd s \\
    \notag
    &\leq \frac{1}{n\sqrt{2}} \sum^n_{i=1} \sup_{0\leq s \leq t_0}\norm{\nabla_{\bm{W}^{(l,p)}} g_s^{(p)}(\x_i)}_{\mr{F}} \int^{t_0}_0 \norm{f_s(\X)-\bm{y}}_2~\dd s \\
    \label{eq:Proof_Lazy_GradientControl0}
    &\leq  O\xkm{\frac{\norm{\bm{y}}}{\lambda_0}} \sum^n_{i=1}\sup_{0\leq s \leq t_0} \norm{\nabla_{\bm{W}^{(l,p)}} g_s^{(p)}(\x_i)}_{\mr{F}} .
  \end{align}
  for all $t_0\in[0,t]$.
Suppose that \[s_0 = \min\dkm{ s \in [0,t]:\norm{\bm{W}^{(l,p)}_s- \bm{W}^{(l,p)}_0}_{\mr{F}} \geq  {\sqrt{m}}/{(\ln m)^3}}\] exists,
then $\norm{\bm{W}^{(l',p')}_s - \bm{W}^{(l',p')}_0}_{\mr{F}} \leq {\sqrt{m}}/{(\ln m)^3}$ holds for all $s \in [0,s_0] $, $l'\in \dk{0,1,\cdots,L}$ and $p'\in\dk{1,2}$.
  Applying \cref{lem:Training_GradientBound} with $\Delta=0$ and $\tau={\sqrt{m}}/{(\ln m)^3}$, we know that for any $i\in[n]$, with probability at least $1-\exp\xkm{-\Omega\xkm{m/(\ln m)^{2}}}\geq 1-\exp\xkm{-\Omega(m^{5/6})}$, we have
  \begin{equation*}
    \sum_{i=1}^n\sup_{s \in [0,s_0]} \norm{\nabla_{ \bm{W}^{(l,p)}} g_s^{(p)}(\x_i)}_{\mr{F}} = O(\tilde{M}_{\bm{X}}).
  \end{equation*}
  Plugging it back to \cref{eq:Proof_Lazy_GradientControl0}, with probability at least $1-n\exp\xkm{-\Omega(m^{5/6})}$, we obtain
  \[\norm{\bm{W}^{(l,p)}_{s_0} - \bm{W}^{(l,p)}_0}_{\mr{F}} =  O({n \tilde{M}_{\bm{X}} \norm{\bm{y}}}/{\lambda_0}),\]
  which contradicts to $\norm{\bm{W}^{(l,p)}_{s_0}- \bm{W}^{(l,p)}_0}_{\mr{F}}  \geq {\sqrt{m}}/{(\ln m)^3}$ when $m\geq C_1\xk{n\tilde{M}_{\bm{X}}\norm{\bm{y}}_2\lambda_0^{-1}}^5$ for some positive constant $C_1$.

  Now we show that $\norm{\bm{W}^{(l',p')}_s- \bm{W}^{(l',p')}_0}_{\mr{F}} \leq {\sqrt{m}}/{(\ln m)^3}$ holds for all $s \in [0,t]$ and any $(l',p')$.
  The desired bound then follows from applying \cref{lem:Training_GradientBound} again for the interval $[0,t]$.

  Also, it is easy to check that there exists a positive absolute constant $C$ such that when $m
  \geq C_2\ln(n/\delta)^{6/5}$, we have $1-n\exp\xkm{-\Omega(m^{5/6})}\geq 1-\delta$.
  Finally, by choosing
  \[\poly\xkm{n,\lambda_0^{-1},\ln(1/\delta)}=C\zk{\xk{n\tilde{M}_{\bm{X}}\norm{\bm{y}}_2\lambda_0^{-1}}^{5}+\xk{n+\ln(1/\delta)}^2+1}\]
  for some positive absolute constant $C>0$, we can complete the proof.
\end{proof}

The following lemma is the key lemma that we aim to prove in this subsection.
It serves as the prerequisite for establishing the conclusions of the lemmas in the preceding and subsequent subsections.

\begin{lemma}
[Lazy regime]
  \label{lem:A_lazy_regime}
  There exists a polynomial $\poly(\cdot)$ such that for any $\delta\in(0,1)$, with probability at least $1-\delta$, for all $p\in\dk{1,2}$ and $l\in\dk{0,1,\cdots,L}$, we have
  \begin{align*}
    \sup_{t \geq 0}\norm{\bm{W}^{(l,p)}_t - \bm{W}^{(l,p)}_0 }_{\mr{F}} =O(m^{1/4}).
  \end{align*}
  when $m\geq\poly\xkm{n,\tilde{M}_{\bm{X}},\norm{\bm{y}}_2,\lambda_0^{-1},\ln(1/\delta)}$.
\end{lemma}

\begin{proof}
  Let us assume that
  \begin{align*}
    t_0 = \min\dkm{t\geq 0:\exists l,p~\text{such that}~
    \norm{\bm{W}^{(l,p)}_t - \bm{W}^{(l,p)}_0}_{\mr{F}} \geq m^{1/4}~\text{or}~\norm{\bm{u}(t)} \geq \exp(\tfrac{-\lambda_0}{4n}t) \norm{\bm{y}}}
  \end{align*}
  exists.
  Then, for all $t \in  [0,t_0]$, we have
  \begin{align*}
    \begin{gathered}
      \norm{\bm{u}(t)} \leq \exp(\frac{-\lambda_0}{4n}t) \norm{\bm{y}}\qquad~\text{and}~\qquad
      \norm{\bm{W}^{(l,p)}_t - \bm{W}^{(l,p)}_0}_{\mr{F}} \leq m^{1/4}~\text{for all}~l, p.
    \end{gathered}
  \end{align*}
  According to \cref{lem:A_lazy_W} and \cref{lem:GradientFlow_ExpDecay}, we know that there exists a polynomial $\poly(\cdot)$ such that with probability at least $1-\delta$, we have
  \begin{align*}
    \begin{gathered}
      \norm{\bm{u}(t_0)} \leq \exp(\frac{-\lambda_{0}}{2n}t_0)\norm{\bm{y}}\qquad~\text{and}~\qquad
      \norm{\bm{W}^{(l,p)}_{t_0}-\bm{W}^{(l,p)}_0 }_{\mr{F}} = O\xkm{\frac{n\tilde{M}_{\bm{X}}\norm{\bm{y}}}{\lambda_0}}~\text{for all}~l, p
    \end{gathered}
  \end{align*}
  when $m\geq \poly\xkm{n,\norm{\bm{y}}_2,\lambda_0^{-1},\ln(1/\delta)}$,
  which contradicts to the definition of $t_0$ when $m\geq C\xkm{n\tilde{M}_{\bm{X}}\norm{\bm{y}}_2\lambda_0^{-1}}^5$ for some positive absolute constant $C>0$.
\end{proof}
We also have a simple corollary:
\begin{corollary}
  \label{cor:UpperboundNN}
  There exists a positive absolute constant $C$ such that, under the same conditions as in \cref{lem:A_lazy_regime}, possibly after enlarging the polynomial there, with probability at least $1-\delta$,
  \begin{align*}
    \abs{f_t^m(x)} \leq  C \sqrt{m}\norm{\tilde{\x}}, \quad \forall x \in \R^d,~ \forall t \geq 0.
  \end{align*}
\end{corollary}
\begin{proof}
  Recall that
  \begin{align*}
    f_t^m(x) = \frac{\sqrt{2}}{2}[\bm{W}^{(L,1)}_t\bm{\alpha}_t^{(L,1)}(x) - \bm{W}^{(L,2)}_t\bm{\alpha}_t^{(L,2)}(x) ].
  \end{align*}
  By applying \cref{lem:A_lazy_regime} with $\delta/2$ and using \cref{lem:Random matrix}, after enlarging the polynomial in \cref{lem:A_lazy_regime} if necessary, with probability at least $1-\delta$ we have
  \begin{align*}
    \sup_{t\geq 0}\norm{\bm{W}_t^{(l,p)}}_2
    \leq \norm{\bm{W}_0^{(l,p)}}_2 + \sup_{t\geq 0}\norm{\bm{W}_t^{(l,p)} - \bm{W}_0^{(l,p)}}_{\mr{F}}
    \leq O(\sqrt{m}),
  \end{align*}
  for all $l\in\dk{0,1,\cdots,L}$ and $p\in\dk{1,2}$, where we used $m_l\in[m,C_{\mathrm{m}}m]$ for hidden layers.
  Since $\norm{\bm{D}_{x,t}^{(l,p)}}_2 \leq 1$, \cref{eq:NN_Layer_MatrixProductForm} implies
  \begin{align*}
    \norm{\bm{\alpha}_t^{(L,p)}(x)}_2
    &\leq \xk{\prod_{r=1}^{L}\sqrt{\frac{2}{m_r}}\norm{\bm{D}_{x,t}^{(r,p)}}_2\norm{\bm{W}_t^{(r-1,p)}}_2}\norm{\tilde{\x}}
    \leq O(\norm{\tilde{\x}}).
  \end{align*}
  Therefore,
  \begin{align*}
    \abs{\bm{W}^{(L,p)}_t\bm{\alpha}_t^{(L,p)}(x)}
    \leq \norm{\bm{W}^{(L,p)}_t}_2 \norm{\bm{\alpha}_t^{(L,p)}(x)}_2
    \leq O(\sqrt{m}\norm{\tilde{\x}}),
  \end{align*}
  and the desired bound follows from the triangle inequality.
\end{proof}

\subsection{Hölder continuity of \texorpdfstring{$\NTK$}{NTK}}\label{subsec:Holder_NTK}
For convenience, let us first introduce the following definition of Hölder spaces~\citep{adams2003_SobolevSpaces}.
For an open set $\Omega \subset \R^p$ and a real number $\alpha \in [0,1]$, let us define a semi-norm for $f : \Omega \to \R$ by
\begin{align*}
  \abs{f}_{0,\alpha} = \sup_{x,y \in \Omega,~x\neq y}{\abs{f(x) - f(y)}}/{\norm{x-y}^\alpha}
\end{align*}
and define the Hölder space by $C^{0,\alpha}(\Omega) = \left\{ f \in C(\Omega) : \abs{f}_{0,\alpha} < \infty \right\},$
which is equipped with norm $\norm{f}_{C^{0,\alpha}(\Omega)} = \sup_{x \in \Omega} \abs{f(x)} + \abs{f}_{\alpha}$.
Then it is easy to show that
\begin{enumerate}[$(a)$~]
  \item $C^{0,\alpha}(\Omega) \subseteq C^{0,\beta}(\Omega)$ if $\beta \leq \alpha$;
  \item if $f,g \in C^{0,\alpha}(\Omega)$, then $f + g,~ fg \in C^{\alpha}(\Omega)$;
  \item if $f \in C^{0,\alpha}(\Omega_1)$ and $g \in C^{0,\beta}(\Omega_2)$ with $\ran g \subseteq \Omega_1$, then $f\circ g \in C^{0,\alpha\beta}(\Omega_2)$.
\end{enumerate}
Consequently, using the formula \cref{eq:NTK_Formula}, we can show the following proposition:
\begin{proposition}
  \label{prop:NTK_Continuity}
  We have $\NTK \in C^{0,s}( \tilde{B}_R \times  \tilde{B}_R)$ with $s = 2^{-L}$.
  Particularly, there is some absolute constant $C>0$ such that
  for any $x,x',z,z' \in  \tilde{B}_R$,
  \begin{align}
    \abs{\NTK(\x,\x') - \NTK(\z,\z')} \leq C R^2 \norm{(\x,\x') - (\z,\z')}^s.
  \end{align}
\end{proposition}
\begin{proof}
  Let us recall \cref{eq:NTK_Formula}.
Since $\NTK$ is symmetric, by triangle inequality it suffices to prove that $\NTK(\x_0,\cdot) \in C^{0,s}(\tilde{B}_R)$
  with $\abs{\NTK(\x_0,\cdot)}_{0,s}$ bounded by a constant independent of $\x_0$.
  Now, the latter is proven by
  \begin{enumerate}[$(a)$~]
    \item $\x \mapsto \bar{u} = \ang{{\tilde{\x}}/{\norm{\tilde{\x}}}, {\tilde{\x}_0}/{\norm{\tilde{\x}_0}}} \in C^{0,1}(\tilde{B}_R)$, where the bound of the Hölder
    norm is independent of $\x_0$;
    \item as functions of $\bar{u}$, both $\sqrt{1 - \bar{u}^2}$ and $\arccos \bar{u}$ belong to $C^{0,{1}/{2}}([-1,1])$, and thus
    $\kappa_0,\kappa_1 \in C^{0,{1}/{2}}([-1,1])$;
    \item the expression of NTK together with the properties of Hölder functions.
  \end{enumerate}
\end{proof}

\subsection{The kernel uniform convergence} \label{subsec:KernelUniformConvergence}

\begin{proposition}[Kernel uniform convergence]
  \label{prop:2_KernelUniform}
  Denote $B_r = \{x\in \mb{R}^{d}: 1 \leq \norm{\tilde{\x}} \leq r\}$.
  There exists a polynomial $\operatorname{poly}(\cdot)$ such that
  for any $\delta \in (0,1)$,
  as long as $m \geq \mathrm{poly}\left(n, \tilde{M}_{\bm{X}}, \lambda_0^{-1},\norm{\bm{y}},\ln(1/\delta),k\right)$ and $m \geq r^k$,
  with probability at least $1-\delta$ we have
  \begin{equation*}
    \sup_{t \geq 0} \sup_{\x,\x'\in B_r}\abs{K_t(\x,\x') - \NTK(\x,\x')} \leq O\xkm{ r^2 m^{-\frac{1}{12}}\sqrt{\ln m}}.
  \end{equation*}
\end{proposition}

\begin{proof}
  First, by \cref{lem:A_lazy_regime}, we know that there exists a polynomial $\poly_1(\cdot)$ such that for any $\delta\in(0,1)$, when $m\geq\poly_1\xkm{n, \tilde{M}_{\bm{X}},\norm{\bm{y}},\lambda_0^{-1},\ln(1/\delta)}$, then with probability at least $1-\delta/2$, for all $p\in\dk{1,2}$ and $l\in\dk{0,1,\cdots,L}$, we have
  \begin{align*}
    \sup_{t \geq 0} \norm{\bm{W}^{(l,p)}_t - \bm{W}^{(l,p)}_0}_{\mr{F}} = O(m^{1/4}).
  \end{align*}
  Next, we condition on this event happens.

  Since $B_r \subset \R^d$ is bounded, for any $\ep > 0$ we have an $\ep$-net $\caN_{\ep}$ (with respect to $\norm{\cdot}_2$)
  of $\tilde{B}_R$ such that the cardinality $\abs{\caN_{\ep}} = O(r^d\ep^{-d})$
  ~\citep[Section 4.2]{vershynin2018_HighdimensionalProbability}.
  Specifically, we choose $\ep = m^{-2^L}$ and thus $\ln \abs{\caN_{\ep}} = O(\ln m)$ if $m \geq r^k$ and $ m \geq \poly_3(k)$.
  Denote by
  \[B_{\z,\z'}(\varepsilon)=\dk{(\x,\x'):\norm{\x-\z}\leq\varepsilon,~\norm{\x'-\z'} \leq \ep,~\x,\x' \in\tilde{B}_R}.\]
  Then, fixing $\z,\z' \in \mc{N}_\varepsilon$, for any $(\x,\x') \in B_{\z,\z'}(\varepsilon)$, we have
  \begin{align*}
    \abs{K_t(\x,\x') - \NTK(\x,\x')}&\leq \abs{ K_t(\x,\x')  - \NTK(\z,\z') } + \abs{\NTK(\z,\z') - \NTK(\x,\x')}.
  \end{align*}
  Then, noticing that $K_t = (K_t^{(1)} + K_t^{(2)})/2$, we control the two terms on the right hand side by
  \cref{cor:Training_Kernel_uniform} and \cref{prop:NTK_Continuity} respectively,
  deriving that with probability at least $1-\delta/\xk{2|\mathcal{N}_{\varepsilon}|^2}$, for all $t\geq 0$, we have
  \begin{align*}
\begin{gathered}
      \sup_{(\x,\x')\in B_{\z,\z'}(\varepsilon)} |K_t(\x,\x') - \NTK(\z,\z')| = O\xkm{r^2 m^{-1/12}\sqrt{\ln m}}, \\
      \abs{\NTK(\z,\z') - \NTK(\x,\x')} = O(r^2\ep^{2^{-L}}) = O(r^2m^{-1}),
    \end{gathered}
  \end{align*}
  if $m\geq C_1\ln\xkm{C_2|\mathcal{N}_{\varepsilon}|^2/\delta}^5$ for some positive absolute constants $C_1>0$ and $C_2\geq 1$.
  
  And there also exists a polynomial $\poly_2(\cdot)$ such that when $m\geq\poly_2\xkm{\ln(1/\delta)}$, we have $m\geq C_1\ln\xkm{C_2|\mathcal{N}_{\varepsilon}|^2/\delta}^5$, since $\ln |\mathcal{N}_{\varepsilon}| = O(\ln m)$.
  Combining these two terms, we have
  \begin{align*}
\sup_{t \geq 0} \sup_{(\x,\x')\in B_{\z,\z'}(\ep)} \abs{K_t(\x,\x') - \NTK(\x,\x')} = O\xkm{r^2m^{-{1/12}}\sqrt{\ln m}}
  \end{align*}
  if $m\geq\poly_2(\ln(1/\delta))$.

  Combining all of the above results and applying the union bound for all pair $\bm{z},\bm{z}'\in\mathcal{N}_\varepsilon$, with probability at least $1-\delta$, we have
  \begin{align*}
    \sup_{t\geq 0}\sup_{\x,\x'\in B_r}\abs{ K_{t}(\x,\x') - \NTK(\x,\x') } = O\xkm{r^2 m^{-{1}/{12}}\sqrt{\ln m}}
  \end{align*}
  if $m\geq\poly_1\xkm{n,\tilde{M}_{\bm{X}},\norm{\bm{y}}_2,\lambda_0^{-1},\ln(1/\delta)}  + \poly_2(\ln(1/\delta)) + \poly_3(k)$.
\end{proof}

  \section{Auxiliary Results on the NTK}

\subsection{Positive definiteness}
\label{subsec:D_PDNTK}

The following proposition is an elementary result on the power series expansion of the arc-cosine kernels.
\begin{proposition}
  \label{prop:D_ACK_PS}
  We have the following power series expansion for $\kappa_0$ and $\kappa_1$ in \cref{eq:Arccos_Formula}:
  \begin{align}
    \kappa_0(u) &= \frac{1}{2} + \frac{1}{\pi} \sum_{r=0}^\infty \frac{\left(\frac{1}{2}\right)_r}{(2r+1) r!} u^{2r+1},
    \quad
    \kappa_1(u) = \frac{1}{\pi} + \frac{1}{2}u + \frac{1}{\pi} \sum_{r = 1}^{\infty} \frac{\left(\frac{1}{2}\right)_{r-1}}{2(2r-1) r!} u^{2r},
  \end{align}
  where $(a)_p \coloneqq a(a+1)\dots(a+p-1)$ represents the rising factorial
  and both series converge absolutely for $u \in [-1,1]$.
\end{proposition}

\begin{lemma}
[Lemma B.1 in \citet{lai2023_GeneralizationAbility}]
  \label{lem:B_PDSphere}
  Let $f: [-1,1] \to \R$ be a continuous function with the expansion
  $f(u) = \sum^{\infty}_{n=0}a_n u^n, \quad u \in [-1,1]$
  and $k(x,y) = f(\ang{x,y})$ be the associated inner-product kernel on $\bbS^{d}$.
  If $a_{n}\geq 0$ for all $n \geq 0$ and there are infinitely many $a_n > 0$, then $k$ is
  positive definite on $\mathbb{S}^{d}_{+}$.
\end{lemma}

\begin{proof}
[of \cref{prop:NTK_PD}]
  Following the proof of \cref{thm:NTK_EDR}, we introduce the transformation $\Phi$ and the homogeneous NTK $\NTK_0$.
  Plugging \cref{prop:D_ACK_PS} into \cref{eq:NTK0_Def}, by \cref{lem:B_PDSphere} we can show that $\NTK_0$ is strictly positive definite on
  $\bbS_{+}^{d} = \Phi(\R^d)$.
  Consequently, the positive definiteness of $\NTK$ follows from the fact that $\Phi$ is bijective and $\norm{\tilde{x}} \geq 1$.
\end{proof}

\subsection{Numerical experiments}

\begin{figure*}[t]
  \centering
  \includegraphics[width=0.49\textwidth]{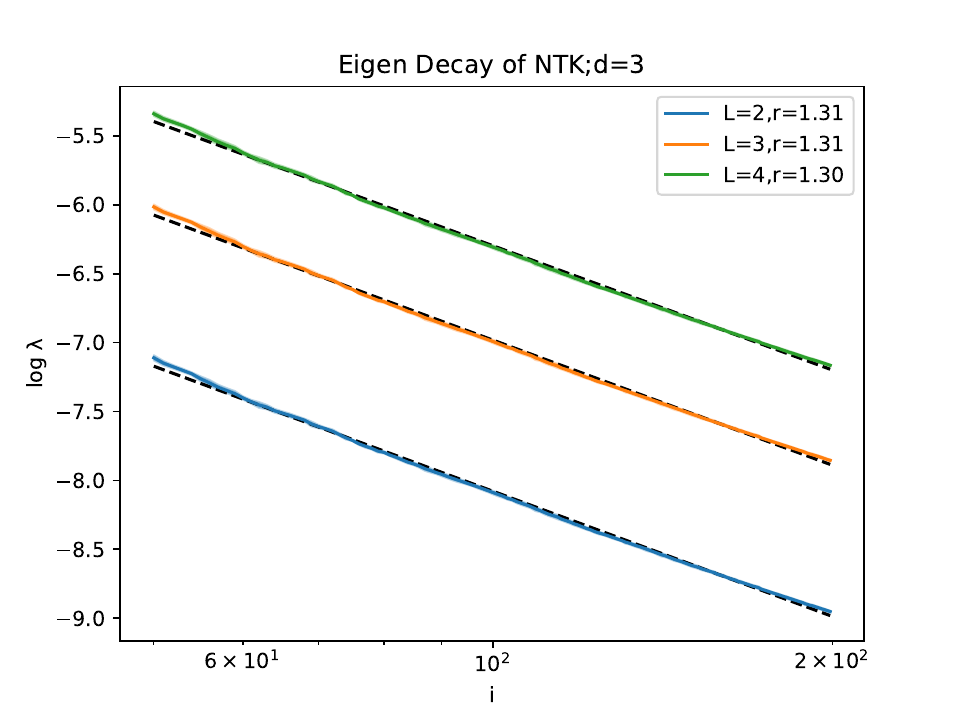}
  \includegraphics[width=0.49\textwidth]{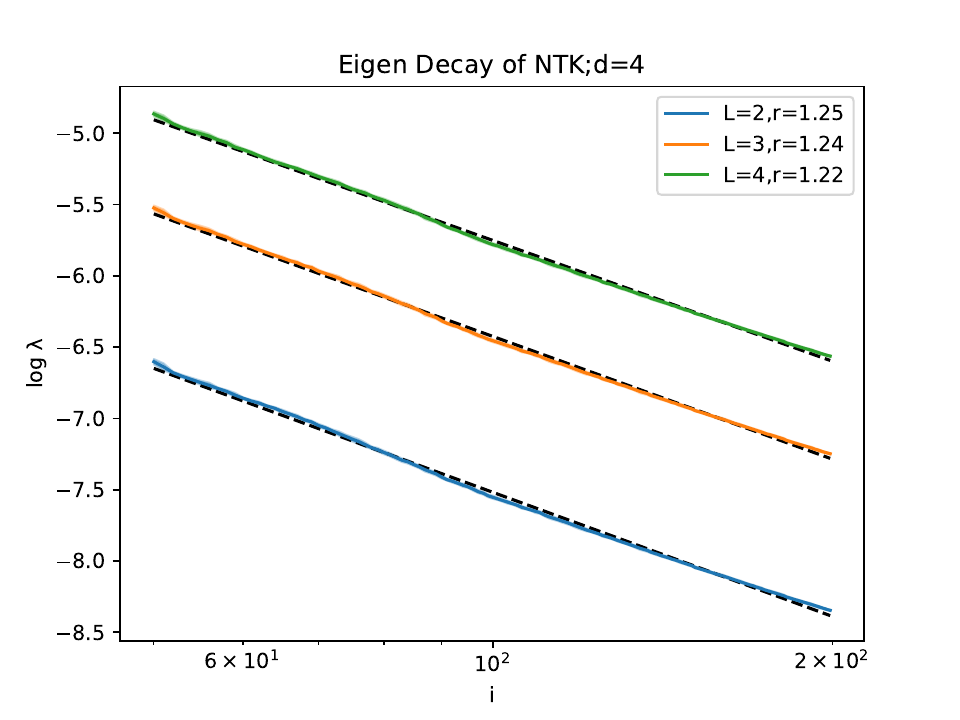}
  \includegraphics[width=0.49\textwidth]{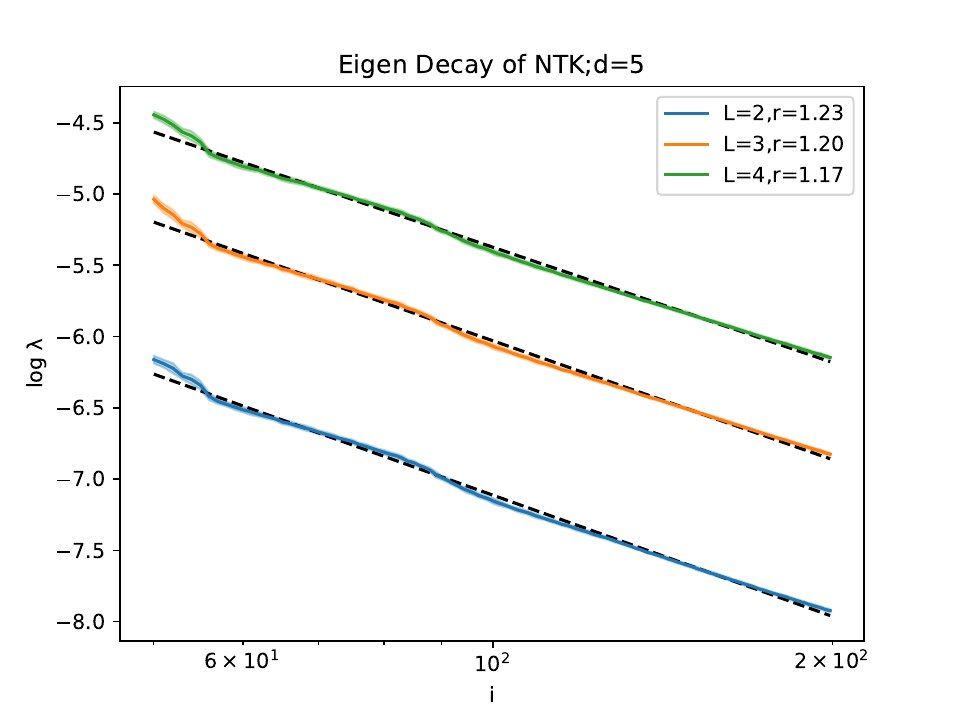}
  \includegraphics[width=0.49\textwidth]{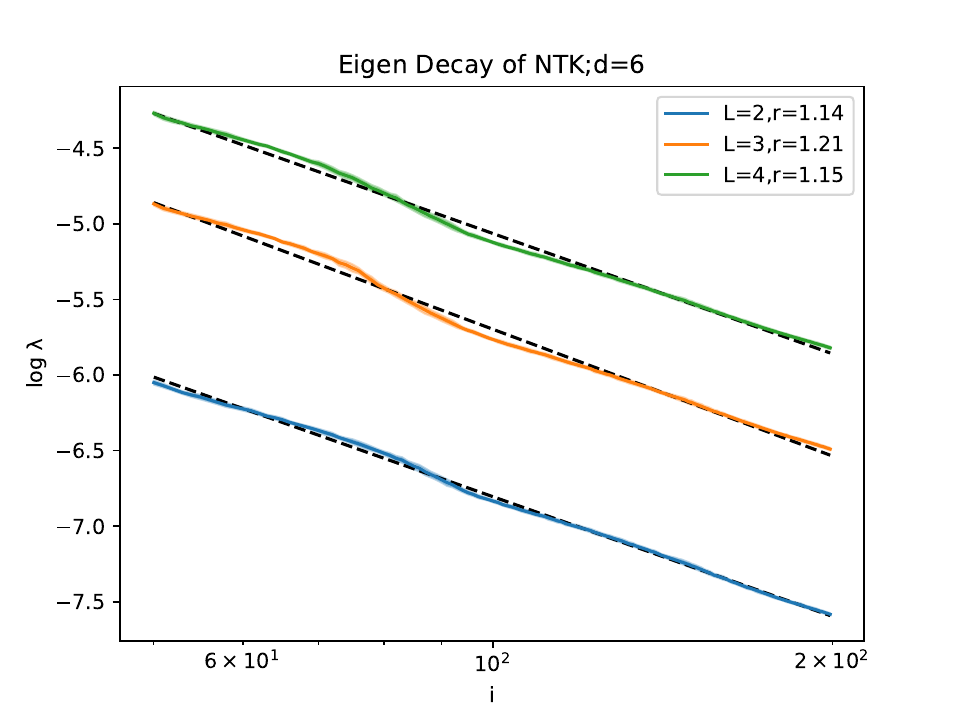}
  \caption{Eigenvalue decay of NTK under uniform distribution on $[-1,1]^d$,
    where $i$ is selected in $[50,200]$ and $n = 1000$.
    The dashed black line represents the log least-square fit and the decay rates $r$ are reported.}
  \label{fig:EDR1}
\end{figure*}
\begin{table*}[t]
  \centering
  \begin{tabular}{c|ccc|ccc|ccc}
    & \multicolumn{3}{c|}{$d=3$} & \multicolumn{3}{c|}{$d=4$} & \multicolumn{3}{c}{$d=5$} \\
    \hline
    Distribution   & $L=2$ & $L=3$ & $L=4$ & $L=2$ & $L=3$ & $L=4$ & $L=2$ & $L=3$ & $L=4$ \\
    \hline
    $U(-1,1)$      & 1.31  & 1.31  & 1.30  & 1.25  & 1.24  & 1.22  & 1.23  & 1.20  & 1.17  \\
    $U(0,1)$       & 1.33  & 1.33  & 1.32  & 1.26  & 1.26  & 1.25  & 1.14  & 1.13  & 1.12  \\
    Triangular     & 1.34  & 1.33  & 1.32  & 1.21  & 1.23  & 1.22  & 1.22  & 1.16  & 1.13  \\
    Clipped normal & 1.28  & 1.30  & 1.28  & 1.26  & 1.24  & 1.21  & 1.11  & 1.09  & 1.06
  \end{tabular}
  \caption{Eigenvalue decay rates of NTK, where each entry of $\x$ is drawn independently from multiple distributions.
  Triangular: $p(\x) = 1+x$, $x \in [-1,0]$, $p(\x) = 1-x$, $x \in [0,1]$;
  Clipped normal: standard normal clipped into $(-10,10)$.}
  \label{tab:1}
\end{table*}

We provide some numerical experiments on the eigenvalue decay of the neural tangent kernel.
We approximate the eigenvalue $\lambda_i$ by the eigenvalue $\lambda_i(K)$ of the regularized sample kernel matrix for $n$ much larger than $i$.
Then, we estimate eigenvalue decay by fitting a log least-square
$\ln \lambda_i = r \ln i + b$.
We skip the first several eigenvalues since they do not reflect the asymptotic decay.
We report the results in \cref{fig:EDR1} and \cref{tab:1}.
The results match our theoretical prediction and justify our theory.

  \section{Omitted proofs}\label{sec:omitted-proofs}
  \paragraph{Proof of \cref{prop:NN_Gen}}
\cref{thm:NTK_EDR} shows the eigenvalue decay rate of $\NTK$ is $(d+1)/d$.
Therefore, the results in \citet{lin2018_OptimalRates} implies the lower rate and that the gradient flow of NTK satisfies
\begin{align}
  \label{eq:E_Proof1}
  \norm{\hat{f}^{\mathrm{NTK}}_{t_{\mathrm{op}}} - f^*}_{L^2} \leq C \left( \ln\frac{6}{\delta} \right) n^{-\frac{1}{2}\frac{s\beta}{s\beta + 1}}
\end{align}
with probability at least $1-\delta$, where $\beta = (d+1)/d$.

On the other hand, since $\mu$ is sub-Gaussian, $\sum_{i=1}^n \norm{x_i}_2 \leq C n^2$
for probability at least $1-\delta$ if $n \geq \mathrm{poly}(\ln(1/\delta))$.
From $y_i = f^*(\x_i) + \ep_i$, $f^* \in L^\infty$ and $\ep_i$ is sub-Gaussian,
we have $\norm{\bm{y}} \leq 2C n$ for probability at least $1-\delta$ as long as $n \geq \mathrm{poly}(\ln(1/\delta))$.
Then, taking $k=1/48$ and $r = m^{k}$ in \cref{lem:UnifConverge},
when $m \geq \mathrm{poly}(n,\lambda_0^{-1},\ln(1/\delta))$,
with probability $1-3\delta$ we have
\begin{align*}
  \sup_{t\geq 0} \sup_{\x \in B_r} \abs{\fNTK(\x) - \fNN(\x)} \leq C m^{-\frac{1}{24}}\sqrt {\ln m} \leq C n^{-1}
\end{align*}
as long as we take a larger power of $n$ in the requirement of $m$.
Consequently,
\begin{align*}
  \norm{(\hat{f}^{\mathrm{NN}}_{t_{\mathrm{op}}} - f^*)\bm{1}_{B_r}}_{L^2}
  \leq \norm{(\hat{f}^{\mathrm{NN}}_{t_{\mathrm{op}}} - \hat{f}^{\mathrm{NTK}}_{t_{\mathrm{op}}})\bm{1}_{B_r}}_{L^2} + \norm{(\hat{f}^{\mathrm{NTK}}_{t_{\mathrm{op}}} - f^*)\bm{1}_{B_r}}_{L^2}
  \leq  \frac{1}{n} + C \left( \ln\frac{12}{\delta} \right) n^{-\frac{1}{2}\frac{s\beta}{s\beta + 1}}.
\end{align*}
Now,
\begin{align*}
  \norm{\hat{f}^{\mathrm{NN}}_{t_{\mathrm{op}}} - f^*}_{L^2}
&\leq \norm{(\hat{f}^{\mathrm{NN}}_{t_{\mathrm{op}}} - f^*)\bm{1}_{B_r}}_{L^2}
  + \norm{\hat{f}^{\mathrm{NN}}_{t_{\mathrm{op}}}\bm{1}_{B_r^\complement}}_{L^2}
  + \norm{f^* \bm{1}_{B_r^\complement}}_{L^2},
\end{align*}
where the first term is already bounded.
Noting that $\mu$ is sub-Gaussian and $r = m^{1/48}$, by \cref{cor:UpperboundNN} we bound the second term by
\begin{align*}
  \norm{\hat{f}^{\mathrm{NN}}_{t_{\mathrm{op}}}\bm{1}_{B_r^\complement}}_{L^2}
  \leq \norm{C \sqrt{m}\norm{\tilde{x}} \bm{1}_{B_r^\complement}}_{L^2}
  \leq C m^{-1} \leq Cn^{-1}
\end{align*}
and the third term by
\begin{align*}
  \norm{f^* \bm{1}_{B_r^\complement}}_{L^2} \leq \norm{f^*}_{L^\infty} \mu(B_r^\complement)^{1/2} \leq C n^{-1}.
\end{align*}
Plugging these bounds into the above inequality, we finish the proof.

\subsection{Choosing stopping time with cross-validation}
Before proving \cref{prop:NN_CV}, we introduce a modified version of \citet[Theorem 3]{caponnetto2010_CrossvalidationBased}.
\begin{proposition}
  \label{prop:E_CV}
  Let $\delta \in (0,1)$ and $\ep > 0$.
  Suppose $\hat{f}_t$ is a family of estimators indexed by $t \in T_n$ such that with probability at least $1-\delta$,
  it holds that $\norm{\hat{f}_{t_n} - f^*}_{L^2} \leq \ep$ for some $t_n \in T_n$.
  Then, if $\hat{t}_{\mathrm{cv}}$ is chosen by cross-validation according to
  \cref{eq:5_StoppingTimeCV}, with probability at least $1-2\delta$, it holds that
  \begin{align}
    \norm{\hat{f}_{\hat{t}_{\mathrm{cv}}} - f^*}_{L^2} \leq 2\ep + \left( \frac{160 M^2}{\tilde{n}} \ln \frac{2\abs{T_n}}{\delta} \right)^{1/2}.
  \end{align}
\end{proposition}

\paragraph{Proof of \cref{prop:NN_CV}}
The choice of $T_n$ guarantees that there is $t_n \in T_n$
such that $t_{\mathrm{op}} \leq t_n \leq Q t_{\mathrm{op}}$ for $t_{\mathrm{op}} = n^{(d+1)/ [s(d+1)+d]}$
and that $\abs{T_n} \leq \log_Q n + 1 \leq C\ln n$.
Then, by \cref{prop:NN_Gen} we know that
\begin{align*}
  \norm{\hat{f}_{t_n} - f^*}_{L^2} \leq C \left( \ln\frac{12}{\delta} \right) n^{-\frac{1}{2}\frac{s\beta}{s\beta + 1}}.
\end{align*}
Consequently, by \cref{prop:E_CV}, we conclude that
\begin{align*}
  \norm{\hat{f}_{\hat{t}_{\mathrm{cv}}} - f^*}_{L^2}
  \leq C \left( \ln\frac{12}{\delta} \right) n^{-\frac{1}{2}\frac{s\beta}{s\beta + 1}}
  + \left( \frac{160 M^2}{c_{\mathrm{v}} n} \ln \frac{C \ln n}{\delta} \right)^{1/2}
  \leq C \left( \ln\frac{12}{\delta} \right) n^{-\frac{1}{2}\frac{s\beta}{s\beta + 1}}
\end{align*}
as long as $n$ is sufficiently large.

  \section{Auxiliary Results}\label{sec:A_Aux}

\subsection{Self-adjoint compact operator}

For a self-adjoint compact positive operator $A$ on a Hilbert space,
we denote by $\lambda_n(A)$ the $n$-th largest eigenvalue of $A$.
The following minimax principle is a classic result in functional analysis.

\begin{lemma}[Minimax principle]
  Let $A$ be a self-adjoint compact positive operator.
  Then
  \begin{align*}
    \lambda_n(A) &= \sup_{\substack{V \subseteq H \\ \dim V = n}} \inf_{\substack{x \in V \\ \norm{x} = 1}} \ang{Ax,x}.
  \end{align*}
\end{lemma}

\begin{lemma}[Weyl's inequality for operators]
  \label{lem:A_WeylOp}
  Let $A,B$ be self-adjoint compact positive operators.
  Then
  \begin{align}
    \lambda_{i+j-1}(A+B) & \leq \lambda_{i}(A) +  \lambda_{j}(B),\quad i,j \geq 1.
\end{align}
\end{lemma}

\begin{lemma}
  \label{lem:seca_SumEigenCount}
  Let $A_1,\dots,A_k$ be self-adjoint and compact.
  Suppose $\ep = \sum_{i=1}^k \ep_i$.
  Denote by $N^{\pm}(\ep,T)$ the count of eigenvalues of $T$ that is strictly greater(smaller) than $\ep$ ($-\ep$).
  We have
  \begin{align}
    N^{\pm}(\ep, \sum_{i=1}^k A_i) \leq \sum_{i=1}^k N^{\pm}(\ep_i,A_i),
  \end{align}
\end{lemma}
\begin{proof}
  \citet[Lemma 5]{widom1963_AsymptoticBehavior}.
\end{proof}

\subsection{Subdomains on the sphere}

Let $d(\x,\y) = \arccos\ang{\x,\y}$ be the geodesic distance on the sphere.
The first proposition deals with the ``overlapping  area'' after rotation of two subdomains.

\begin{proposition}
  \label{prop:AreaControl}
  Let $\Omega_1, \Omega_2 \subset \bbS^d$ be two disjoint domains with piecewise smooth boundary.
  Fix two points ${e},\y \in \bbS^d$.
  Suppose that for any $\x \in \bbS^d$, $R_{{e},\x}$ is an isometric transformation such that $R_{{e},\x} {e} = \x$.
  Then, there exists some $M$ such that
  \begin{align}
    \abs{\left\{ \x \in \Omega_1 : R_{{e},\x} \y \in \Omega_2 \right\}} \leq
    M d(\y,{e}) = M \arccos \ang{\y,{e}},
  \end{align}
\end{proposition}
\begin{proof}
  Let $r = d(\y,{e})$.
  Since $R_{{e},\x}$ is isometric, we have
  \begin{align*}
    r = d(\y,{e}) = d(R_{{e},\x}\y, R_{{e},\x}{e}) = d(R_{{e},\x}\y,\x).
  \end{align*}
  Therefore, if $r < d(\x,\partial \Omega_1)$, noticing that $\x \in \Omega_1$, we have $R_{{e},\x} \y \in \overline{B}_{\x}(r) \subset \Omega_1$,
  and hence $R_{{e},\x}\y \notin \Omega_2$.
  Therefore,
  \begin{align*}
    \left\{ \x \in \Omega_1 : R_{{e},\x} \y \in \Omega_2 \right\}
    \subseteq \left\{ \x \in \Omega_1 : d(\x,\partial \Omega_1) \leq r \right\}.
  \end{align*}
  The latter is a tube of radius $r$ of $\partial \Omega_1$ as defined in \citet{weyl1939_VolumeTubes}.
  Moreover, since $\partial \Omega_1$ is piecewise smooth, the results in \citet{weyl1939_VolumeTubes} show that there is some constant $M$
  such that
  \begin{align*}
    \abs{\left\{ \x \in \Omega_1 : d(\x,\partial \Omega_1) \leq r \right\}}
    \leq M r,
  \end{align*}
  giving the desired estimation.
\end{proof}

This following proposition provides a decomposition of the sphere.
\begin{proposition}
  \label{prop:C_SphereDecomp}
  There exists a sequence of subdomains $U_0,V_0,U_1,V_1,\dots \subseteq \bbS^d$ with piecewise smooth boundary such that
  \begin{enumerate}[(1)]
    \item $U_0 = \bbS^d$;
    \item There are disjoint isometric copies $V_{i,1},\dots,V_{i,n_i}$ of $V_i$ such that
    $U_i = \bigcup_{j=1}^{n_i} V_{i,n_i} \cup Z,$
    where $Z$ is a null-set;
    \item $V_i \subseteq U_{i+1}$ after some isometric transformation;
    \item $\mathrm{diam}~V_i \to 0$.
  \end{enumerate}
\end{proposition}
\begin{proof}
  Let us denote by $S_{\bm{p}, r} = \left\{ \x \in \bbS^d \;|\; \ang{\x,\bm{p}} > \cos r \right\}$
  the spherical cap centered at $p$ with radius $r$,
  and $S_r = S_{{e}_{d+1},r}$, where ${e}_{d+1}$ is the unit vector for the last coordinate.

  First, let $V_0 = \left\{ \x = (x_1,\dots,x_{d+1}) \in \bbS^d ~\big|~ x_i > 0,~ i=1,\dots,d+1 \right\}.$
  Then, by reflection, there are $2^{d+1}$ isometric copies of $\Omega$
  such that their disjoint union is whole sphere minus equators,
  which is a null set.

  To proceed, taking $\bm{p} = \frac{1}{\sqrt {d+1}}(1,\dots,1)$, for any points $\x \in V_0$, we have
  \begin{align*}
    \ang{\x,\bm{p}} = \frac{1}{\sqrt {d+1}}\left( x_1 + \dots + x_{d+1} \right) \geq \frac{1}{\sqrt {d+1}}.
  \end{align*}
  Therefore, $V_0 \subset S_{\bm{p},r_1}$ for $r_1 = \arccos \frac{1}{\sqrt {d+1}} < \frac{\pi}{2}$
  and we may take $U_1 = S_{r_1}$.

  Now suppose we have $U_i = S_{r_i}$ with $r_i < \frac{\pi}{2}$.
  Using polar coordinate, we have the parametrization
  \begin{align}
    \label{eq:4_PolarCoord}
    x_1 &= \sin \theta_1 \cdots \sin \theta_d,\quad    x_2 = \sin \theta_1 \cdots \cos \theta_d,
    \dots,          x_{d} = \sin\theta_1 \cos \theta_2,\quad    x_{d+1} = \cos\theta_1,
  \end{align}
  where $\theta_d \in [0,2\pi]$ and $\theta_j \in [0,\pi]$, $j = 1,\dots,d-1$.
  Then, the spherical cap is given by $S_r = \left\{ \x ~|~ \theta_1 < r \right\}$.
  Let us consider the slice
  \begin{align*}
    V_i = \left\{ \x \in \bbS^d ~\Big|~ \theta_1 < r,\theta_j \in (0,\frac{\pi}{2}),~j=2,\dots,d-1,\
    \theta_d \in \left(-\frac{\pi}{4},\frac{\pi}{4}\right) \right\}
    \subset S_{r},
  \end{align*}
  Then, by rotation over $\theta_d$ and reflection over $\theta_{j},~j=2,\dots,d-1$
  we can find $2^{d}$ isometric copies of $V_i$ such that their disjoint union
  is only different with $S_{r_i}$ by the union of the boundaries of $V_i$'s,
  which a null-set.

  Now, we find some $r_{i+1} < r_i$ such that
  $V_i \subset U_{i+1} = S_{\bm{p},r_{i+1}}$.
  Let us take the point $\bm{p} = (p_1,\dots,p_{d+1})$ by
  \begin{align*}
    p_{d+1} = \cos \eta, \quad p_d = \dots = p_2 = \frac{1}{\sqrt {d-1}} \sin \eta, \quad p_1 = 0,
  \end{align*}
  where $\eta \in (0,r_i)$ will be determined later.
  Suppose now $x \in V_i$ is given by \cref{eq:4_PolarCoord}.
  We obtain that
  $\ang{\bm{p},\x} = \cos \eta \cos \theta_1 + \frac{\sin \eta}{\sqrt {d-1}} \left( x_{d} + \dots + x_2 \right)$.
  Noticing that $\theta_j \in [0,\frac{\pi}{2}]$ and $\abs{\theta_d} \geq \frac{\pi}{4}$, we have
  $x_{d} + \dots + x_2 \geq \sin \theta_1 \cos \theta_d \geq \frac{1}{\sqrt {2}} \sin \theta_1$.
Therefore,
  \begin{align*}
    \ang{\bm{p},\x} &\geq \cos \eta \cos \theta_1 + a \sin \eta \sin \theta_1, \quad a = \frac{1}{\sqrt {2(d-1)}}, \\
    & \geq \min\left( \cos\eta,\  \cos \eta \cos r_i + a \sin \eta \sin r_i \right), \quad \text{since } \theta_1 \in (0,r_i), \\
    &=
    \begin{cases}
      \cos\eta, & \tan \eta > \frac{1 - \cos r_i}{a \sin r_i}, \\
      \cos r_i \cos \eta  + (a \sin r_i) \sin \eta, & \text{otherwise}.
    \end{cases}
  \end{align*}
  We know that the second term is maximized by $\tan \eta_0 = a \tan r_i$ and
  \begin{align*}
    \cos r_i \cos \eta_0 + a \sin r_i \sin \eta_0  = \sqrt {\cos^2 r_i + a^2 \sin^2 r_i}.
  \end{align*}
  On one hand, if $\tan \eta_0 \leq \frac{1 - \cos r_i}{a \sin r_i}$,
  we take $\eta = \eta_0$ and the minimum is taken by the second term, so
  $\ang{\bm{p},\x}  \geq \sqrt {\cos^2 r_i + a^2 \sin^2 r_i}$
  and $V_i \subset S_{\bm{p},r_{i+1}}$ for $r_{i+1} = \arccos \sqrt {\cos^2 r_i + a^2 \sin^2 r_i}$.
  In this case, we have
  \begin{align}
    \label{eq:4_ProofRRelation1}
    \sin^2 r_{i+1} = 1 - (\cos^2 r_i + a^2 \sin^2 r_i) = (1-a^2) \sin^2 r_i.
  \end{align}
  On the other hand, if $\tan \eta_0 = a \tan r_i > \frac{1 - \cos r_i}{a \sin r_i}$,
  we take $\eta = \arctan \frac{1 - \cos r_i}{a \sin r_i}$.
  Then, the minimum is taken by the first term and $\ang{\bm{p},\x} \geq \cos \eta$, implying
  $V \subset S_{\bm{p},r_{i+1}}$ for $r_{i+1} = \eta$.
  In this case, we have
  \begin{align}
    \label{eq:4_ProofRRelation2}
    \tan r_{i+1} = \tan \eta = \frac{1 - \cos r_i}{a \sin r_i} < a \tan r_i.
  \end{align}

  Considering both cases \cref{eq:4_ProofRRelation1},\cref{eq:4_ProofRRelation2} and noticing that
  $a = \frac{1}{\sqrt {2(d-1)}} \in (0,1)$,
  we conclude that $r_{i+1} < r_i$ and $r_i \to 0$.
\end{proof}

\subsection{Cesaro sum}\label{subsec:AUX_Cesaro}
We will use Cesaro sum in our analysis of dot-product kernels.
We also refer to \citet[Section A.4]{dai2013_ApproximationTheory}.
\begin{definition}
  Let $p \geq 0$.
  The $p$-Cesaro sum $s_n$ of a sequence $a_k$ is defined by
  \begin{align}
    s_n^p = \frac{1}{A_n^p}\sum_{k=0}^n A_{n-k}^{p} a_{k}, \qquad A_k^{p} \coloneqq \binom{k+p}{k}.
  \end{align}
\end{definition}

\begin{definition}[Difference]
  \label{def:DifferenceOperator}
  Let $\bm{a} = (a_k)_{k\geq 0}$ be a sequence.
  We define the difference operator on sequence by
  \begin{align}
  (\triangle^0 \bm{a})
    _k = a_k, \quad
    (\triangle \bm{a})_k = a_{k} - a_{k+1}, \quad
    \triangle^{p+1} \bm{a} = \triangle(\triangle^{p}\bm{a}).
  \end{align}
  We often write $(\triangle^{p} \bm{a})_k = \triangle^{p} a_{k}$.
  It is easy to see that $\triangle^{p} a_{k} = \sum_{r=0}^p \binom{k}{r} (-1)^r a_{k+r}.$

\end{definition}

\begin{definition}[Tail sum]
  Let $\bm{a} = (a_k)_{k\geq 0}$ be a sequence.
  Assuming all the following summations are absolutely convergent,
  we define the tail sum operator on sequence by
  \begin{align}
  (S^0 \bm{a}) _k = a_k,\quad
    (S \bm{a})_k = \sum_{r \geq k} a_r,\quad
    S^{p+1} \bm{a} = S(S^{p}\bm{a}).
\end{align}
  We often write $(S^{p} \bm{a})_k = S^p a_k$.
\end{definition}

The following is an elementary proposition about the connection between tail sum and difference.
\begin{proposition}
  \label{prop:TailSumDifference}
  We have (a) $S^p a_n = \sum_{k=0}^{\infty} A^{p-1}_k a_{n+k}$;
  (b) $S\triangle a_n = \triangle S a_n = a_n$;
  (c) Consequently, $\sum_{k=0}^{\infty} A^{p}_k \triangle^{p+1} a_{n+k} = (S^{p+1} \triangle^{p+1} a)_{n} = a_n$.

\end{proposition}

We have the following summation by parts formula, see also \citet[(A.4.8)]{dai2013_ApproximationTheory}.

\begin{proposition}[Summation by parts]
  \label{prop:Aux_SumByParts}
  Let $a_k,b_k$ be two sequence and $p \in \bbN$. Then,
  \begin{align*}
    \sum_{k=0}^{\infty} a_k b_k = \sum_{k=0}^{\infty} \triangle^{p+1} b_k \sum_{j=0}^k A^p_{k-j}a_j
    =  \sum_{k=0}^{\infty} \triangle^{p+1} b_k A_k^p s_k^p,
  \end{align*}
  where $s_k^p$ is the $p$-Cesaro mean of $a_k$.
\end{proposition}

For a function $f : [a,b] \to \R$, we can define similarly $\triangle f(x) = f(x) - f(x+1)$ and $\triangle^{p+1} f(x) = \triangle^p f(x)-\triangle^p f(x+1)$.
The following elementary lemma provides a connection between the difference and the derivative of the function.
\begin{lemma}
  Let $p \in \bbN$.
  Suppose $f \in C^p([a,a+p])$, then
  \begin{align}
    \label{eq:A_Difference_Derivative}
    \triangle^p f(a) = (-1)^{p} \int_{[0,1]^p} f^{(p)}(a+t_1+\dots+t_p) \dd t_1\cdots \dd t_p.
  \end{align}
\end{lemma}

\begin{proposition}
  \label{prop:Aux_Difference_Derivative}
  Suppose that $\mu_k = c_0 (k+1)^{-\beta},~ k\geq 0$.
  Letting $(\beta)_p \coloneqq \beta (\beta+1)\cdots (\beta+p-1)$, then
  \begin{align*}
    0 < \triangle^{p} \mu_k \leq c_0 (\beta)_p  (k+1)^{-(\beta+p)}.
  \end{align*}
\end{proposition}
\begin{proof}
Apply the previous lemma with $f(x) = c(x+1)^{-\beta}$ and  $f^{(p)}(x) = (-1)^p c_0 (\beta)_p (x+1)^{-(\beta+p)}$.
\end{proof}

\begin{lemma}
  \label{lem:LeftExtrapolation}
  Let $\bm{\mu} = (\mu_k)_{k \geq 0}$ be a sequence such that $\triangle^p \mu_k \geq 0,~\forall k \geq 0$ for some $p\geq 0$.
  Given $N \geq 0$, we can construct a left extrapolation sequence $(\tilde{\mu}_k)_{k \geq 0}$ such that
  \begin{enumerate}[(1)]
    \item $\tilde{\mu}_k = \mu_k$ for $k \geq N$ and $\tilde{\mu}_k \leq \mu_k$ for $k < N$;
    \item $\triangle^p \tilde{\mu}_k \geq 0$, $\forall k \geq 0$;
    \item Let $\bar{\mu}_k = \mu_k - \tilde{\mu}_k$ be the residual sequence, then $\triangle^p \bar{\mu}_k \geq 0$, $\forall k \geq 0$;
    \item The leading term satisfies
    \begin{align*}
      \tilde{\mu}_0 = \caL_N^p \bm{\mu} \coloneqq \sum_{l=0}^{p-1} A_N^l \triangle^l \mu_N.
    \end{align*}
    We remark that the $\caL_N^p$ is in the same form as the LHS of \cref{eq:EDRDerivativeBound} in \cref{cond:EDR}.
  \end{enumerate}
\end{lemma}
\begin{proof}
  We define $\tilde{\mu}_k$ recursively by its $p$-differences.
  Let $\triangle^p \tilde{\mu}_k = 0$ for $k < N$ and $\triangle^p \tilde{\mu}_k = \triangle^p \mu_k$ for $k \geq N$.
  Then, summing up the terms iteratively yields (1), (2) and also the recursive formula:
  \begin{align*}
    \triangle^{p-s}\tilde{\mu}_{N-r} = \sum_{l=0}^{s-1} A^l_r \triangle^{p-s+l} \mu_{N},\quad s = 1,\dots,p,\quad r = 0,\dots,N-1,
  \end{align*}
  which gives (4).
  The last statement (3) follows from the fact that $\triangle^p \bar{\mu}_k = \triangle^p \mu_k - \triangle^p \tilde{\mu}_k \geq 0$.
\end{proof}


\vskip 0.2in
  \begin{thebibliography}{48}
\providecommand{\natexlab}[1]{#1}
\providecommand{\url}[1]{\texttt{#1}}
\expandafter\ifx\csname urlstyle\endcsname\relax
  \providecommand{\doi}[1]{doi: #1}\else
  \providecommand{\doi}{doi: \begingroup \urlstyle{rm}\Url}\fi

\bibitem[Adams and Fournier(2003)]{adams2003_SobolevSpaces}
Robert~A Adams and John~JF Fournier.
\newblock \emph{Sobolev Spaces}.
\newblock {Elsevier}, 2003.

\bibitem[{Allen-Zhu} et~al.(2019{\natexlab{a}}){Allen-Zhu}, Li, and
  Song]{allen-zhu2019_ConvergenceRate}
Zeyuan {Allen-Zhu}, Yuanzhi Li, and Zhao Song.
\newblock On the convergence rate of training recurrent neural networks.
\newblock \emph{Advances in neural information processing systems}, 32,
  2019{\natexlab{a}}.

\bibitem[{Allen-Zhu} et~al.(2019{\natexlab{b}}){Allen-Zhu}, Li, and
  Song]{allen-zhu2019_ConvergenceTheory}
Zeyuan {Allen-Zhu}, Yuanzhi Li, and Zhao Song.
\newblock A convergence theory for deep learning via over-parameterization,
  June 2019{\natexlab{b}}.
\newblock URL \url{http://arxiv.org/abs/1811.03962}.

\bibitem[Andreas~Christmann(2008)]{andreaschristmann2008_SupportVector}
Ingo Steinwart~(auth.) Andreas~Christmann.
\newblock \emph{Support Vector Machines}.
\newblock Information {{Science}} and {{Statistics}}. {Springer-Verlag New
  York}, {New York, NY}, 1 edition, 2008.
\newblock ISBN 0-387-77242-1 0-387-77241-3 978-0-387-77241-7 978-0-387-77242-4.
\newblock \doi{10.1007/978-0-387-77242-4}.

\bibitem[Arora et~al.(2019{\natexlab{a}})Arora, Du, Hu, Li, and
  Wang]{arora2019_FinegrainedAnalysis}
Sanjeev Arora, Simon Du, Wei Hu, Zhiyuan Li, and Ruosong Wang.
\newblock Fine-grained analysis of optimization and generalization for
  overparameterized two-layer neural networks.
\newblock In \emph{International {{Conference}} on {{Machine Learning}}}, pages
  322--332. {PMLR}, 2019{\natexlab{a}}.

\bibitem[Arora et~al.(2019{\natexlab{b}})Arora, Du, Hu, Li, Salakhutdinov, and
  Wang]{arora2019_ExactComputation}
Sanjeev Arora, Simon~S. Du, Wei Hu, Zhiyuan Li, Russ~R Salakhutdinov, and
  Ruosong Wang.
\newblock On exact computation with an infinitely wide neural net.
\newblock In \emph{Advances in {{Neural Information Processing Systems}}},
  volume~32. {Curran Associates, Inc.}, 2019{\natexlab{b}}.
\newblock URL
  \url{https://proceedings.neurips.cc/paper/2019/hash/dbc4d84bfcfe2284ba11beffb853a8c4-Abstract.html}.

\bibitem[Azevedo and Menegatto(2014)]{azevedo2014_SharpEstimates}
D.~Azevedo and V.A. Menegatto.
\newblock Sharp estimates for eigenvalues of integral operators generated by
  dot product kernels on the sphere.
\newblock \emph{Journal of Approximation Theory}, 177:\penalty0 57--68, January
  2014.
\newblock ISSN 00219045.
\newblock \doi{10.1016/j.jat.2013.10.002}.

\bibitem[Bartlett et~al.(2020)Bartlett, Long, Lugosi, and
  Tsigler]{bartlett2020_BenignOverfittinga}
Peter~L. Bartlett, Philip~M. Long, G{\'a}bor Lugosi, and Alexander Tsigler.
\newblock Benign overfitting in linear regression.
\newblock \emph{Proceedings of the National Academy of Sciences}, 117\penalty0
  (48):\penalty0 30063--30070, 2020.

\bibitem[Bauer et~al.(2007)Bauer, Pereverzyev, and
  Rosasco]{bauer2007_RegularizationAlgorithms}
F.~Bauer, S.~Pereverzyev, and L.~Rosasco.
\newblock On regularization algorithms in learning theory.
\newblock \emph{Journal of complexity}, 23\penalty0 (1):\penalty0 52--72, 2007.
\newblock \doi{10.1016/j.jco.2006.07.001}.

\bibitem[Beaglehole et~al.(2022)Beaglehole, Belkin, and
  Pandit]{beaglehole2022_KernelRidgeless}
Daniel Beaglehole, Mikhail Belkin, and Parthe Pandit.
\newblock Kernel ridgeless regression is inconsistent in low dimensions, June
  2022.

\bibitem[Bietti and Bach(2020)]{bietti2020_DeepEquals}
Alberto Bietti and Francis Bach.
\newblock Deep equals shallow for {{ReLU}} networks in kernel regimes.
\newblock \emph{arXiv preprint arXiv:2009.14397}, 2020.

\bibitem[Bietti and Mairal(2019)]{bietti2019_InductiveBias}
Alberto Bietti and Julien Mairal.
\newblock On the inductive bias of neural tangent kernels.
\newblock In \emph{Advances in {{Neural Information Processing Systems}}},
  volume~32, 2019.

\bibitem[Caponnetto and Yao(2010)]{caponnetto2010_CrossvalidationBased}
A.~Caponnetto and Y.~Yao.
\newblock Cross-validation based adaptation for regularization operators in
  learning theory.
\newblock \emph{Analysis and Applications}, 08:\penalty0 161--183, 2010.
\newblock \doi{10.1142/S0219530510001564}.

\bibitem[Caponnetto and De~Vito(2007)]{caponnetto2007_OptimalRates}
Andrea Caponnetto and Ernesto De~Vito.
\newblock Optimal rates for the regularized least-squares algorithm.
\newblock \emph{Foundations of Computational Mathematics}, 7\penalty0
  (3):\penalty0 331--368, 2007.
\newblock \doi{10.1007/s10208-006-0196-8}.

\bibitem[Chen and Xu(2020)]{chen2020_DeepNeural}
Lin Chen and Sheng Xu.
\newblock Deep neural tangent kernel and laplace kernel have the same {{RKHS}}.
\newblock \emph{arXiv preprint arXiv:2009.10683}, 2020.

\bibitem[Cho and Saul(2009)]{cho2009_KernelMethods}
Youngmin Cho and Lawrence Saul.
\newblock Kernel methods for deep learning.
\newblock In Y.~Bengio, D.~Schuurmans, J.~Lafferty, C.~Williams, and
  A.~Culotta, editors, \emph{Advances in Neural Information Processing
  Systems}, volume~22. {Curran Associates, Inc.}, 2009.
\newblock URL
  \url{https://proceedings.neurips.cc/paper/2009/file/5751ec3e9a4feab575962e78e006250d-Paper.pdf}.

\bibitem[Dai and Xu(2013)]{dai2013_ApproximationTheory}
Feng Dai and Yuan Xu.
\newblock \emph{Approximation Theory and Harmonic Analysis on Spheres and
  Balls}.
\newblock Springer {{Monographs}} in {{Mathematics}}. {Springer New York}, {New
  York, NY}, 2013.
\newblock ISBN 978-1-4614-6659-8 978-1-4614-6660-4.
\newblock \doi{10.1007/978-1-4614-6660-4}.

\bibitem[Devlin et~al.(2019)Devlin, Chang, Lee, and
  Toutanova]{devlin2019_BERTPretraining}
Jacob Devlin, Ming-Wei Chang, Kenton Lee, and Kristina Toutanova.
\newblock {{BERT}}: {{Pre-training}} of deep bidirectional transformers for
  language understanding, May 2019.

\bibitem[Du et~al.(2018)Du, Zhai, Poczos, and Singh]{du2018_GradientDescent}
Simon~S. Du, Xiyu Zhai, Barnabas Poczos, and Aarti Singh.
\newblock Gradient descent provably optimizes over-parameterized neural
  networks.
\newblock In \emph{International {{Conference}} on {{Learning
  Representations}}}, September 2018.
\newblock URL \url{https://openreview.net/forum?id=S1eK3i09YQ}.

\bibitem[Du et~al.(2019)Du, Lee, Li, Wang, and Zhai]{du2019_GradientDescent}
Simon~S. Du, Jason Lee, Haochuan Li, Liwei Wang, and Xiyu Zhai.
\newblock Gradient descent finds global minima of deep neural networks.
\newblock In \emph{Proceedings of the 36th {{International Conference}} on
  {{Machine Learning}}}, pages 1675--1685. {PMLR}, May 2019.
\newblock URL \url{https://proceedings.mlr.press/v97/du19c.html}.

\bibitem[Fan and Wang(2020)]{fan2020_SpectraConjugate}
Zhou Fan and Zhichao Wang.
\newblock Spectra of the conjugate kernel and neural tangent kernel for
  linear-width neural networks.
\newblock \emph{Advances in neural information processing systems},
  33:\penalty0 7710--7721, 2020.

\bibitem[Fischer and Steinwart(2020)]{fischer2020_SobolevNorm}
Simon-Raphael Fischer and Ingo Steinwart.
\newblock Sobolev norm learning rates for regularized least-squares algorithms.
\newblock \emph{Journal of Machine Learning Research}, 21:\penalty0
  205:1--205:38, 2020.
\newblock URL
  \url{https://www.semanticscholar.org/paper/248fb62f75dac19f02f683cecc2bf4929f3fcf6d}.

\bibitem[Frei et~al.(2022)Frei, Chatterji, and
  Bartlett]{frei2022_BenignOverfitting}
Spencer Frei, Niladri~S. Chatterji, and Peter Bartlett.
\newblock Benign overfitting without linearity: {{Neural}} network classifiers
  trained by gradient descent for noisy linear data.
\newblock In \emph{Proceedings of {{Thirty Fifth Conference}} on {{Learning
  Theory}}}, pages 2668--2703. {PMLR}, June 2022.
\newblock URL \url{https://proceedings.mlr.press/v178/frei22a.html}.

\bibitem[Geifman et~al.(2020)Geifman, Yadav, Kasten, Galun, Jacobs, and
  Ronen]{geifman2020_SimilarityLaplace}
Amnon Geifman, Abhay Yadav, Yoni Kasten, Meirav Galun, David Jacobs, and Basri
  Ronen.
\newblock On the similarity between the {{Laplace}} and neural tangent kernels.
\newblock In \emph{Advances in {{Neural Information Processing Systems}}},
  volume~33, pages 1451--1461, 2020.

\bibitem[He et~al.(2016)He, Zhang, Ren, and Sun]{he2016deep}
Kaiming He, Xiangyu Zhang, Shaoqing Ren, and Jian Sun.
\newblock Deep residual learning for image recognition.
\newblock In \emph{Proceedings of the {{IEEE}} Conference on Computer Vision
  and Pattern Recognition}, pages 770--778, 2016.

\bibitem[Hu et~al.(2021)Hu, Wang, Lin, and Cheng]{hu2021_RegularizationMatters}
Tianyang Hu, Wenjia Wang, Cong Lin, and Guang Cheng.
\newblock Regularization matters: {{A}} nonparametric perspective on
  overparametrized neural network.
\newblock In \emph{International {{Conference}} on {{Artificial Intelligence}}
  and {{Statistics}}}, pages 829--837. {PMLR}, 2021.

\bibitem[Jacot et~al.(2018)Jacot, Gabriel, and
  Hongler]{jacot2018_NeuralTangent}
Arthur Jacot, Franck Gabriel, and Clement Hongler.
\newblock Neural tangent kernel: {{Convergence}} and generalization in neural
  networks.
\newblock In S.~Bengio, H.~Wallach, H.~Larochelle, K.~Grauman,
  N.~{Cesa-Bianchi}, and R.~Garnett, editors, \emph{Advances in Neural
  Information Processing Systems}, volume~31. {Curran Associates, Inc.}, 2018.
\newblock URL
  \url{https://proceedings.neurips.cc/paper/2018/file/5a4be1fa34e62bb8a6ec6b91d2462f5a-Paper.pdf}.

\bibitem[Karras et~al.(2019)Karras, Laine, and
  Aila]{karras2019_StylebasedGenerator}
Tero Karras, Samuli Laine, and Timo Aila.
\newblock A style-based generator architecture for generative adversarial
  networks.
\newblock In \emph{Proceedings of the {{IEEE}}/{{CVF}} Conference on Computer
  Vision and Pattern Recognition}, pages 4401--4410, 2019.

\bibitem[Krizhevsky et~al.(2017)Krizhevsky, Sutskever, and
  Hinton]{krizhevsky2017_ImagenetClassification}
Alex Krizhevsky, Ilya Sutskever, and Geoffrey~E Hinton.
\newblock Imagenet classification with deep convolutional neural networks.
\newblock \emph{Communications of the ACM}, 60\penalty0 (6):\penalty0 84--90,
  2017.

\bibitem[Lai et~al.(2023)Lai, Xu, Chen, and Lin]{lai2023_GeneralizationAbility}
Jianfa Lai, Manyun Xu, Rui Chen, and Qian Lin.
\newblock Generalization ability of wide neural networks on {{R}}, February
  2023.

\bibitem[Lee et~al.(2019)Lee, Xiao, Schoenholz, Bahri, Novak, {Sohl-Dickstein},
  and Pennington]{lee2019_WideNeural}
Jaehoon Lee, Lechao Xiao, Samuel Schoenholz, Yasaman Bahri, Roman Novak, Jascha
  {Sohl-Dickstein}, and Jeffrey Pennington.
\newblock Wide neural networks of any depth evolve as linear models under
  gradient descent.
\newblock In \emph{Advances in {{Neural Information Processing Systems}}},
  volume~32. {Curran Associates, Inc.}, 2019.
\newblock URL
  \url{https://proceedings.neurips.cc/paper/2019/hash/0d1a9651497a38d8b1c3871c84528bd4-Abstract.html}.

\bibitem[Li et~al.(2023{\natexlab{a}})Li, Zhang, and
  Lin]{li2023_KernelInterpolation}
Yicheng Li, Haobo Zhang, and Qian Lin.
\newblock Kernel interpolation generalizes poorly.
\newblock \emph{Biometrika}, page asad048, August 2023{\natexlab{a}}.
\newblock ISSN 0006-3444, 1464-3510.
\newblock \doi{10.1093/biomet/asad048}.

\bibitem[Li et~al.(2023{\natexlab{b}})Li, Zhang, and
  Lin]{li2023_SaturationEffect}
Yicheng Li, Haobo Zhang, and Qian Lin.
\newblock On the saturation effect of kernel ridge regression.
\newblock In \emph{International {{Conference}} on {{Learning
  Representations}}}, February 2023{\natexlab{b}}.
\newblock URL \url{https://openreview.net/forum?id=tFvr-kYWs_Y}.

\bibitem[Li and Liang(2018)]{li2018_LearningOverparameterized}
Yuanzhi Li and Yingyu Liang.
\newblock Learning overparameterized neural networks via stochastic gradient
  descent on structured data.
\newblock \emph{Advances in neural information processing systems}, 31, 2018.

\bibitem[Liang and Rakhlin(2020)]{liang2020_JustInterpolate}
Tengyuan Liang and Alexander Rakhlin.
\newblock Just interpolate: {{Kernel}} "ridgeless" regression can generalize.
\newblock \emph{The Annals of Statistics}, 48\penalty0 (3), June 2020.
\newblock ISSN 0090-5364.
\newblock \doi{10.1214/19-AOS1849}.

\bibitem[Lin et~al.(2018)Lin, Rudi, Rosasco, and Cevher]{lin2018_OptimalRates}
Junhong Lin, Alessandro Rudi, L.~Rosasco, and V.~Cevher.
\newblock Optimal rates for spectral algorithms with least-squares regression
  over {{Hilbert}} spaces.
\newblock \emph{Applied and Computational Harmonic Analysis}, 48:\penalty0
  868--890, 2018.
\newblock \doi{10.1016/j.acha.2018.09.009}.

\bibitem[Montanari and Zhong(2022)]{montanari2022_InterpolationPhase}
Andrea Montanari and Yiqiao Zhong.
\newblock The interpolation phase transition in neural networks:
  {{Memorization}} and generalization under lazy training.
\newblock \emph{The Annals of Statistics}, 50\penalty0 (5):\penalty0
  2816--2847, 2022.

\bibitem[Nakkiran et~al.(2019)Nakkiran, Kaplun, Bansal, Yang, Barak, and
  Sutskever]{nakkiran2019_DeepDouble}
Preetum Nakkiran, Gal Kaplun, Yamini Bansal, Tristan Yang, Boaz Barak, and Ilya
  Sutskever.
\newblock Deep double descent: {{Where}} bigger models and more data hurt.
\newblock In \emph{International {{Conference}} on {{Learning
  Representations}}}, September 2019.
\newblock URL \url{https://openreview.net/forum?id=B1g5sA4twr}.

\bibitem[Nguyen et~al.(2021)Nguyen, Mondelli, and
  Montufar]{nguyen2021_TightBounds}
Quynh Nguyen, Marco Mondelli, and Guido~F. Montufar.
\newblock Tight bounds on the smallest eigenvalue of the neural tangent kernel
  for deep {{ReLU}} networks.
\newblock In \emph{International {{Conference}} on {{Machine Learning}}}, pages
  8119--8129. {PMLR}, 2021.

\bibitem[Rakhlin and Zhai(2018)]{rakhlin2018_ConsistencyInterpolation}
Alexander Rakhlin and Xiyu Zhai.
\newblock Consistency of interpolation with {{Laplace}} kernels is a
  high-dimensional phenomenon, December 2018.
\newblock URL \url{http://arxiv.org/abs/1812.11167}.

\bibitem[Ronen et~al.(2019)Ronen, Jacobs, Kasten, and
  Kritchman]{ronen2019_ConvergenceRate}
Basri Ronen, David Jacobs, Yoni Kasten, and Shira Kritchman.
\newblock The convergence rate of neural networks for learned functions of
  different frequencies.
\newblock \emph{Advances in Neural Information Processing Systems}, 32, 2019.

\bibitem[Simon(2015)]{simon2015_OperatorTheory}
Barry Simon.
\newblock \emph{Operator Theory}.
\newblock {American Mathematical Society}, {Providence, Rhode Island}, November
  2015.
\newblock ISBN 978-1-4704-1103-9 978-1-4704-2763-4.
\newblock \doi{10.1090/simon/004}.

\bibitem[Steinwart and Scovel(2012)]{steinwart2012_MercerTheorem}
Ingo Steinwart and C.~Scovel.
\newblock Mercer's theorem on general domains: {{On}} the interaction between
  measures, kernels, and {{RKHSs}}.
\newblock \emph{Constructive Approximation}, 35\penalty0 (3):\penalty0
  363--417, 2012.
\newblock \doi{10.1007/S00365-012-9153-3}.

\bibitem[Suh et~al.(2022)Suh, Ko, and Huo]{suh2022_NonparametricRegression}
Namjoon Suh, Hyunouk Ko, and Xiaoming Huo.
\newblock A non-parametric regression viewpoint: {{Generalization}} of
  overparametrized deep {{ReLU}} network under noisy observations.
\newblock In \emph{International {{Conference}} on {{Learning
  Representations}}}, May 2022.
\newblock URL \url{https://openreview.net/forum?id=bZJbzaj_IlP}.

\bibitem[Vershynin(2010)]{vershynin2010_IntroductionNonasymptotic}
Roman Vershynin.
\newblock Introduction to the non-asymptotic analysis of random matrices.
\newblock \emph{arXiv preprint arXiv:1011.3027}, 2010.

\bibitem[Vershynin(2018)]{vershynin2018_HighdimensionalProbability}
Roman Vershynin.
\newblock \emph{High-Dimensional Probability: {{An}} Introduction with
  Applications in Data Science}, volume~47.
\newblock {Cambridge university press}, 2018.
\newblock ISBN 1-108-24454-8.

\bibitem[Weyl(1939)]{weyl1939_VolumeTubes}
Hermann Weyl.
\newblock On the volume of tubes.
\newblock \emph{American Journal of Mathematics}, 61\penalty0 (2):\penalty0
  461--472, 1939.

\bibitem[Widom(1963)]{widom1963_AsymptoticBehavior}
Harold Widom.
\newblock Asymptotic behavior of the eigenvalues of certain integral equations.
\newblock \emph{Transactions of the American Mathematical Society},
  109\penalty0 (2):\penalty0 278--295, 1963.
\newblock ISSN 0002-9947.
\newblock \doi{10.2307/1993907}.

\end{thebibliography}
\end{document}